%% file: neurips_2026.tex
\documentclass{article}

\PassOptionsToPackage{numbers, compress}{natbib}

\usepackage[main, final]{neurips_2026}

\usepackage[utf8]{inputenc} %
\usepackage[T1]{fontenc}    %
\usepackage{hyperref}       %
\addtocontents{toc}{\protect\setcounter{tocdepth}{-1}}
\usepackage{url}            %
\usepackage{booktabs}       %
\usepackage{amsfonts}       %
\usepackage{graphicx}       %
\usepackage{amsthm}         %
\usepackage{amsmath}
\usepackage{nicefrac}       %
\usepackage{microtype}      %
\PassOptionsToPackage{table}{xcolor}
\usepackage{xcolor}         %
\usepackage[most]{tcolorbox}      %
\usepackage{soul}           %
\usepackage{caption}
\usepackage{subcaption}
\usepackage{wrapfig}
\usepackage{lipsum}
\usepackage{enumitem}       %
\usepackage{multirow}       %
\usepackage{tikz}           %
\usepackage{kotex}          %
\usetikzlibrary{positioning, arrows.meta, fit, backgrounds, calc}
\usepackage{booktabs}
\usepackage{makecell}
\usepackage{pifont}         %
\newcommand{\cmark}{\textcolor{myteal}{\ding{51}}}
\newcommand{\xmark}{\textcolor{statgray}{\ding{55}}}
\newcommand{\nmark}{\textcolor{statgray}{--}}
\usepackage{standalone}
\usepackage{pgfplots}
\usepackage{cancel}
\usepackage{cleveref}
\input{math_commands}

\usepgfplotslibrary{groupplots} %
\usetikzlibrary{calc}           %
\usepackage[font=small]{caption}
\usepackage{algorithm}
\usepackage{algorithmic}

\newcommand*\tcircled[1]{\tikz[baseline=(char.base)]{
            \node[shape=circle,draw,inner sep=0.3pt] (char) {#1};}}
            
\newcommand*\ttcircled[1]{\tikz[baseline=(char.base)]{
            \node[shape=circle,draw,inner sep=0.1pt] (char) {#1};}}
\newcommand{\capbest}{%
  \begingroup
  \setlength{\fboxsep}{1pt}%
  \colorbox{hltealbest}{\strut highest}%
  \endgroup
}

\newcommand{\capsecond}{%
  \begingroup
  \setlength{\fboxsep}{1pt}%
  \colorbox{hltealsecond}{\strut second}%
  \endgroup
}
\newcommand{\subhead}[1]{\scriptsize #1}
\newcommand{\res}[2]{#1 \textcolor{statgray}{\scriptsize $\pm$ #2}}

\newcommand{\best}[2]{\cellcolor{hltealbest}#1 \textcolor{statgray}{\scriptsize $\pm$ #2}}
\newcommand{\second}[2]{\cellcolor{hltealsecond}#1 \textcolor{statgray}{\scriptsize $\pm$ #2}}

\newcommand{\bestnum}[1]{\cellcolor{hltealbest}#1}
\newcommand{\secondnum}[1]{\cellcolor{hltealsecond}#1}

\definecolor{mygray}{HTML}{6B7280}
\definecolor{myblue}{HTML}{1D4ED8}
\definecolor{mynavy}{HTML}{2d00f7}
\definecolor{myteal}{HTML}{0F766E}
\definecolor{myred}{HTML}{d62728}
\definecolor{myorange}{HTML}{FF7259}
\definecolor{mypurple}{HTML}{7C3AED}
\definecolor{myamber}{HTML}{D97706}
\definecolor{statgray}{gray}{0.42}

\definecolor{hltealbest}{HTML}{8CDED2} %
\definecolor{hltealsecond}{HTML}{C8F2EB}
\hypersetup{colorlinks=true,urlcolor=myteal,citecolor=myteal,linkcolor=myteal}

\title{FLAG: Flow Policy MaxEnt-RL by \\ Latent Augmented Guidance}

\author{%
    \textbf{Sungha Kim}\textsuperscript{\rm 1}\thanks{Equal contribution.}
    \hspace{2.5em}
    \textbf{Gawon Lee}\textsuperscript{\rm 2}\footnotemark[1]
    \hspace{2.5em}
    \textbf{Jusuk Lee}\textsuperscript{\rm 2} \\
    \textbf{Jonghae Park}\textsuperscript{\rm 2}
    \hspace{2.5em}
    \textbf{H. Jin Kim}\textsuperscript{\rm 1,2}
    \hspace{2.5em}
    \textbf{Daesol Cho}\textsuperscript{\rm  3} \\ [4pt]
    \textsuperscript{\rm 1} Interdisciplinary Program in Artificial Intelligence, Seoul National University \\
    \textsuperscript{\rm 2} Dept. of Aerospace Engineering, Seoul National University \\
    \textsuperscript{\rm 3} Georgia Institute of Technology \\ [5pt]
    {\normalfont\fontsize{10}{12}\selectfont 
    \url{https://flag-rl.github.io/}
    \vspace{-1.5em}
  }
}

\begin{document}

\theoremstyle{plain}
\newtheorem{theorem}{Theorem}[section]
\newtheorem{proposition}[theorem]{Proposition}
\newtheorem{lemma}[theorem]{Lemma}
\newtheorem{corollary}[theorem]{Corollary}
\theoremstyle{definition}
\newtheorem{definition}[theorem]{Definition}
\newtheorem{assumption}[theorem]{Assumption}
\theoremstyle{remark}
\newtheorem{remark}[theorem]{Remark}
\newtheorem*{remark*}{Remark}
\tcbset{
  tealresultbox/.style={
    colback=myteal!5,
    colframe=myteal!30,
    boxrule=0.4pt,
    arc=0pt,
    left=6pt,
    right=6pt,
    top=4pt,
    bottom=4pt,
    boxsep=0pt,
    before skip=3pt,
    after skip=3pt,
    enhanced,
    breakable,
  },
  tealdefinitionbox/.style={
    colback=myteal!4,
    colframe=myteal!18,
    boxrule=0.3pt,
    arc=0pt,
    left=6pt,
    right=6pt,
    top=4pt,
    bottom=4pt,
    boxsep=0pt,
    before skip=3pt,
    after skip=3pt,
    enhanced,
    breakable,
  },
  blueassumptionbox/.style={
    colback=mypurple!7,
    colframe=mypurple!30,
    boxrule=0.4pt,
    arc=0pt,
    left=6pt,
    right=6pt,
    top=4pt,
    bottom=4pt,
    boxsep=0pt,
    before skip=3pt,
    after skip=3pt,
    enhanced,
    breakable,
  },
  blueremarkbox/.style={
    colback=myblue!5,
    colframe=myblue!35,
    boxrule=0.5pt,
    arc=0pt,
    left=6pt,
    right=6pt,
    top=4pt,
    bottom=4pt,
    boxsep=0pt,
    before skip=3pt,
    after skip=3pt,
    enhanced,
    breakable,
    }
}
\tcolorboxenvironment{theorem}{tealresultbox}
\tcolorboxenvironment{proposition}{tealresultbox}
\tcolorboxenvironment{lemma}{tealresultbox}
\tcolorboxenvironment{corollary}{tealresultbox}

\tcolorboxenvironment{definition}{tealdefinitionbox}

\tcolorboxenvironment{assumption}{blueassumptionbox}
\tcolorboxenvironment{remark*}{blueremarkbox}

\maketitle

\begin{abstract}
    Maximum entropy reinforcement learning (MaxEnt-RL) enables robust exploration, yet practical implementations often restrict policies to simple Gaussians.
    While recent approaches incorporate expressive generative policies via importance-weighted supervised learning, they are prone to importance weight collapse, which limits their scalability in high-dimensional action spaces.
    Our key insight is to mitigate this limitation by localizing the sampling region, avoiding the weight degeneracy induced by importance sampling over the entire action space.
    To instantiate this insight, we introduce \textbf{FLAG} (\textbf{F}low policy with \textbf{L}atent-\textbf{A}ugmented \textbf{G}uidance).
    FLAG augments the state space with a flow latent variable and optimizes a provably consistent proxy MaxEnt-RL objective.
    We empirically demonstrate that FLAG enables expressive policy optimization with limited importance samples and scales to state-based high-dimensional control tasks.
    Furthermore, FLAG achieves state-of-the-art performance among expressive-policy MaxEnt online RL methods on state-based high-dimensional control benchmarks. 
\end{abstract}

\input{sections/introduction}
\input{sections/related_works}
\input{sections/preliminaries}

\input{sections/method}
\input{sections/experiments}

\input{sections/conclusion}

\begin{ack}
    This work was partly supported by Samsung Research Funding \& Incubation Center of Samsung Electronics under Project Number SRFC-IT2402-17 and Institute of Information \& communications Technology Planning \& Evaluation (IITP) grant funded by the Korea government (MSIT) [NO.RS-2021-II211343, Artificial Intelligence Graduate School Program (Seoul National University)].
\end{ack}
\bibliographystyle{plainnat} %
\bibliography{reference}    %

\newpage

\appendix
\crefalias{section}{appendix}
\crefalias{subsection}{appendix}
\crefalias{subsubsection}{appendix}
\addtocontents{toc}{\protect\setcounter{tocdepth}{3}}
\setcounter{tocdepth}{3}

\renewcommand{\contentsname}{Appendix}
\tableofcontents
\clearpage

\input{sections/appendix/extended_related_works}
\input{sections/appendix/flow_maxent-rl}

\input{sections/appendix/derivations}
\input{sections/appendix/proofs}

\input{sections/appendix/details}
\input{sections/appendix/baselines}

\input{sections/appendix/additional_ablations}
\input{sections/appendix/experiment_details}

\newpage
\input{checklist.tex}

\end{document}

%% file: math_commands.tex
\usepackage{amsmath,amsfonts,bm}
\usepackage{amssymb, amsthm}
\usepackage{mathtools}

\def\Figref#1{Figure~\ref{#1}}

\def\eqref#1{equation~\ref{#1}}
\def\Eqref#1{Eq.~(\ref{#1})}

\def\1{\bm{1}}

\DeclareMathAlphabet{\mathsfit}{\encodingdefault}{\sfdefault}{m}{sl}
\SetMathAlphabet{\mathsfit}{bold}{\encodingdefault}{\sfdefault}{bx}{n}

\DeclareMathOperator{\Tr}{Tr}

\newcommand{\cmid}{\!\mid\!}
\DeclareMathOperator*{\argmin}{arg\,min}
\DeclareMathOperator*{\argmax}{arg\,max}

%% file: sections/introduction.tex
\section{Introduction}

Maximum entropy reinforcement learning (MaxEnt-RL) enhances conventional reward maximization with an entropy regularization term, enabling robust decision-making and persistent exploration~\citep{ziebart2008maximum, toussaint2009robot, softqlearning, sac}.
By encouraging stochasticity in the policy, MaxEnt-RL can represent multiple near-optimal behaviors through expressive action distributions~\citep{softqlearning}.
Most existing methods parameterize the policy as a Gaussian distribution due to its optimization efficiency \citep{sac, crossq, bro}.
However, the unimodal nature of Gaussian policies fundamentally limits their ability to capture the complex optimal policies induced by the MaxEnt-RL objective.
To address this expressivity bottleneck, recent studies have introduced diffusion- and flow-based generative policies, achieving strong performance on high-dimensional continuous control benchmarks \citep{dacer, dime}.

Despite their empirical success, these approaches typically require Backpropagation Through Time (BPTT) across multiple generative steps, which can suffer from numerical instability \citep{vanishing-gradient, exploding-gradient}.
To avoid these issues, a separate line of work \citep{maxentdp, rsm, fpmd} adopts an Expectation-Maximization (EM, \citep{emalgorithm}) framework in the vein of MPO \citep{mpo}, casting policy optimization as supervised learning on the target distribution.
While this strategy circumvents BPTT, it still relies on importance sampling (IS, \citep{importance-sampling}) because sampling from the MaxEnt target distribution is intractable.
In high-dimensional spaces, the proposal--target mismatch can cause importance weights to explode, exacerbating the vanishing support problem \citep{awac, thomas2017importance}.
Prior works~\citep{maxentdp, rsm, fpmd, qvpo} heuristically constrain IS weights for stability, but these post-hoc corrections do not directly increase the probability of sampling target-relevant actions.
As a result, IS over the entire action space, which we refer to as \emph{global IS}, remains prone to weight degeneracy, leading to poor sample efficiency and sparse supervision insufficient for tracking non-stationary targets in online RL.
We illustrate this sampling inefficiency with a didactic multigoal example under limited sample budgets $N$ (\Figref{fig:multigoal}), where global IS baselines fail to capture the target mode as $N$ decreases from $32$ to $2$.
We observe the same trend in high-dimensional control tasks (\Figref{fig:5_1}).

\begin{figure}[t]
    \centering
    \includegraphics[width=0.95\textwidth]{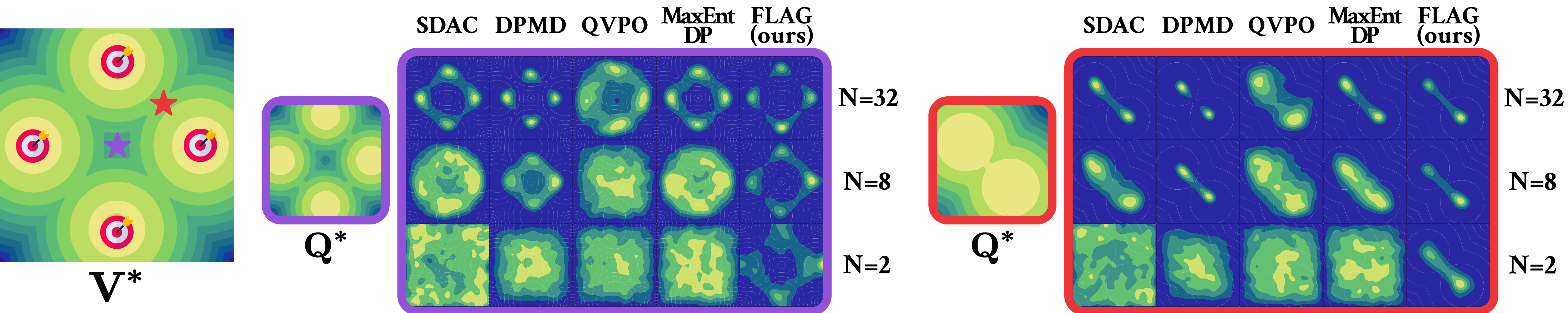}
    \caption{
        \textbf{Comparisons with global policy sampling methods in multi-goal environment.}
        The multi-goal environment \citep{softqlearning} tasks a point mass with navigating to one of four symmetrically placed targets \raisebox{-0.2em}{\includegraphics[height=1.0em]{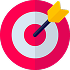}}. 
        We plot the optimal value function $\mathbf{V^*}$ and Q-functions $\mathbf{Q^*}$ at two different states (\raisebox{-0.2em}{\includegraphics[height=1.0em]{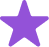}} \& \raisebox{-0.2em}{\includegraphics[height=1.0em]{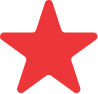}}), along with policies at each state.
        While other methods fail to learn with fewer samples $(N \le 8)$, \textbf{FLAG} captures optimal and multi-modal behaviors leveraging latent-augmented guidance, even in limited importance sample budget commonly arising in high-dimensional action spaces.
        Additional details on the environment and setup, along with statistics that quantify the local-IS mechanism, are presented in Appendix \ref{multigoal}.
        }
    \label{fig:multigoal}
    \vspace{-15pt}
\end{figure}

Our key insight is to mitigate the sparsity of global IS by jointly localizing the proposal and target distributions. By conditioning on a latent variable, both distributions are restricted to the same local region, allowing IS to operate locally rather than over the full action space, which we refer to \emph{local IS}.
To achieve this, we leverage the deterministic nature of flow-matching models, which induces a fixed mapping between latent vectors and actions.
Based on this property, we construct a Gaussian \emph{local policy}
centered on the flow's output and formulate a latent-augmented MDP by extending the state space to include the latent variable, similar to prior generative-policy RL methods~\citep{dppo,reinflow,ncdpo}.
Within this framework, the \emph{global policy} is recovered by marginalizing over the latent space, providing a theoretical foundation for optimization via the local IS rather than the global IS.

Building on this latent-augmented MDP formulation, 
we propose \textbf{FLAG} (\textbf{F}low policy MaxEnt-RL by \textbf{L}atent \textbf{A}ugmented \textbf{G}uidance), which enables flow policy optimization through the standard conditional flow-matching objective \citep{flowmatching}.
Specifically, we introduce a cross-entropy-based proxy MaxEnt objective that avoids the intractable entropy computation of the global policy, and show that it remains consistent with the original MaxEnt objective when the local policy variance is properly controlled during training. 
We then optimize the proxy objective with an EM-style procedure, using the resulting locally updated policies as latent-conditioned supervised targets for the flow policy, which we call \emph{latent-augmented guidance}.
This latent-conditioned local IS allows FLAG to capture the target mode with only $N=2$ samples in \Figref{fig:multigoal}, suggesting better scalability to high-dimensional action spaces (\Figref{fig:5_1}).

In summary, our contributions are threefold:
\begin{itemize}[leftmargin=1.2em, itemsep=2pt, topsep=2pt]
    \item We introduce the latent-augmented MDP to derive local IS and derive a proxy MaxEnt objective on this augmented MDP, with a theoretical consistency with respect to the original objective.
    \item We propose FLAG, which optimizes the proxy MaxEnt objective via an EM algorithm where samples are drawn by local IS, and prove that FLAG monotonically improves the proxy objective.
    \item We empirically demonstrate that FLAG outperforms global IS baselines and achieves the best performance among expressive-policy MaxEnt online RL methods on state-based high-dimensional control tasks.
\end{itemize}

%% file: sections/related_works.tex
\section{Related Works}
\label{section:related_works}

We briefly review the most relevant lines of work. A side-by-side comparison is provided in \Cref{app:extended-related-works}.

\paragraph{Action Gradient.}
DIPO~\citep{dipo} and QSM~\citep{qsm} leverage Q-gradients ($\nabla_a Q$) to construct supervised policy updates, either by refining replay-buffer actions or matching policy scores.
However, their reliance on Gaussian noise exploration motivates a more principled MaxEnt-RL framework.

\paragraph{BPTT-based Actor-Critic.}

Recent methods treat generative models as parameterized policies within actor-critic frameworks.
For example, DACER \citep{dacer, dacerv2} optimizes Q-values along reverse diffusion, while DIME \citep{dime} optimizes Q-values along with a tractable marginal-entropy lower bound. 
However, these objectives require BPTT through generation, making optimization memory-intensive and numerically unstable.
FlowRL \citep{flowrl} avoids BPTT using a single-step policy with implicit optimality guidance, but its objective lacks exploration and can be suboptimal.
EXPO \citep{expo} also avoids BPTT by learning an edit policy over an expressive base policy, but it requires a two-stage training procedure.
In contrast, our method derives latent-augmented guidance and optimizes the proxy MaxEnt-RL objective with supervised updates, bypassing BPTT without sacrificing the expressivity of the underlying generative model.

\paragraph{Target Policy Matching.}

Supervised policy updates in RL have been widely studied under the probabilistic inference framework~\citep{levine2018reinforcement}, including EM-based methods such as RWR~\citep{rwr}, AWR~\citep{awr}, and MPO~\citep{mpo}.
Recent generative-policy extensions, such as RSM~\citep{rsm} and MaxEntDP~\citep{maxentdp}, use reweighted score matching or Q-weighted noise estimation to project policies toward target distributions.
However, they rely on global IS, which disperses samples over the full action space and provides sparse supervision in regions important for policy improvement.
FLAG instead uses local IS to concentrate samples around each action, yielding denser and more effective policy supervision.

\paragraph{Policy Gradient.}
On-policy methods adapt PPO~\citep{ppo} to generative policies.
FPO~\citep{fpo} and GenPO~\citep{dinggenpo} do so by making the likelihood ratio tractable, while DPPO~\citep{dppo} and ReinFlow~\citep{reinflow} fine-tune pretrained policies by unrolling the denoising chain into an augmented MDP, which inflates the effective horizon.
All of these methods require a large number of environment interactions.
FLAG also augments the MDP with the flow latent, but the augmentation localizes the proposal and the target while keeping the effective horizon unchanged, and the policy is trained off-policy.
Our closest off-policy baseline is QVPO~\citep{qvpo}, which optimizes a weighted diffusion loss to improve the policy.
QVPO's weighted score-matching update relies on a global proposal distribution, inheriting the sparsity issue of global IS, while FLAG uses a localized proposal for policy improvement.

%% file: sections/preliminaries.tex
\section{Preliminaries}

\subsection{Reinforcement Learning as Probabilistic Inference}
\label{subsection: RL-as-inference}

\paragraph{Notation.} We consider a Markov Decision Process (MDP) defined by the tuple $\langle \mathcal{S}, \mathcal{A}, p, r, \gamma \rangle$, where $\mathcal{S}$ and $\mathcal{A}$ denote the continuous state and action spaces, $p: \mathcal{S} \times \mathcal{S} \times \mathcal{A} \to \mathbb{R}^+$ represents the transition probability density, $r: \mathcal{S} \times \mathcal{A} \to \mathbb{R}$ is the reward function, and $\gamma \in [0, 1)$ is the discount factor. We define a stochastic policy $\pi: \mathcal{S} \times \mathcal{A} \to \mathbb{R}^{+}$, mapping states to probability densities over actions. In the infinite-horizon setting, $\pi$ induces $p_\pi(\tau)$ and discounted marginals $\rho_\pi(s)$ and $\rho_\pi(s,a)$.

\paragraph{Control-as-inference and MaxEnt-RL.}
RL admits a probabilistic inference view via binary optimality variables $O_t$, where $O_t=1$ indicates $(s_t,a_t)$ is optimal~\citep{levine2018reinforcement}. With $\mathcal{O}=\{O_t=1\}_{t=0}^{\infty}$ and $p(\mathcal{O}\mid\tau)\propto\exp\!\left(\sum_t \gamma^t r(s_t,a_t)/\alpha\right)$ for temperature $\alpha>0$, Bayes' rule gives the optimal posterior 
\begin{equation}
    p^\star(\tau \mid \mathcal O) \propto p(s_0) \prod\nolimits_{t=0}^{\infty}p(a_t\cmid s_t) p(s_{t+1}\cmid s_t,a_t)
    \exp\left(
        \sum\nolimits_{t=0}^{\infty}\gamma^t \big(r(s_t,a_t)/\alpha\big)
    \right).
\end{equation}
Matching $p_\pi(\tau)$ to this posterior via KL-divergence recovers the standard MaxEnt-RL objective~\citep{sac}:
\begin{equation}
\label{eq:maxent-as-inference}
    \pi^\star
    = \argmax_{\pi}\mathbb{E}_{\tau\sim p_\pi}\!\left[\sum\nolimits_{t=0}^{\infty}\gamma^t\big(r(s_t,a_t)-\alpha\log\pi(a_t\cmid s_t)\big)\right].
\end{equation}

\paragraph{EM-style Policy Optimization.}
Introducing a variational policy $q(a\mid s)$ that shares the initial-state distribution and dynamics with $p_\pi$, the trajectory-level KL reduces to action-only terms, yielding the variational lower bound~\citep{RERPI, mpo}
\begin{equation}
\label{eq:em-elbo}
    \mathcal{J}(q,\pi) = \mathbb{E}_{\tau\sim p_q}\!\left[\sum\nolimits_{t=0}^{\infty}\gamma^t\big(r(s_t,a_t)/\alpha\big)\right] - D_{\mathrm{KL}}\!\left(p_q(\tau)\,\|\,p_\pi(\tau)\right).
\end{equation}
For the current parameterized policy $\pi_k=\pi_{\theta_k}$, the KL-constrained E-step
\begin{equation}
\label{eq:em-estep}
    q_k(\cdot\cmid s) = \argmax_{q\in\Delta(\mathcal{A})}\mathbb{E}_{a\sim q}\left[Q^{\pi_k}(s,a)\right]\qquad \text{s.t.}\;\;D_{\mathrm{KL}}\left(q(\cdot\cmid s)\;\|\;\pi_k(\cdot\cmid s)\right)\le\epsilon
\end{equation}
admits a closed form $q_k(a\cmid s)\propto\pi_k(a\cmid s)\exp\left(Q^{\pi_k}(s,a)/\lambda_k\right)$, with $\lambda_k>0$ the dual variable. The M-step projects $q_k$ back into the parametric class, yielding a weighted maximum-likelihood:
\begin{equation}
\label{eq:em-mstep}
    \theta_{k+1}\in\argmax_{\theta}\mathbb{E}_{s\sim\rho_{\pi_k}}\mathbb{E}_{a\sim q_k(\cdot\mid s)}\left[\log\pi_\theta(a\mid s)\right],
\end{equation}
typically estimated via importance weighting on samples from $\pi_k$.

\subsection{Conditional Flow Matching}
\label{subsection:CFM_def}

Flow matching \citep{flowmatching} transports a prior $p_0$ to a target $p_1$. It learns a time-conditioned vector field $v:[0,1] \times \mathbb{R}^d \to \mathbb{R}^d$.
The flow $\psi$ is governed by the following ordinary differential equation (ODE):
\begin{equation}
    \label{eq: flow-map}
    \frac{d}{dt}\psi^t(x)=v^t(\psi^t(x)) \quad \textit{i.e.} \quad \psi^t(x)=\psi^0(x)+\int_0^t v^\tau(\psi^\tau(x)) d\tau.
\end{equation}
By parameterizing the vector field $v$ with a parameter $\theta$ and under the Optimal Transport linear path, the vector field can be directly optimized via the Conditional Flow Matching (CFM) objective:
\begin{align}
\label{eq:CFM-OT-objective}
    \mathcal{L}_\text{CFM}(\theta) = \mathbb{E}_{t, x_0, x_1} \Big[ \big \lVert v^t_\theta \big( (1-t)x_0 + tx_1 \big) - (x_1-x_0) \big\rVert^2 \Big],
\end{align}
where $t \sim \mathcal{U}[0,1]$, $x_0 \sim p_0$, and $x_1 \sim p_1$.

%% file: sections/method.tex
\section{Method}
\label{section: method_main}
FLAG targets the optimization of a flow policy guided by local supervision from a Gaussian local policy.
First, we introduce a policy parameterization (\ref{subsection: policy-parameterization}) that explicitly couples the flow-based anchor actions with a local Gaussian distribution.
Next, we lift the original environment to a latent-augmented MDP (\ref{subsection: z-MDP}).
This augmented construction makes the local conditioning variable explicit and establishes that optimizing the local policy in the lifted MDP is mathematically equivalent to optimizing the global MaxEnt policy.
On top of this equivalence, we derive an EM-based policy improvement procedure (\ref{subsection: local-policy-update}) that admits the latent-augmented guidance.
Finally, we introduce a practical algorithm (\ref{subsection: implementation-details}) that enables updates of the flow policy via the latent-augmented guidance.
All derivations in this section are in Appendix~\ref{app: derivations}.

\subsection{Policy Parameterization}
\label{subsection: policy-parameterization}

Our policy parameterization centers the Gaussian on the flow's anchor $T_\theta(s,z)$, which localizes the effective sampling region around each latent $z$, providing dense supervision for EM updates and circumventes the importance-weight degeneracy of global proposal sampling.
We first define a flow transformation $T_\theta(s, z)$ parameterized by $\theta$ 
that maps a latent vector $z = a^0 \sim p_z$ to an anchor action 
$T_\theta(s, z) = a^1 = z + \int_0^1 v_\theta(a^\tau, \tau, s)\,d\tau$ 
following \Eqref{eq: flow-map}, inducing the base flow policy 
$\tilde\pi_\theta(a \mid s) = p_z(z)\,|\det(\partial T_\theta(s,z)/\partial z)|^{-1}$.

Next, we augment each flow-generated anchor action $T_\theta(s, z)$ with an auxiliary local Gaussian:
\begin{align}
\label{eq:local-policy}
    \hat\pi(a \mid s, z;\,\theta) = \mathcal{N}(a;\, T_\theta(s,z),\, \Sigma).
\end{align}
For theoretical development in this paper, we treat $\Sigma$ as non-learnable, isotropic covariance whose schedule is specified in \Cref{subsection: implementation-details}; 
a learnable variant is discussed in Appendix~\ref{subsection:covariance_learning}.
Marginalizing over the prior recovers the global policy:
\begin{equation}
\label{eq:marginal-policy}
    \pi(a \mid s;\,\theta) = \int p_z(z)\,\hat\pi(a \mid s, z;\,\theta)\,dz.
\end{equation}

\subsection{Latent augmented MDP and Proxy MaxEnt-RL Objective} \label{subsection: z-MDP}
In this section, we show that local policy optimization on the augmented MDP is mathematically equivalent to maximizing the original MaxEnt-RL objective in the original MDP.
We start by defining a latent-augmented MDP, similar to prior generative-policy RL methods that incorporate variables from
the action-generation process into an MDP formulation~\citep{dppo,ncdpo,reinflow}.
\begin{definition}[Latent-augmented MDP]
\label{def: z-mdp}
    The latent augmented MDP $\hat{\mathcal M}$ is defined by the state $\hat s=(s, z)\in \mathcal S\times \mathcal Z$ and action $a \in \mathcal A$. Since $z$ is sampled from a prior $p_z$ which is independent of state, the transition dynamics $\hat p$ and reward function $\hat r$ are given as:
    \begin{equation}
        \hat p\left(\hat s' \mid \hat s, a\right)=p(s'\mid s,a)\,p_z(z'),\qquad \hat r(\hat s,a)=r(s,a).
    \end{equation}
\end{definition}
We hereafter refer to this construction as the \textbf{$\mathbf{z}$-MDP}, and adopt the notation $(\hat{\cdot})$ to distinguish entities in the $z$-MDP from those in the original MDP.
In Corollary~\ref{cor:marginal-consistency}, we show that any expectation over a measurable function is unchanged whether you average over $\hat{\rho}_{\hat{\pi}}(\hat s, z)$ or $\rho_\pi(s, z)$.
\begin{corollary}[Marginal consistency]
\label{cor:marginal-consistency}
    Given $\hat \pi$ and $\pi$ in Section~\ref{subsection: policy-parameterization}, the discounted state--action marginal distribution $\hat\rho_{\hat\pi}(s, z, a)$ in the $z$-MDP $\hat{\mathcal M}$ satisfies $\rho_\pi(s, a) = \int \hat\rho_{\hat\pi}(s, z, a)dz$.
    Consequently, the following equation holds for any measurable function $f:\mathcal S\times \mathcal A\to\mathbb R$
    \begin{equation}
    \label{eq:marginal-consistency}
        \mathbb{E}_{(s,z,a)\sim \hat\rho_{\hat\pi}}\left[f(s,a)\right]
        = \mathbb{E}_{(s,a)\sim \rho_{\pi}}\!\left[f(s,a)\right].
    \end{equation}
\end{corollary}

\paragraph{Cross-entropy.}

The MaxEnt-RL objective requires the intractable entropy of $\pi$, as the marginalization over $z$ in \Eqref{eq:marginal-policy} has no closed form.
To circumvent this, we propose using the cross-entropy between $\pi$ and $\tilde\pi_\theta$, ${\mathcal H}\left(\pi(\cdot\cmid s), \,\tilde\pi_\theta(\cdot\cmid s)\right)\triangleq \mathbb{E}_{a\sim\pi(\cdot\mid s)}\left[ -\log \tilde\pi_\theta(a\mid s)\right]$, which can be evaluated in an unbiased manner via the Hutchinson trace estimator (see Appendix~\ref{app: flow-maxent-rl} for details).
The cross-entropy is reformulated as an expectation over $z$:
\begin{equation}
\label{eq:cross-ent-def-z}
    \mathcal H\left( \pi(\cdot\cmid s), \tilde \pi_\theta(\cdot\cmid s)\right)=\mathbb{E}_{z\sim p_z}\left[ \mathcal H \left(\hat \pi(\cdot\cmid \hat s),\, \tilde \pi_\theta(\cdot\cmid s) \right)\right]
\end{equation}
For brevity, we denote $\tilde \pi_\theta$ by $\tilde \pi$ throughout this section. The following proposition establishes that the cross-entropy converges to the entropy of $\pi$.

\begin{proposition}[Cross-Entropy as a Valid Entropy Surrogate]
\label{prop:surrogate-validity}
    For smooth $\tilde\pi$ and $\mathrm{tr}\left(\Sigma\right) \ll 1$, the local variance governs the Total Variation (TV) distance and KL divergence between $\pi$ and $\tilde\pi$:
    \begin{equation}
        D_{\mathrm{TV}}(\pi, \tilde\pi) = \mathcal{O}\big(\sqrt{\mathrm{tr}(\Sigma)}\big), \qquad
        D_{\mathrm{KL}}(\pi \;\|\; \tilde\pi) = \mathcal{O}\left(\mathrm{tr}(\Sigma)^2\right).
    \end{equation}
    Consequently, as the local variance vanishes ($\operatorname{tr}(\Sigma) \to 0$), the surrogate cross-entropy converges to the true entropy of the global policy:
    \begin{equation}
        \mathcal{H}(\pi, \tilde\pi) = \mathcal{H}(\pi) + \mathcal{O}\left(\mathrm{tr}(\Sigma)^2\right) \approx \mathcal{H}(\pi).
    \end{equation}
\end{proposition}

Proposition~\ref{prop:surrogate-validity} justifies a tractable proxy for the MaxEnt-RL objective in $\mathcal{M}$ via the cross-entropy under a small-variance assumption on the local Gaussian.
We make this assumption hold by design through the covariance scheduling in \Cref{subsection: implementation-details}, which keeps $\sigma_k$ small throughout training.
And we verify in \Cref{app:kl_proxy_accuracy} that $D_{\mathrm{KL}}(\pi \,\|\, \tilde\pi)$ stays small under this schedule.

\paragraph{Proxy MaxEnt-RL Objective.}

Using the cross-entropy, we now define the proxy MaxEnt-RL objective and corresponding soft value functions.
The augmented policy objective is given as
\begin{equation}
\label{eq:surr-policy-obj}
\mathcal{J}(\pi) = \sum\nolimits_{t=0}^{\infty} \gamma^t \mathbb{E}_{(s_t, a_t)\sim \rho_{\pi}}\big[\big( r(s_t, a_t) - \alpha \log \tilde{\pi}(a_t \mid s_t) \big) \big].
\end{equation}
By Corollary~\ref{cor:marginal-consistency} and \Eqref{eq:cross-ent-def-z}, this objective is equivalently evaluated in the $z$-MDP as $\mathcal{J}(\hat{\pi})$, providing a theoretical bridge for optimization within the latent-augmented space $\hat{\mathcal{M}}$.

\paragraph{Value function.} Following SAC~\citep{sac}, we define the soft state-value function and the corresponding soft Bellman operator $\mathcal T^\pi$ by replacing the entropy term with the cross-entropy term:
\begin{alignat}{2}
    &\text{(Soft Value Function)} \quad && V^{\pi}(s) = \mathbb{E}_{a\sim \pi}\left[ Q^\pi(s, a) - \alpha \log \tilde\pi(a\mid s)\right] \label{eq: cross-entropy-soft-v} \\
    &\text{(Soft Bellman Operator)} \quad && (\mathcal T^{\pi}Q)(s, a) \triangleq r(s, a) + \gamma \mathbb{E}_{s'\sim p}\left[ V^\pi(s')\right] \label{eq: cross-entropy-soft-bellman-op}
\end{alignat}
In the $z$-MDP $\hat{\mathcal M}$, the Q-function is defined with $Q^{\hat\pi} \left(\hat s, a \right)$.
\begin{corollary}[Q-function consistency]
\label{cor: q-func-consistency}
By Corollary~\ref{cor:marginal-consistency}, we use the same Q-function both in $\mathcal M$ and $\hat{\mathcal M}$
\begin{equation}
    Q^{\hat \pi}(\hat s, a)=Q^{\pi}(s, a), \quad \forall s,z, a
\end{equation}
Therefore, we use the Q-function $Q^\pi(s, a)$ of the original MDP in z-MDP.
\end{corollary}
According to Corollaries~\ref{cor:marginal-consistency} and~\ref{cor: q-func-consistency}, maximizing the Q-function $Q^\pi(s, a)$ via the local policy $\hat \pi$ for all $\hat s\sim \hat\rho_{\hat\pi}$ in the $z$-MDP is equivalent to optimizing the global policy $\pi$ for all $s\sim \rho_\pi$.

\subsection{FLAG: A Practical Local Policy Update Algorithm} \label{subsection: local-policy-update}
In this section, we extend the EM updates from Section~\ref{subsection: RL-as-inference} to the $z$-MDP and our local 
policy $\hat\pi(\cdot\cmid \hat s;\theta)$. At each iteration $k$, we treat $\tilde\pi_{\theta_k}$ as a fixed reference, making the augmented reward $r_t - \alpha\log\tilde\pi_{\theta_k}(a_t\mid s_t)$ play the role of an iteration-fixed reward, enabling the EM derivation
\begin{equation}
\label{eq:em-objective}
\begin{aligned}
    \mathcal{J}_{\text{EM}}(q, \theta)=  &\mathbb{E}_{\hat\tau\sim p_q} \left[ \sum\nolimits_{t=0}^{\infty} {\gamma^t}\bigl( ( r_t -\alpha \log \tilde \pi_\theta (a_t\cmid s_t) ) / {\lambda}\bigr) \right] - D_{\text{KL}}\bigl(p_q(\hat{\tau}) \,\|\, p_{\hat{\pi}}(\hat{\tau})\bigr).
\end{aligned}
\end{equation}
Crucially, the cross-entropy structure allows us to define and evaluate this objective for a variational distribution $q$, which admits E-step target in a closed form.
For a reference policy $\hat \pi_k$ parameterized by $\theta_k$ and its associated $Q$-function $Q^{\pi_k}$, we define the energy function for brevity
\begin{equation}
\label{eq: energy-func}
    f_{\hat s, k}(a):=Q^{\pi_k}(s, a)-\alpha\log \tilde \pi_{\theta_k}(a\mid s).
\end{equation}
For each $\hat s\sim \hat \rho_{\hat\pi_k}$,
\begin{alignat}{2}
    &\text{({E-step:})} \qquad &&  q_k(a \mid \hat{s}) \propto \hat{\pi}(a \mid \hat{s}; \theta_k) \exp(f_{\hat{s}, k}(a) / \lambda)
    \label{eq: e-step-target} \\
    &\text{({M-step:})} \qquad && \theta_{k+1} \in \argmax_{\theta} \mathbb{E}_{a \sim q_k} \big[\log \hat{\pi}(a \mid \hat{s}; \theta)\big]  \label{eq:M-step-final}
\end{alignat}

\paragraph{M-step via Moment Matching.}

Since $\hat\pi$ is Gaussian with fixed $\Sigma_k$, the M-step in 
\Eqref{eq:M-step-final} reduces to matching the first moment :
\begin{equation}
\label{eq: M-step-regression}
    \theta_{k+1} \in \argmin_\theta \mathbb{E}_{(s, z) \sim \hat \rho_{\hat \pi}}\| T_\theta(s,z) - \mu_k^*(\hat s) \|^2,
    \quad \mu_k^*(\hat s) := \mathbb{E}_{q_k}[a].
\end{equation}
We approximate $\mu_k^*$ via self-normalized importance sampling \citep{bishop2006pattern} 
with $N$ samples $\{\delta_i\}_{i=1}^{N}\sim \hat\pi_k(\cdot \cmid \hat s)$:
\begin{equation}
\label{eq:SNIS-moment-matching}
    \mu_k^*(\hat s)
    = \mathbb{E}_{q_k}[a]
    \approx \sum\nolimits_{i=1}^N \bar w_i a_i,
    \quad \bar w_i = \operatorname{softmax} \big(f_{\hat s, k}(a + \delta_i)/{\lambda} \big).
\end{equation}
\begin{wrapfigure}{r}{0.35\textwidth}
    \vspace{-1em}
    \centering
    \includegraphics[width=\linewidth]{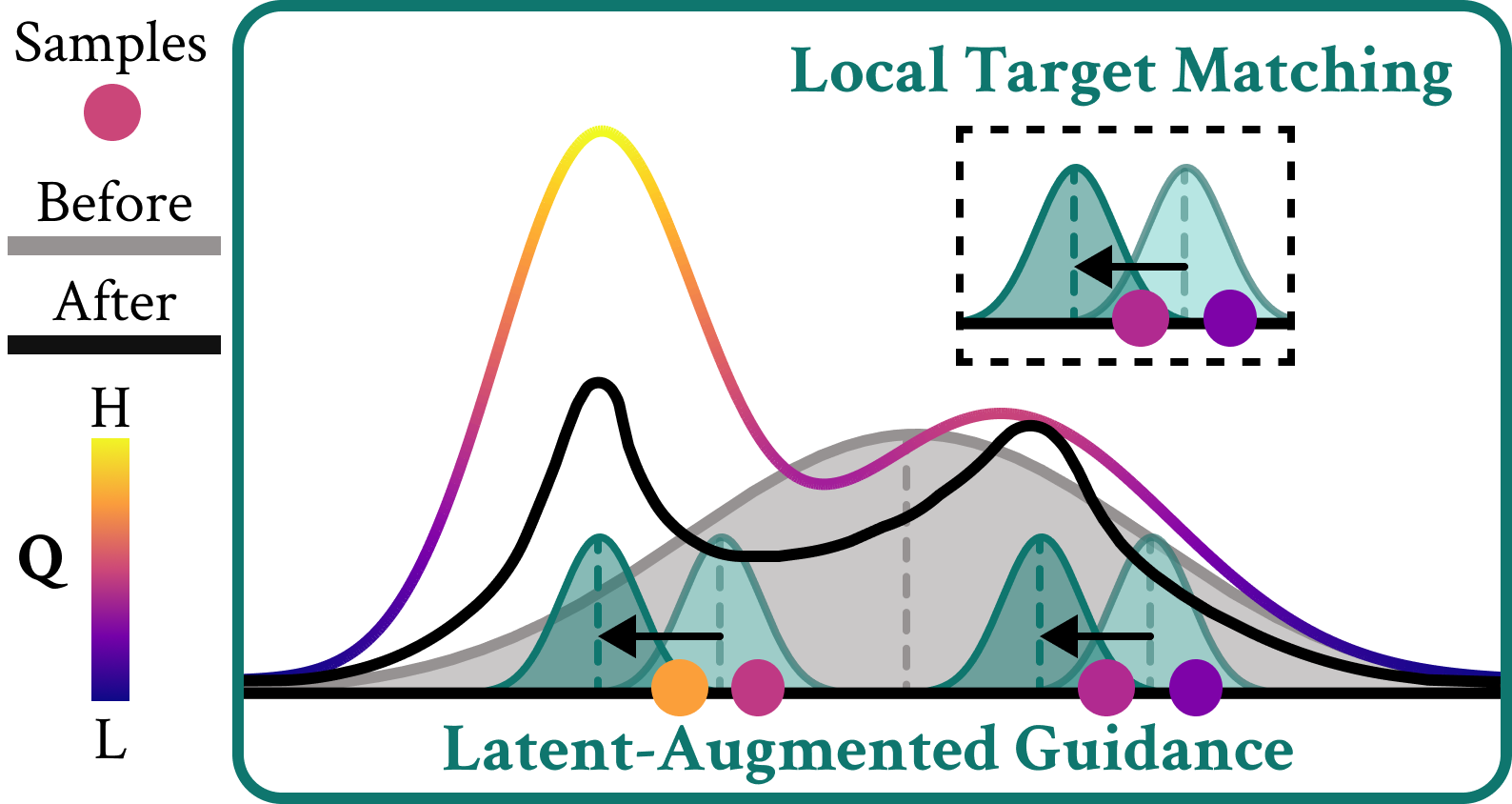}
    \caption{
    \textbf{Visualization of FLAG.}
    Target matching is performed on the local policy and target actions are distilled back to the flow policy.
    }
    \label{fig:4.3}
    \vspace{-1.5em}  %
\end{wrapfigure}
We treat $\mu_k^*(\hat s)$ as an improved action label for the anchor action $T_{\theta_k}(s, z)$. To realize \Eqref{eq: M-step-regression}, we distill these labels into the base flow policy $\tilde\pi_{\theta_k}$ using the CFM objective in \Eqref{eq:CFM-OT-objective}, i.e.,
\begin{equation}
\label{eq: M-step-CFM}
    \mathcal L(\theta;s, z)=\mathbb{E}_{\tau}\left[ \| u_\theta^\tau(s, \zeta^\tau(z, \mu^*)) -(\mu^*-z)\|^2\right].
\end{equation}
This label acts as a \emph{guiding flag}, as it indicates the latent vector $z$ where to 
generate in the action space (\Figref{fig:4.3}).

\paragraph{Monotonic improvement guarantee.}
We now show that FLAG's EM update monotonically improves the variational 
objective, $\mathcal J_{k+1} \geq \mathcal J_k$, where $\mathcal J_k := \mathcal J_{\text{EM}}(q_k, \theta_k)$.
Let $\mathcal G_k$ denote the squared gradient norm of the \Eqref{eq: M-step-regression} and the local covariance be $\Sigma_k = \sigma_k^2 I$.

\begin{theorem}[Monotonic Improvement Guarantee]
\label{thm:flag-monotonic-improvement}
Let the CFM update (\Eqref{eq: M-step-CFM}) approximates the ideal KL projection 
(\Eqref{eq: M-step-regression}) up to an error $\epsilon_k^{\mathrm{proj}}$.
For sufficiently small step size $\beta$ the one-step improvement of the FLAG 
update is lower-bounded by
\begin{equation}
\label{eq: thm-flag-monotonic-improvement}
    \mathcal J_{k+1} - \mathcal J_k 
    \;\geq\; 
    \underbrace{\lambda \beta \mathcal G_k}_{\substack{\text{frozen-reward}\\\text{MPO improvement}}}
    \;-\;
    \underbrace{\alpha\, C_\Sigma \sigma_k^2\, \beta \mathcal G_k}_{\substack{\text{cross-entropy}\\\text{reward drift}}}
    \;-\;
    \underbrace{\lambda\, \epsilon_k^{\mathrm{proj}}}_{\substack{\text{CFM projection}\\\text{error}}}
    \;+\; \mathcal{O}(\beta^2),
\end{equation}
where $C_\Sigma > 0$ is a constant absorbing zeroth-order approximation residuals.
Consequently, whenever $\lambda > \alpha C_\Sigma \sigma_k^2$ and $\epsilon_k^{\mathrm{proj}}$ is sufficiently small, $\mathcal J_{k+1} \geq \mathcal J_k$.
\end{theorem}

The three terms in Theorem~\ref{thm:flag-monotonic-improvement} correspond to the standard
MPO improvement under frozen reward, the drift of the cross-entropy reward as $\theta_k$
changes, and the approximation error introduced by the CFM realization of the KL projection
(see \Figref{fig:flag-mpo-flow} for the proof overview).
The design choices in \Cref{subsection: implementation-details} are intended to keep the sufficient conditions $\lambda > \alpha C_\Sigma \sigma_k^2$ and $\epsilon_k^\text{proj}\to 0$ satisfied in practice, and we measure both quantities empirically in \Cref{app:theorem_diagnostics}.
While Theorem~\ref{thm:flag-monotonic-improvement} follows a variational
framework, \Figref{fig:sac_vs_flag} and Appendix~\ref{app: cov-schedule-justification}
provide a complementary derivation showing how FLAG's moment-matching update
recovers SAC's soft policy improvement.

\subsection{Implementation Details} \label{subsection: implementation-details}

\textbf{Autotuning Temperature.}
We follow \citep{sac} to automatically tune the entropy scaling term $\alpha$:
\begin{equation}
\label{eq: alpha-objective}
    \mathcal{L}(\alpha) = \alpha ( \mathcal H ( \pi, \tilde\pi ) - \mathcal H_\text{target} ),
\end{equation}
where $\mathcal H_\text{target}$ is a hyperparameter that determines the stochasticity of the policy.

\textbf{Effective Temperature.}
We normalize the energy function as $f_{\mathrm{norm}}(a) = f_{\hat s, k}(a) / \alpha$ 
with a fixed reference temperature $\lambda_{\mathrm{ref}}$ 
\citep{minimalist_approach, BEAR, flowrl}, equivalent to using an effective 
temperature $\lambda = \alpha \lambda_{\mathrm{ref}}$ in \Eqref{eq: e-step-target}.
Fixing $\lambda_{\mathrm{ref}}$ avoids the numerical instability of 
optimizing two Lagrange multipliers $\alpha$ and $\lambda$, and reduces the sufficient 
condition for monotonic improvement 
(Theorem~\ref{thm:flag-monotonic-improvement}) to 
$\lambda_{\mathrm{ref}} > C_\Sigma \sigma_k^2$.

\textbf{Covariance Scheduling.}
We use an isotropic covariance $\Sigma_k = \sigma_k^2 I$.
$\sigma_k$ is annealed from $\sigma_{\mathrm{init}}$ to a small $\sigma_{\mathrm{final}}$, keeping $\sigma_k$ small throughout training.
This scheduling is crucial in Theorem~\ref{thm:flag-monotonic-improvement}: since the reward-drift term scales as $\mathcal{O}(\sigma_k^2)$, annealing $\sigma_k \to 0$ tightens the guarantee.

\textbf{Guidance Buffer.}
We cache locally improved action labels $\mu_k^*(\hat s)$ in a small guidance buffer 
and reuse them across multiple off-policy updates 
\citep{TRPO, tomar2020mirror}. Since the E-step targets drift slowly across iterations, 
recent labels remain effective targets for the current policy update, providing dense 
supervision that keeps the CFM projection error $\epsilon_k^{\mathrm{proj}}$ in 
Theorem~\ref{thm:flag-monotonic-improvement} small in practice.

\textbf{Other Details.}
Following DIME~\citep{dime}, we employ a distributional critic~\citep{distributionalrl} 
trained via CrossQ~\citep{crossq}. Full implementation details and the 
training procedure are provided in Appendix~\ref{app:implementation_details} and 
Algorithm~\ref{algorithm::FLAG}.

\begin{tcolorbox}[
    enhanced,
    colback=myteal!10,
    colframe=myteal!95,
    colbacktitle=myteal!95,
    coltitle=white,
    title=\textbf{FLAG realizes local IS through EM on the $\bm{z}$-MDP.},
    fonttitle=\bfseries\normalsize,
    boxrule=0.8pt,
    arc=2mm,
    outer arc=2mm,
    left=9pt,
    right=9pt,
    top=5pt,
    bottom=5pt,
    toptitle=2pt,
    bottomtitle=2pt,
    lefttitle=9pt,
    righttitle=9pt,
    titlerule=0pt,
    before skip=4pt,
    after skip=4pt,
]
The $z$-MDP augments the original MDP with a latent variable while preserving marginal and $Q$-function consistency. On this augmented MDP, we define a tractable proxy MaxEnt-RL objective whose EM-style optimization induces a localized proposal-target pair conditioned on $z$. FLAG instantiates this procedure, enabling the local IS within the localized region of the action space and making MaxEnt-RL via supervised learning scalable.
\end{tcolorbox}

%% file: sections/experiments.tex
\section{Experimental Results} \label{section:experiments}

We answer three questions about FLAG in this section.
\begin{itemize}[leftmargin=2.5em, itemsep=2pt, topsep=2pt]
    \item[\textbf{(Q1)}] \Cref{subsection: global policy sampling}: 
    Does FLAG scale to high-dimensional action spaces under limited sample budgets?
    \item[\textbf{(Q2)}] \Cref{subsection: comparison_with_diff-flow-policy}: 
    How does FLAG compare to action-gradient and BPTT-based actor-critic methods?
    \item[\textbf{(Q3)}] \Cref{subsection:key_design_choices}: 
    How do our key design choices, covariance scheduling and the guidance buffer, connect to the theoretical results?
\end{itemize}
We defer additional ablation studies and analyses to Appendix~\ref{app: additional_ablation_studies} and \ref{app:additional_analyses}, respectively.
Hyperparameters and experiment details are deferred to Appendix~\ref{app:experiment_details}.

\textbf{Benchmarks and Metrics.}
We evaluate online RL algorithms on three challenging benchmark suites: DMC suite~\citep{dmc}, MyoSuite~\citep{myosuites}, and MuJoCo~\citep{todorov2012mujoco}. MuJoCo is included because target-matching policy optimization methods are commonly evaluated in this domain. We evaluate on four tasks each: HalfCheetah, Walker2d, Ant, and Humanoid for MuJoCo-v5;
Dog-Stand, Dog-Walk, Dog-Trot, and Dog-Run for DMC Dog;
and Reach-Hard, Obj-Hold-Hard, Key-Turn-Hard, and Pen-Twirl-Hard for MyoSuite. To compare results across benchmarks with different reward scales, we report normalized scores: MuJoCo scores are normalized by CrossQ performance~\citep{crossq}, DMC Dog returns are divided by the maximum return of 1{,}000, and MyoSuite scores are reported as success rates.

\subsection{FLAG is Scalable to High-dimensional Action Spaces and Robust to Sample Budget} \label{subsection: global policy sampling}

\emph{Is FLAG scalable to high-dimensional action space?}

\begin{figure}
    \centering
    \includegraphics[width=0.9\linewidth]{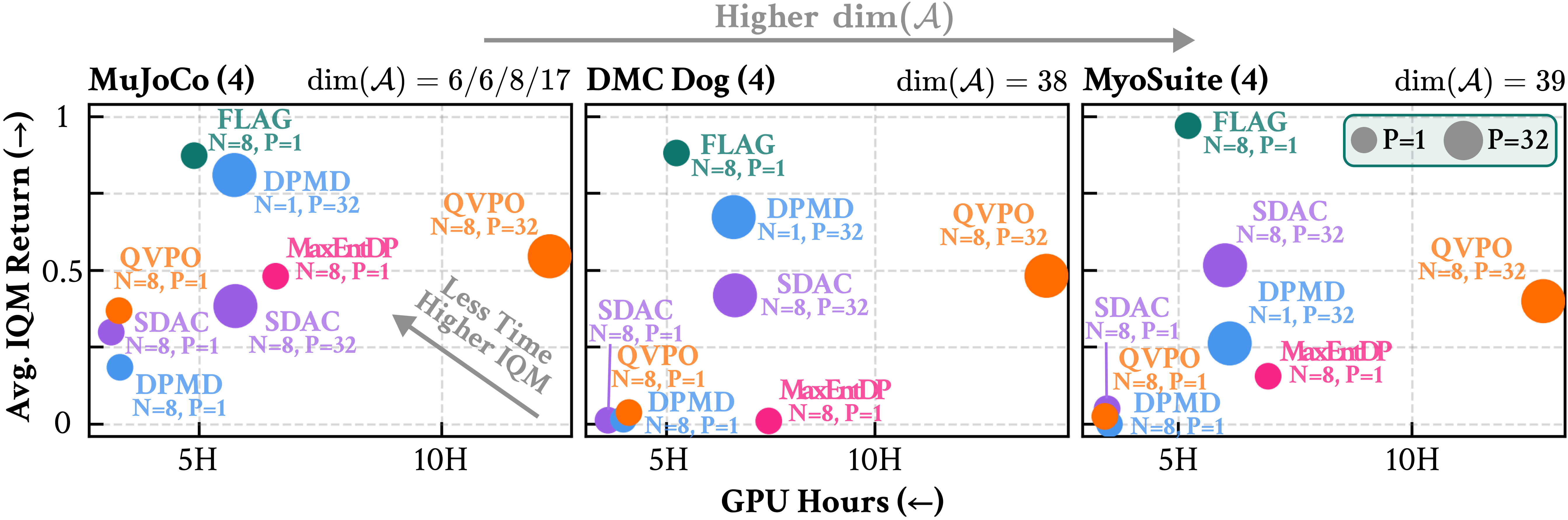}
    \caption{
    \textbf{Performance and computational efficiency of global and local proposal sampling.}
    We plot normalized performance against wall-clock GPU runtime to examine the performance--efficiency trade-off.
    The $x$-axis denotes the time (hours) for a single 1M-step training run on an NVIDIA L40S GPU,
    and the $y$-axis reports aggregated IQM returns across 10 random seeds at 1M steps.
    A method in the \textit{upper-left} region is both higher-performing and takes less time to train.
    P denotes the number of samples in the best-of-P heuristic, and is indicated by the diameter of each marker.
    FLAG consistently occupies the upper-left region across all domains, demonstrating superior performance \textit{without} additional computational overhead.
    }
    \label{fig:5_1}
    \vspace{-1em}
\end{figure}

To answer the question, we first compare FLAG with the global IS baselines:
SDAC and DPMD \citep{rsm}, QVPO \citep{qvpo}, and MaxEntDP \citep{maxentdp}. This comparison highlights the limited scalability of the global IS in high-dimensional action spaces under the same conditions.
QVPO is included as a baseline since it also obtains samples from global proposal distribution to estimate a lower bound on the policy gradient.
All methods use the same CrossQ \citep{crossq} critic with distributional critics \citep{distributionalrl}, and the number of Q-function evaluations per policy update (N) is fixed to 8 due to computational cost.
Our main comparison sets the best-of-P budget to P$=1$, where best-of-P selects the highest-Q action among P sampled candidates, to isolate the effect of the proposal distribution. 
We additionally report $\text{P}=32$ variants where applicable.\footnote{
The $\text{P}=32$ variant is not applied to MaxEntDP because its policy-update candidates are drawn from the replay buffer, making best-of-P selection inapplicable.
}

\begin{wrapfigure}{r}{0.32\textwidth}
    \vspace{-1.2em}
    \centering
    \includegraphics[width=\linewidth]{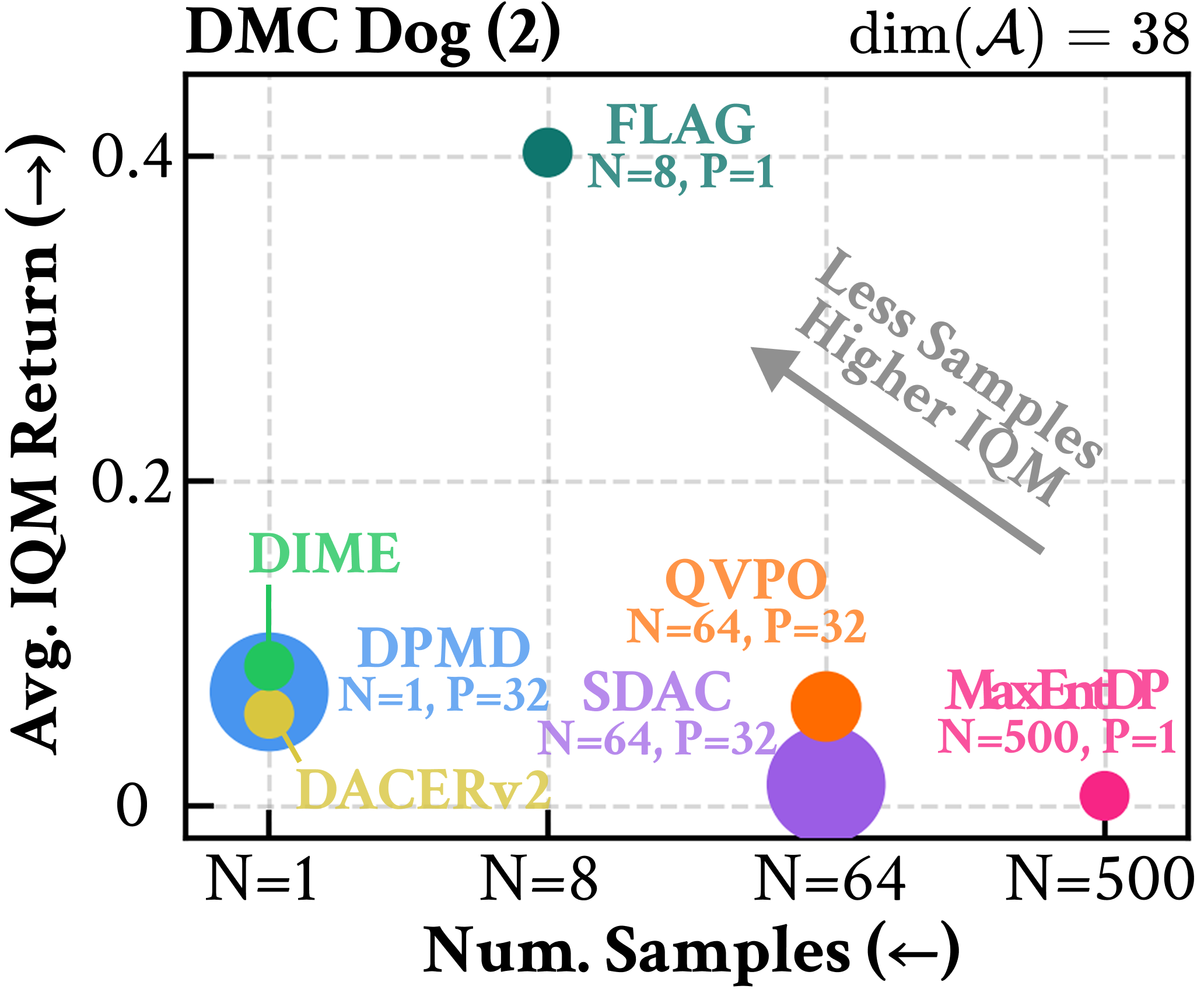}
    \caption{ \textbf{Additional Comparison without CrossQ.}
    Average IQM return across DMC Dog-run and Dog-trot (10 seeds for $\text{P}$=1; 5 for $\text{P}$=32).
    See \Cref{tab:dog_trot_run_performance_no_crossq} and Appendix~\ref{app: subsub-fig4} for details
    }
    \label{fig:wo_crossq}
    \vspace{-1.5em}
\end{wrapfigure}

As shown in \Cref{fig:5_1}, global-proposal baselines struggle in high-dimensional action spaces, especially without best-of-P.
Even with the stronger $\text{P}=32$ heuristic, their performance remains substantially below FLAG.
In contrast, FLAG achieves consistently higher returns with the same Q-function evaluation budget ($\text{N}=8$), indicating that local proposal matching is the key mechanism enabling scalable target matching in high-dimensional control. We also present the learning curves in \Figref{fig:5_1_learning_curves_p1} and \Figref{fig:5_1_learning_curves_p32}.

To rule out the possibility that FLAG's advantage comes from the use of CrossQ or from an artificially restricted Q-function inference budget, we evaluate methods without CrossQ.
In this setting, the baselines use their default, larger sample budgets.
As shown in \Figref{fig:wo_crossq}, the global-proposal baselines still fail to learn meaningful behaviors on challenging DMC Dog tasks, whereas FLAG remains effective.
This suggests that the performance gap is not explained by the value-learning architecture or the evaluation budget, but by the proposed local proposal matching strategy.

\emph{Is FLAG robust to the sample budget?}

\begin{wraptable}{r}{0.38\textwidth}
\vspace{-1.2em}
\centering
\caption{\textbf{Ablation study on the number of training samples in DMC Dog tasks.} We aggregate 15 random seeds at 1M steps. The \capbest{} and \capsecond{} scores are highlighted.}
\label{tab:ablation_num_training_samples_dmc_dog}

\footnotesize
\renewcommand{\arraystretch}{0.95}
\newcommand{\unithead}[1]{\textcolor{statgray}{\scriptsize #1}}
\setlength{\tabcolsep}{6pt}

\begin{tabular}{@{}c c c@{}}
\toprule
\makecell{\textbf{Num.} \\ \textbf{Samples}} &
\makecell{\textbf{Dog-trot} \\ \unithead{Return (1k)}} &
\makecell{\textbf{Dog-run} \\ \unithead{Return (1k)}} \\
\midrule
$\text{N}=2$  & \res{0.853}{0.064} & \res{0.621}{0.051} \\
$\text{N}=4$  & \res{0.898}{0.049} & \res{0.642}{0.035} \\
$\text{N}=8$  & \secondnum{\res{0.922}{0.038}} & \secondnum{\res{0.686}{0.032}} \\
$\text{N}=16$ & \res{0.917}{0.014} & \bestnum{\res{0.691}{0.055}} \\
$\text{N}=32$ & \bestnum{\res{0.923}{0.014}} & \res{0.677}{0.036} \\
\bottomrule
\end{tabular}
\vspace{-1em}
\end{wraptable}

We evaluate the sensitivity of FLAG to the sample budget.
As shown in \Cref{tab:ablation_num_training_samples_dmc_dog}, FLAG remains robust even with limited samples ($\text{N}=2$), exhibiting only mild performance degradation under limited sampling budgets.
This robustness provides empirical evidence for our main design principle: by localizing both the proposal and target distributions, FLAG reduces their discrepancy and makes importance sampling effective within a restricted region of the action space.
In contrast to global IS, whose samples are spread over the full action space and may have little overlap with the target density, local IS operates between two distributions supported on the same local region.
As a result, the importance weights remain informative even with few samples.

\subsection{FLAG Outperforms BPTT-based Actor-Critic and Action Gradient Methods} \label{subsection: comparison_with_diff-flow-policy}

FLAG does not rely on critic gradients, making it robust even when critic gradient signals are unreliable.
To investigate this, we compare methods in two settings: with and without CrossQ.
For BPTT baselines, we use DIME \citep{dime}, FlowRL \citep{flowrl}, and DACERv2 \citep{dacerv2},
and for action-gradient baselines, we use DIPO \citep{dipo} and QSM \citep{qsm}.
Following \Cref{subsection: global policy sampling}, we match critics using CrossQ and distributional
critics across all methods\footnote{Since DACERv2 trains critic to predict the mean and standard deviation of the Q-value, we exclude distributional critics.}. 
As shown in \Cref{tab:table4_dmc_myo_split}, FLAG outperforms all baselines across high-dimensional benchmarks.
To further examine FLAG's robustness to critic quality, we conduct experiments without
CrossQ (\Figref{fig:wo_crossq}) \footnote{We only include DACERv2 and DIME in this comparison, as these methods achieve
comparable performance in \Cref{tab:table4_dmc_myo_split}.}.
While gradient-based methods are susceptible to noisy signals from an insufficiently trained critic,
FLAG aggregates critic signals through importance-weighted averaging rather than direct differentiation,
smoothing out noise in the learning target and thereby maintaining stable performance.

\begin{table}[t]
\centering
\caption{
\textbf{IQM final performance after 1M environment steps on DMC Dog and MyoSuite.}
All results are computed from evaluation collected over 10 seeds with 5 evaluation episodes.
Each entry reports the IQM point estimate with a 95\% confidence interval;
the $\pm$ notation is used only as a compact summary of the confidence interval (\citep{simba-v2}).
The \capbest{} and \capsecond{} IQM point estimates are highlighted in each column, respectively.
}
\label{tab:table4_dmc_myo_split}

\footnotesize
\renewcommand{\arraystretch}{0.88}
\setlength{\tabcolsep}{3.2pt}

\begin{tabular}{@{}lllccccc@{}}
\toprule
& & & \multicolumn{5}{c}{\textsc{DMC Dog} (4)} \\
\cmidrule(l){4-8}
\textbf{Policy}
& \textbf{Method}
& \textbf{Alg.}
& \textbf{Dog-stand}
& \textbf{Dog-walk}
& \textbf{Dog-trot}
& \textbf{Dog-run}
& \textbf{Avg.} \\
\midrule

Gauss.
& -
& CrossQ
& \res{0.957}{0.121}
& \res{0.882}{0.079}
& \second{0.890}{0.016}
& \res{0.467}{0.080}
& 0.799 \\

\midrule

\multirow{6}{*}{Express.}
& \multirow{2}{*}{\makecell[l]{Action Gradient}}
& DIPO
& \res{0.906}{0.030}
& \res{0.570}{0.164}
& \res{0.404}{0.080}
& \res{0.264}{0.017}
& 0.536 \\

&
& QSM
& \res{0.061}{0.131}
& \res{0.137}{0.144}
& \res{0.142}{0.053}
& \res{0.064}{0.055}
& 0.101 \\

\cmidrule(lr){2-8}

& \multirow{3}{*}{\makecell[l]{BPTT-based \\ Actor-Critic}}
& DACERv2
& \res{0.923}{0.031}
& \res{0.726}{0.258}
& \res{0.367}{0.226}
& \res{0.090}{0.108}
& 0.526 \\

&
& DIME
& \second{0.968}{0.017}
& \second{0.944}{0.024}
& \second{0.886}{0.050}
& \second{0.629}{0.056}
& \secondnum{0.857} \\

&
& FlowRL
& \second{0.964}{0.014}
& \res{0.910}{0.011}
& \res{0.870}{0.018}
& \res{0.452}{0.036}
& 0.799 \\

\cmidrule(lr){2-8}

& Target Matching
& FLAG (Ours)
& \best{0.974}{0.016}
& \best{0.950}{0.011}
& \best{0.912}{0.044}
& \best{0.680}{0.062}
& \bestnum{0.879} \\

\midrule
& & & \multicolumn{5}{c}{\textsc{MyoSuite} (4)} \\
\cmidrule(l){4-8}
\textbf{Policy}
& \textbf{Method}
& \textbf{Alg.}
& \textbf{Reach-Hard}
& \textbf{Obj-Hold-Hard}
& \textbf{Key-Turn-Hard}
& \textbf{Pen-Twirl-Hard}
& \textbf{Avg.} \\
\midrule

Gauss.
& -
& CrossQ
& \res{0.600}{0.217}
& \best{1.000}{0.100}
& \res{0.900}{0.150}
& \res{0.100}{0.133}
& 0.650 \\

\midrule

\multirow{6}{*}{Express.}
& \multirow{2}{*}{\makecell[l]{Action Gradient}}
& DIPO
& \res{0.600}{0.219}
& \res{0.340}{0.201}
& \res{0.560}{0.408}
& \res{0.760}{0.174}
& 0.565 \\

&
& QSM
& \res{0.089}{0.069}
& \res{0.000}{0.000}
& \res{0.209}{0.153}
& \res{0.026}{0.048}
& 0.081 \\

\cmidrule(lr){2-8}

& \multirow{3}{*}{\makecell[l]{BPTT-based \\ Actor-Critic}}
& DACERv2
& \res{0.800}{0.107}
& \second{0.925}{0.157}
& \res{0.911}{0.100}
& \second{0.800}{0.129}
& 0.859 \\

&
& DIME
& \second{0.900}{0.167}
& \best{1.000}{0.017}
& \best{1.000}{0.000}
& \res{0.700}{0.233}
& \secondnum{0.900} \\

&
& FlowRL
& \res{0.000}{0.008}
& \res{0.017}{0.042}
& \res{0.000}{0.000}
& \res{0.050}{0.067}
& 0.017 \\

\cmidrule(lr){2-8}

& Target Matching
& FLAG (Ours)
& \best{0.930}{0.078}
& \best{1.000}{0.000}
& \second{0.930}{0.142}
& \best{0.870}{0.174}
& \bestnum{0.933} \\

\bottomrule
\end{tabular}
\vspace{-1em}
\end{table}

\subsection{FLAG Key Design Choices Align with Theoretical Results} \label{subsection:key_design_choices}

While Theorem~\ref{thm:flag-monotonic-improvement} rests on idealized assumptions and is not a guarantee for the neural implementation, it provides concrete guidance for our design choices, and the quantities it isolates can be measured in practice.
The bound identifies two quantities that the implementation can control:
\textbf{(i)} the covariance-dependent cross-entropy reward drift and
\textbf{(ii)} the CFM projection error $\epsilon_k^{\mathrm{proj}}$.
FLAG controls them with covariance scheduling and a guidance buffer, respectively (Section~\ref{subsection: implementation-details}).
In this section, we first measure both quantities under these design choices and show that they remain small in practice (Diagnosis), and then ablate each design choice to examine its effect on performance (Ablation).

\paragraph{Diagnosis}

To measure \textbf{(i)} and \textbf{(ii)}, we compare one FLAG update, which projects the E-step targets with CFM (\Eqref{eq: M-step-CFM}), with one direct KL-projection update (\Eqref{eq: M-step-regression}) taken from the same parameters $\theta_k$ on the same batch $\{\hat s_i\}_{i=1}^{B}$.
To mitigate the BPTT-related gradient scaling issue, we normalize the KL gradient to match the CFM gradient norm.
Let $\mathcal{L}_{\mathrm{KL}}(\theta) = \frac{1}{B}\sum_{i=1}^{B} D_{\mathrm{KL}}\big(q_k(\cdot\mid\hat s_i)\,\|\,\hat\pi(\cdot\mid\hat s_i;\theta)\big)$ denote the KL-projection objective on this batch, and let $\Delta\mathcal{L}_{\mathrm{KL}}(\theta) = \mathcal{L}_{\mathrm{KL}}(\theta_k) - \mathcal{L}_{\mathrm{KL}}(\theta)$ be how much an update $\theta_k \to \theta$ decreases it.
We report both quantities relative to $\lambda\,\Delta\mathcal{L}_{\mathrm{KL}}(\theta^{\mathrm{KL}})$:
\begin{equation*}
    e_{\mathrm{proj}} = 1 - \frac{\Delta\mathcal{L}_{\mathrm{KL}}(\theta^{\mathrm{CFM}})}{\Delta\mathcal{L}_{\mathrm{KL}}(\theta^{\mathrm{KL}})},
    \qquad
    e_{\mathrm{drift}} = \frac{\alpha\,\mathbb{E}\big[\,|\log\tilde\pi_{\theta^{\mathrm{CFM}}}(a\mid s) - \log\tilde\pi_{\theta_k}(a\mid s)|\,\big]}{\lambda\,\Delta\mathcal{L}_{\mathrm{KL}}(\theta^{\mathrm{KL}})}.
\end{equation*}

\begin{wraptable}{r}{0.5\textwidth}
\centering
\caption{\textbf{Diagnosis of Theorem~\ref{thm:flag-monotonic-improvement} on Walker2d.} Each term as a percentage of the improvement of an ideal KL-projection update; Retained $= 100\% - e_{\mathrm{proj}} - e_{\mathrm{drift}}$. Mean $\pm$ standard deviation over 3 seeds.}
\label{tab:theorem_diagnostics}
\scriptsize
\setlength{\tabcolsep}{2.5pt}
\renewcommand{\arraystretch}{0.95}
\resizebox{\linewidth}{!}{%
\begin{tabular}{@{}lcccc@{}}
\toprule
& \textbf{0.1M} & \textbf{0.4M} & \textbf{0.7M} & \textbf{1M} \\
\midrule
$e_{\mathrm{proj}}$ (\%) $\downarrow$  & 6.4 $\pm$ 1.5  & 6.3 $\pm$ 5.7  & 18.7 $\pm$ 3.9 & 34.4 $\pm$ 5.2 \\
$e_{\mathrm{drift}}$ (\%) $\downarrow$ & 1.6 $\pm$ 0.2  & 1.8 $\pm$ 0.2  & 1.7 $\pm$ 0.3  & 1.4 $\pm$ 0.2 \\
\midrule
Retained (\%) $\uparrow$               & 92.0 $\pm$ 1.5 & 91.9 $\pm$ 5.6 & 79.6 $\pm$ 3.7 & 64.2 $\pm$ 5.0 \\
\bottomrule
\end{tabular}%
}
\vspace{-1em}
\end{wraptable}
Here, $e_{\mathrm{proj}}$ is the share of this improvement lost because FLAG uses CFM instead of the ideal projection, and $e_{\mathrm{drift}}$ is the change of the cross-entropy reward caused by the FLAG update, as a share of the same improvement.
They are thus the empirical counterparts of the CFM projection error and the cross-entropy reward drift in Theorem~\ref{thm:flag-monotonic-improvement}, and an update still improves the objective as long as $e_{\mathrm{proj}} + e_{\mathrm{drift}} < 100\%$.
\Cref{tab:theorem_diagnostics} reports both quantities on MuJoCo Walker2d (protocol in \Cref{app:theorem_diagnostics}).
On average, the reward drift stays below $2\%$ throughout training, and although the projection error grows from $6\%$ to $34\%$, FLAG updates still retain at least $64\%$ of the ideal improvement.
Hence, under FLAG's design choices, both quantities empirically remain small, thereby supporting the monotonic improvement of Theorem~\ref{thm:flag-monotonic-improvement}.

\paragraph{Ablation}
Having confirmed that both quantities remain small under FLAG's design, we next ablate the design choices that control them: (a) covariance scheduling for the reward drift and (b) the guidance buffer for the projection error.

\begin{table}[H]
\vspace{-1.2em}
\centering
\caption{\textbf{Ablation study on buffer size and covariance scheduling.} We report returns normalized by 1k on DMC Dog-run. $\dagger$ denotes the hyperparameters used in Section~\ref{subsection: global policy sampling} and \ref{subsection: comparison_with_diff-flow-policy}. All experiments are conducted with 5 random seeds. The \capbest{} and \capsecond{} scores are highlighted in each column, respectively.}
\label{tab:ablation_buffer_covariance}

\footnotesize
\renewcommand{\arraystretch}{0.95}
\setlength{\tabcolsep}{3.3pt}

\begin{tabular}{@{}l *{5}{c} *{8}{c}@{}}
\toprule
&
\multicolumn{5}{c}{\textbf{Buffer Size}} &
\multicolumn{8}{c}{\textbf{Covariance Scheduling }$\log\sigma_{\mathrm{init}} \;(\log\sigma_{\mathrm{final}})$} \\
\cmidrule(lr){2-6}
\cmidrule(l){7-14}
&
\textbf{0} &
\textbf{10.24k$^\dagger$} &
\textbf{51.2k} &
\textbf{102.4k} &
\textbf{204.8k} &
\textbf{-1(-1)} &
\textbf{-1(-2)} &
\textbf{-1(-3)} &
\textbf{-2(-2)} &
\textbf{-2(-3)$^\dagger$} &
\textbf{-2(-4)} &
\textbf{-2(-5)} &
\textbf{-2(-6)} \\
\midrule
Return (1k) $\uparrow$
& 0.601
& \bestnum{0.680}
& \secondnum{0.670}
& 0.618
& 0.589
& 0.614
& 0.675
& \bestnum{0.732}
& 0.588
& \secondnum{0.680}
& 0.597
& 0.547
& 0.269 \\
\bottomrule
\end{tabular}
\vspace{-1.2em}
\end{table}

\paragraph{(a) Covariance Scheduling}
The condition $\lambda > \alpha C_\Sigma \sigma_k^2$ becomes easier to satisfy as the local covariance decreases, since $\sigma_k^2$ directly scales the reward-drift term in \Eqref{eq: thm-flag-monotonic-improvement}.
However, an overly small $\sigma_k$ restricts the proposal region and thereby weakens the guidance signal.
To balance these effects, FLAG anneals the local covariance from $\sigma_{\mathrm{init}}$ to $\sigma_{\mathrm{final}}$.
As shown in \Cref{tab:ablation_buffer_covariance}, moderate annealing outperforms a fixed covariance by allowing early local exploration while keeping the drift small (\Cref{tab:theorem_diagnostics}).
In contrast, overly aggressive annealing shrinks the local search region too early, leading to convergence to suboptimal local modes and eventual policy collapse.

\paragraph{(b) Guidance Buffer}
The CFM projection error $\epsilon_k^{\mathrm{proj}}$ measures how accurately the flow policy realizes the improved target actions obtained from the E-step.
The guidance buffer reduces this error by caching recently computed target actions and reusing them across multiple off-policy updates, providing denser supervision for the CFM projection.
As shown in \Cref{tab:ablation_buffer_covariance}, removing the guidance buffer degrades performance, suggesting that local improvements are not sufficiently projected into the flow policy.
Conversely, an excessively large buffer also hurts performance, as stale targets become increasingly mismatched with the current policy.
A moderate buffer size achieves the best trade-off between supervision density and target freshness, keeping $\epsilon_k^{\mathrm{proj}}$ small in practice.

%% file: sections/conclusion.tex
\section{Conclusion} \label{section:conclusion}

We presented FLAG, a MaxEnt-RL framework that optimizes an expressive policy through a supervised learning paradigm via latent-augmented guidance.
By leveraging the deterministic property of the flow map, FLAG replaces the global proposal distribution of prior importance sampling-based methods with a local one, avoiding both the weight degeneracy of global sampling and the numerical instability of BPTT.
FLAG scales to high-dimensional action spaces and performs competitively with methods that rely on BPTT and critic gradients.
As limitations, our evaluation is restricted to dense-reward, state-based control, and our theoretical results rely on assumptions that we verify only empirically.
Furthermore, combining the base flow policy with a local Gaussian head is only one particular instantiation of our framework.
Extending our framework with a more principled construction via ODE-to-SDE conversion \citep{domingoadjoint}, and applying it to offline-to-online fine-tuning for manipulation tasks, is a promising direction for future work.

%% file: sections/appendix/extended_related_works.tex
\section{Related Works at a Glance}
\label{app:extended-related-works}

\Cref{tab:related_works_comparison} summarizes the differences between FLAG and the related methods discussed in \Cref{section:related_works}.
We compare along five axes: whether the method optimizes a MaxEnt-RL objective, whether it is off-policy, whether policy optimization is defined on the original MDP horizon rather than an augmented MDP that unrolls the denoising chain, whether the generative policy is trained with a single end-to-end objective rather than a two-stage procedure, and whether policy supervision is obtained by local IS.
The last axis applies only to methods that estimate a target distribution via importance sampling; for the others, it is marked with a dash.
DPPO and ReinFlow are marked with $^\dagger$ since they assume a pretrained policy to fine-tune.

\begin{table}[h]
    \centering
    \caption{Comparison with related generative-policy RL methods.}
    \label{tab:related_works_comparison}
    \scriptsize
    \begin{tabular}{llccccc}
        \toprule
        \textbf{Category} & \textbf{Method} & \textbf{MaxEnt-RL} & \textbf{Off-policy} & \textbf{Orig.\ horizon} & \textbf{Single obj.} & \textbf{Local IS} \\
        \midrule
        \multirow{5}{*}{Policy gradient}
            & FPO~\citep{fpo}              & \xmark & \xmark & \cmark & \cmark & \nmark \\
            & GenPO~\citep{dinggenpo}      & \xmark & \xmark & \cmark & \cmark & \nmark \\
            & DPPO~\citep{dppo}            & \xmark & \xmark & \xmark & \cmark$^\dagger$ & \nmark \\
            & ReinFlow~\citep{reinflow}    & \xmark & \xmark & \xmark & \cmark$^\dagger$ & \nmark \\
            & QVPO~\citep{qvpo}            & \xmark & \cmark & \cmark & \cmark & \xmark \\
        \midrule
        \multirow{2}{*}{Action gradient}
            & DIPO~\citep{dipo}            & \xmark & \cmark & \cmark & \cmark & \nmark \\
            & QSM~\citep{qsm}              & \xmark & \cmark & \cmark & \cmark & \nmark \\
        \midrule
        \multirow{4}{*}{BPTT-based actor-critic}
            & DACER~\citep{dacer, dacerv2} & \cmark & \cmark & \cmark & \cmark & \nmark \\
            & DIME~\citep{dime}            & \cmark & \cmark & \cmark & \cmark & \nmark \\
            & FlowRL~\citep{flowrl}        & \xmark & \cmark & \cmark & \xmark & \nmark \\
            & EXPO~\citep{expo}            & \xmark & \cmark & \cmark & \xmark & \nmark \\
        \midrule
        \multirow{4}{*}{Target policy matching}
            & SDAC / DPMD~\citep{rsm}      & \cmark & \cmark & \cmark & \cmark & \xmark \\
            & MaxEntDP~\citep{maxentdp}    & \cmark & \cmark & \cmark & \cmark & \xmark \\
            & FPMD~\citep{fpmd}            & \cmark & \cmark & \cmark & \cmark & \xmark \\
            & \textbf{FLAG (ours)}         & \cmark & \cmark & \cmark & \cmark & \cmark \\
        \bottomrule
    \end{tabular}
\end{table}

\label{app:on_policy_two_stage}
To position FLAG against the methods in \Cref{tab:related_works_comparison}, we compare FLAG with DPPO~\citep{dppo}, ReinFlow~\citep{reinflow}, FPO~\citep{fpo}, and EXPO~\citep{expo} on the DMC Dog-run and Dog-trot tasks.
DPPO, ReinFlow, and FPO are on-policy and are typically run with many parallel environments and far more environment interactions than the off-policy methods in \Cref{section:experiments}.
We match the budget to 1M environment steps for all methods to isolate sample efficiency, which is the relevant axis for the off-policy setting FLAG targets.
All methods are trained from scratch without a pretrained policy.
Since the baseline implementations use standard critics, we report FLAG without CrossQ, following the protocol of \Figref{fig:wo_crossq}.
All methods are evaluated at 1M steps with normalized returns, and we report IQM with 95\% confidence intervals over 5 seeds.

As shown in \Cref{tab:on_policy_two_stage}, FLAG achieves the highest IQM on both tasks by a large margin.
The on-policy methods learn little within this budget, and EXPO, which is off-policy but relies on a two-stage procedure with a separately trained edit policy, also falls well below FLAG.
These results support the structural differences summarized in \Cref{tab:related_works_comparison}: off-policy updates at the original horizon with a single objective for the generative policy are important for sample-efficient learning in high-dimensional control.

\begin{table}[h]
    \centering
    \caption{\textbf{Comparison with on-policy and two-stage methods on DMC Dog.} Normalized IQM return at 1M environment steps over 5 seeds, with 95\% confidence intervals in brackets. All methods are trained from scratch, and FLAG is reported without CrossQ to match the critics used by the baseline implementations.}
    \label{tab:on_policy_two_stage}
    \footnotesize
    \begin{tabular}{lcc}
    \toprule
    \textbf{Algorithm} & \textbf{Dog-Run} & \textbf{Dog-Trot} \\
    \midrule
    \textbf{FLAG w/o CrossQ} & \bestnum{0.310} \textcolor{statgray}{\scriptsize [0.112, 0.334]} & \bestnum{0.507} \textcolor{statgray}{\scriptsize [0.445, 0.531]} \\
    DPPO                      & 0.035 \textcolor{statgray}{\scriptsize [0.028, 0.042]} & 0.044 \textcolor{statgray}{\scriptsize [0.023, 0.059]} \\
    ReinFlow                  & 0.033 \textcolor{statgray}{\scriptsize [0.021, 0.046]} & 0.045 \textcolor{statgray}{\scriptsize [0.036, 0.052]} \\
    FPO                       & 0.010 \textcolor{statgray}{\scriptsize [0.006, 0.015]} & 0.010 \textcolor{statgray}{\scriptsize [0.006, 0.018]} \\
    EXPO                      & \secondnum{0.110} \textcolor{statgray}{\scriptsize [0.092, 0.122]} & \secondnum{0.072} \textcolor{statgray}{\scriptsize [0.011, 0.132]} \\
    \bottomrule
    \end{tabular}
\end{table}

\newpage

%% file: sections/appendix/flow_maxent-rl.tex
\section{Computational Challenges in Flow-based MaxEnt-RL}
\label{app: flow-maxent-rl}
In this section, we describe in detail the technical challenges of applying flow-based policies within the standard MaxEnt-RL framework.
Specifically, we define our policy $\tilde{\pi}_\theta$ as a flow-induced distribution where an action is generated by solving the ODE $T_\theta(s, z) = a^1 = z + \int_0^1 v_\theta(a^\tau, \tau, s)d\tau $ following \Eqref{eq: flow-map}. 

\paragraph{Unbiased log-density estimation using Hutchinson trace estimator.}

A significant advantage of flow-based policies is the ability to recover exact log densities, which are typically unavailable in diffusion models.
For a flow policy defined by the ODE trajectory $\phi_t^\tau$, the log-density admits a path-integral form:
\begin{equation*}
    \log \tilde{\pi}_\theta(a_t\mid s_t)
    =\log p_z(z_t)-\int_0^1 \text{tr}\Big(J_\theta^\tau\Big)\,d\tau,
    \qquad
    J_\theta^\tau \triangleq \frac{\partial u_\theta^\tau}{\partial \phi}\big(s_t,\phi_t^\tau\big).
\end{equation*}
Hutchinson's identity \citep{hutchinsontraceestimator} allows us to estimate the trace of a square matrix $J$.
Using any random vector $\epsilon$ s.t. $\mathbb{E}[\epsilon\epsilon^\top]=I$, we obtain the unbiased estimator, ${\Tr(J)=\mathbb{E}_\epsilon[\epsilon^\top J \epsilon]}$.
Applying this estimator to $\log \tilde \pi_\theta (a \mid s)$ yields the following.
\begin{equation}
    \label{eq:hutch_logpi}
    \log \tilde{\pi}_\theta(a_t\mid s_t) =
    \mathbb{E}_\epsilon\left[ \log p_z(z_t)-\int_0^1 \epsilon^\top J_\theta^\tau \epsilon \, d\tau \right] .
\end{equation}
In practice, the quadratic form $\epsilon^\top J_\theta^\tau \epsilon$ is computed without forming $J_\theta^\tau$ explicitly leveraging Jacobian-vector product.
With reparameterization $a_t=T_\theta(s_t,z_t)$ and \Eqref{eq:hutch_logpi}, we can form an unbiased Monte Carlo estimator of MaxEnt objective as follows:
\begin{equation}
\label{eq:unbiased_objective}
    \mathcal{J}_\text{MaxEnt}(\theta) = 
    \mathbb{E}_{z_t\sim p_z, \, \epsilon} \left[ Q_\phi\!\big(s_t,T_\theta(s_t,z_t)\big) - \log p_z(z_t)-\int_0^1 \epsilon^\top J_\theta^\tau \epsilon\bigr) \right].
\end{equation}

\paragraph{Pathwise gradient requires differentiating through the entire ODE.}
Although \Eqref{eq:unbiased_objective} enables unbiased \emph{evaluation}, optimizing $\theta$ by a direct pathwise gradient requires $\partial a_t/\partial\theta$.
Let $S_t^\tau \triangleq \frac{\partial \phi_t^\tau}{\partial \theta}$ denote the parameter sensitivity along the flow.
Differentiating the ODE gives the sensitivity equation
\begin{equation}
\label{eq:sensitivity_ode}
\frac{d S_t^\tau}{d\tau}
=
J_\theta^\tau\, S_t^\tau
+
\frac{\partial u_\theta^\tau}{\partial \theta}\big(s_t,\phi_t^\tau\big),
\qquad
S_t^0 = 0,
\qquad
\frac{\partial a_t}{\partial \theta}=S_t^1 .
\end{equation}
Therefore, the pathwise gradient of the $Q$-term is
\begin{equation}
\label{eq:grad_Q_pathwise}
\nabla_\theta Q_\phi\!\big(s_t,a_t\big)
=
\nabla_a Q_\phi(s_t,a)\big|_{a=a_t}\; S_t^1 .
\end{equation}
Computing $S_t^1$ entails differentiating through the entire ODE evolution in $\tau\in [0, 1]$. Since the sensitivity dynamics repeatedly apply the local Jacobian $J_\theta^\tau$ along the trajectory, the accumulated transformation can become ill-conditioned, leading to vanishing/exploding sensitivities \citep{vanishing-gradient, exploding-gradient}.

\paragraph{The log-density gradient involves second-order terms (HVP).}
The Hutchinson integrand depends on the state through $J_\theta^\tau=\partial u_\theta^\tau/\partial \phi$.
Its state gradient is
\begin{equation*}
    \nabla_{\phi_t^\tau}\!\left(\epsilon^\top J_\theta^\tau \epsilon\right) = \left\langle \epsilon, \Big(\nabla_{\phi}^2 u_\theta^\tau(s_t,\phi_t^\tau)\Big) \epsilon \right\rangle,
\end{equation*}
which is a Hessian--vector product (a double contraction with $\epsilon$).
Propagating this along the flow yields the score term
\begin{align}
\label{eq:score_has_hvp}
\nabla_a \log \tilde \pi_\theta (a_t \mid s_t)
&=
J_{0}^\top \nabla_{z}\log p_z(z_t)
-\mathbb{E}_{\epsilon}\bigg[
\int_0^1
(J_\theta^\tau)^\top
\underbrace{\left\langle \epsilon, \Big(\nabla_{\phi}^2 u_\theta^\tau(s_t,\phi_t^\tau)\Big) \epsilon \right\rangle}_{\text{HVP}}
\,d\tau
\bigg],
\end{align}
where $\nabla_a=(J_\theta^\tau)^\top \nabla_{\phi_t^\tau}$.
\paragraph{Implication.}
\Eqref{eq:sensitivity_ode} and \Eqref{eq:grad_Q_pathwise} show that even the $Q$-term requires differentiating through the entire ODE.
More importantly, \Eqref{eq:score_has_hvp} reveals that the entropy gradient additionally involves Hessian--vector products of the vector field, making direct MaxEnt-RL optimization with flow policies more expensive than objective evaluation.

%% file: sections/appendix/derivations.tex
\section{Derivations} \label{app: derivations}
In this section, we make detail derivations in Section~\ref{section: method_main}.

\subsection{Global and Local Consistency}
\label{app:marginal-Q-and-conditioanl-Q}
\subsubsection{Marginal Distributions Consistency (Corollary~\ref{cor:marginal-consistency})}
\label{app: marginal consistency}
While our general framework assumes an infinite horizon, for notational clarity, we restrict our attention to a finite horizon $T$ in this derivation.
Fix an initial pair $(s_t,a_t)$, and any $z_t$; write $\hat s_t=(s_t,z_t)$.  
Consider the suffix trajectories
\begin{equation*}
    \tau=(s_{t+1},a_{t+1},\dots,s_T,a_T),\qquad
    \hat\tau=(\hat s_{t+1},a_{t+1},\dots,\hat s_T,a_T).
\end{equation*}
Under the augmented process with $z$-conditioned policy $\hat \pi(\cdot\mid \hat s)$ and i.i.d.\ $z$’s,
\begin{equation}
\label{eq: z-MDP-traj}
    p_{\hat \pi}\left(\hat\tau\mid \hat s_t, a_t\right)
    =\prod_{i=t}^{T-1} p(s_{i+1}\mid s_i,a_i)\,p_z(z_{i+1})\,\hat \pi(a_{i+1}\mid s_{i+1},z_{i+1}).
\end{equation}

Marginalizing out $z_{t+1:T}$ gives the $(s,a)$–trajectory law
\begin{align}
\int p_{\hat\pi}(\hat\tau\mid \hat s_t,a_t)\,\prod_{i=t+1}^{T-1}\!dz_i
&=\prod_{i=t}^{T-1} p(s_{i+1}\mid s_i,a_i)\,\underbrace{\int p_z(z_{i+1})\,\hat \pi(a_{i+1}\mid s_{i+1},z_{i+1})\,dz_{i+1}}_{=\ \pi(a_{i+1}\mid s_{i+1})}\nonumber\\
&=\prod_{i=t}^{T-1} p(s_{i+1}\mid s_i,a_i)\,\pi(a_{i+1}\mid s_{i+1})
=:p_\pi(\tau\mid s_t,a_t).
\label{eq: app-traj-marg}
\end{align}
Thus, after integrating out $z$, the $(s,a)$–trajectory distribution exactly matches that induced by the marginal policy $\pi$. We can extend this result to the infinite horizon as $T\to\infty$.
\begin{lemma}
    \label{lem:marginal-consistency}
    In the $z$-MDP framework, let $\hat \pi$ be the local policy defined in \Eqref{eq:local-policy} and $\pi$ be the global policy defined in \Eqref{eq:marginal-policy}. The discounted state and state--action marginal distributions induced by $\hat \pi$ satisfy:
    \begin{equation}
        \hat {\rho}_{\hat\pi}(s,z,a) \;=\; \rho_\pi(s)\,p_z(z)\,\hat \pi(a\mid s, z),
        \qquad
        \hat\rho_{\hat\pi}(\hat s)\;=\;\rho_\pi(s)\,p_z(z).
    \end{equation}
\end{lemma}

\begin{proof}
    Consider the joint probability of the state and latent variable at time step $t$. Since $z_t$ is sampled i.i.d.\ from $p_z(\cdot)$ and is independent of $s_t$, we have $P_{\hat \pi}(z_t=z \mid s_t=s) = p_z(z)$. Furthermore, leveraging the trajectory consistency in \Eqref{eq: app-traj-marg}, the marginal distribution of $s_t$ under $\hat \pi$ is identical to that under $\pi$, i.e., $P_{\hat \pi}(s_t=s) = P_{\pi}(s_t=s)$. 
    
    Combining these observations, the joint probability at time $t$ factorizes as:
    \begin{equation}
        P_{\hat \pi}(s_t=s, z_t=z) \;=\; P_{\pi}(s_t=s)\, p_z(z).
    \end{equation}
    Substituting this equality into the definition of the discounted state marginal distribution yields:
    \begin{equation}
        \begin{aligned}
            \hat\rho_{\hat\pi}(s, z) &= (1-\gamma)\sum_{t=0}^{\infty} \gamma^t P_{\hat \pi}(s_t=s, z_t=z) \\
            &= (1-\gamma)\sum_{t=0}^{\infty} \gamma^t P_{\pi}(s_t=s)\, p_z(z) \\
            &= \left[ (1-\gamma)\sum_{t=0}^{\infty}\gamma^t P_\pi(s_t=s)\right] p_z(z) \\
            &= \rho_\pi(s)\,p_z(z).
        \end{aligned}
    \end{equation}
    Finally, for the state--action marginal distribution, the following equation concludes the proof:
    \begin{equation}
        \hat {\rho}_{\hat\pi}(s,z,a) \;=\; \big[ \rho_\pi(s)\,p_z(z) \big] \hat \pi(a\mid s, z), \qquad \rho_\pi(s, a)=\int p_z(z)\,\hat\pi(a\mid s, z)\,\rho_\pi(s)\,dz.
    \end{equation}
\end{proof}
\subsubsection{Q-function Consistency (Corollary~\ref{cor: q-func-consistency})}
\label{app: Q-function consistency}
Now we compare soft returns. The augmented definition is
\begin{equation}
    Q^{\hat \pi}(\hat s_t,a_t)
    =\mathbb E_{\hat\tau\sim p_{\hat\pi}}\!\left[
    \sum_{k=0}^{\infty}\gamma^k\left(\hat r(\hat s_{t+k},a_{t+k})
    +\alpha\,\mathcal H\left(\hat \pi(\cdot\mid \hat s_{t+k}), \tilde\pi(\cdot\mid s_{t+k})\right)\right)\right].    \notag
\end{equation}
Taking expectation over $z_{t+1:\infty}$ and using \Eqref{eq: app-traj-marg}, the reward terms coincide by $\hat r(\hat s,a)=r(s,a)$, and the cross-entropy terms reduce to
\begin{equation*}
    \mathbb E_{z_{t+k}\sim p_z}\left[H\left(\hat \pi(\cdot\mid \hat s_{t+k}), \tilde\pi(\cdot\mid s_{t+k})\right)\right]
    =\mathcal H\left(\pi(\cdot\mid s_{t+k}), \tilde \pi(a\mid s_{t+k})\right).
\end{equation*}
Hence, marginalizing over $z_{t+1:\infty}$ gives
\begin{equation*}
    Q^\pi(\hat s_t,a_t)
    =\mathbb E_{\tau\sim p_\pi}\!\left[
    \sum_{k=0}^{\infty}\gamma^k\Big(r(s_{t+k},a_{t+k})
    +\alpha\,\mathcal H\big(\pi(\cdot\mid s_{t+k}), \tilde \pi(a\mid s_{t+k})\big)\Big)\right]
    =:Q^\pi(s_t,a_t).    
\end{equation*}

The right–hand side does not depend on the initial $z_t$, therefore $Q^{\hat \pi}(s_t, z_t, a_t)=Q^\pi(s_t, a_t)$ is valid for all $z_t$.

\subsection{EM algorithm Derivations (Section~\ref{subsection: local-policy-update})}
\label{app: EM-derivations}
This section provides the full derivation of FLAG's EM update in Section~\ref{subsection: local-policy-update}.
We proceed in three steps: (i) we cast cross-entropy augmented RL objective as probabilistic inference 
and derive the variational lower bound both in the original MDP and the $z$-MDP; (ii) we 
solve the E-step in closed form, yielding a non-parametric target distribution $q_k$; 
(iii) we project $q_k$ back to the parametric policy class through the M-step. Throughout, 
we follow the MPO~\citep{mpo} framework while adapting it to our cross-entropy augmented 
reward structure, which we discuss in the remark below.

An infinite-horizon discounted reward formulation can be cast as inference problem as follows:
\begin{equation}
    p\left( \mathcal O=1\mid \tau\right)\propto \exp\left( \sum_{t=1}^{\infty}\gamma^t \left(r_t/\lambda\right)\right)
\end{equation}
where $r_t=r(s_t, a_t)$.
Here, the optimality variable $\mathcal O$ reveals \emph{the event of obtaining maximum reward by choosing an action}.
We define the augmented reward in the spirit of \citet{sac}.
\begin{equation}
\label{eq: app-aug-reward}
\begin{aligned}
    r_\pi(s_t, a_t)&\triangleq r(s_t, a_t) + \alpha\mathbb{E}_{s_{t+1}\sim p(\cdot\mid s_t, a_t)}\big[\mathcal H\left(\pi(\cdot\mid s_{t+1}), \tilde \pi(\cdot\mid s_{t+1})\right)\big],  \\
    \hat r_{\hat \pi}(\hat s_t, a_t)&\triangleq r(\hat s_t, a_t) + \alpha\mathbb{E}_{\hat s_{t+1}\sim \hat p(\cdot\mid \hat s_t, a_t)}\big[\mathcal H\left(\hat \pi(\cdot\mid \hat s_{t+1}), \tilde \pi(\cdot\mid s_{t+1})\right)\big].
\end{aligned}
\end{equation}

\begin{remark*}
\textbf{Remark on the reward structure and the E-step.}
{
The augmented reward in \Eqref{eq: app-aug-reward} is defined separately for each policy 
($r_\pi$ for $\pi$, $r_q$ for $q$), since the cross-entropy bonus follows the soft Bellman 
structure of SAC~\citep{sac}. Although this augmentation is policy-dependent, it is 
compatible with the MPO-style E-step derivation. Following MPO~\citep{mpo}, the E-step 
starts from the reference $q = \hat\pi_k$ and performs a one-step lookahead, so the 
Q-function used inside the E-step is $Q^{\hat\pi_k}$, in which the entropy bonus for 
$\hat\pi_k$ is already absorbed. The resulting energy
\begin{equation*}
    f_{\hat s, k}(a) = Q^{\hat \pi_k}(\hat s, a) - \alpha \log\tilde\pi_{\theta_k}(a \mid s)=Q^{\pi_k}(s, a) - \alpha \log \tilde \pi_{\theta_k}(a\mid s)
\end{equation*}
is a function of $(s, a)$ alone, with $\theta_k$ treated as a constant within iteration $k$. 
Hence the MPO closed-form E-step solution and the subsequent M-step derivation apply without 
modification. The effect of the parameter update $\theta_k \to \theta_{k+1}$ between 
iterations (the reward drift) is analyzed separately in the monotonic-improvement proof 
(Appendix~\ref{appendix:monotonic-improvement}).}
\end{remark*}

To derive a tractable objective, we construct a variational lower bound on the log-likelihood of optimality $\log p_\pi(\mathcal O = 1)$, following the RL-as-inference framework~\citep{levine2018reinforcement}. With the policy induced trajectory $\tau\sim p_\pi(\tau)$ in \Eqref{eq: app-traj-marg}, we apply Jensen's inequality with an auxiliary distribution $p_q(\tau)$:
\begin{align*}
\log p_\pi(\mathcal O=1)&=\log \int p_\pi(\tau)p(\mathcal O=1\mid \tau)d\tau\geq \int p_q(\tau)\left[ 
\log p(\mathcal O=1\mid \tau)+\log\frac{p_\pi(\tau)}{p_q(\tau)}
\right]d\tau \\
&=\mathbb{E}_q\left[ \sum_t \gamma^t r_q(s_t, a_t)/\lambda\right]-D_\text{KL}\left( p_q(\tau)\;\|\;p_\pi(\tau)\right),
\end{align*}

Following the standard derivation in~\citep{mpo, levine2018reinforcement}, we absorb the 
entropy component of the trajectory-level KL into the per-step augmented reward
\begin{equation*}
    r_q(s_t, a_t) = r(s_t, a_t) + \alpha\, \mathbb{E}_{s_{t+1}}\!\big[\mathcal H(q(\cdot\mid s_{t+1}), \tilde\pi(\cdot\mid s_{t+1}))\big],
\end{equation*}
and express the remaining per-step KL in discounted form, yielding the variational objective
\begin{equation*}
    J(q, \xi) = \mathbb{E}_{q}\!\left[ \sum_{t=0}^{\infty}\gamma^t\!\left(\frac{r_q(s_t, a_t)}{\lambda}- D_\text{KL}\!\left( q(\cdot\mid s_t)\,\|\,\pi(\cdot\mid s_t; \xi)\right)\right)\right].
\end{equation*}

We now extend the same variational argument to the $z$-MDP $\hat{\mathcal M}$, on which FLAG actually operates. Replacing the trajectory $\tau$ with the augmented trajectory $\hat\tau$ and using the marginal consistency established in Corollary~\ref{cor:marginal-consistency}, we obtain:
\begin{align*}
    \log p_{\hat \pi}(\mathcal O=1)&=\log \int p_{\hat\pi}(\hat\tau)p(\mathcal O=1\mid \hat \tau)d\hat\tau \geq \int p_q(\hat \tau)\left[ 
    \log p(\mathcal O=1\mid \hat \tau)+\log\frac{p_{\hat \pi}(\hat \tau)}{p_q(\hat\tau)}
    \right] d\hat\tau\\
    &=\mathbb{E}_q\left[ \sum_t \hat r_q(\hat s_t, a_t)/\lambda\right]-D_\text{KL}\left( p_q(\hat \tau)\;\|\;p_{\hat \pi}(\hat \tau)\right) \\
    \therefore \; J(q, \xi)&= \mathbb{E}_{q}\left[ \sum_{t=0}^{\infty}\gamma^t\left[\hat r_q(\hat s_t, a_t)-\lambda D_\text{KL}\left( q(\cdot\mid \hat s_t)\;\|\;\hat \pi(\cdot\mid \hat s_t, \xi)\right)\right]\right].
\end{align*}
\paragraph{Constrained E-step}
We follow the MPO \citep{mpo} derivation and adapt it to hard KL constraints and one-step bootstrapping.
At iteration k, the objective function is
\begin{equation}
\label{eq: const-e-step}
    \max_q \mathbb{E}_{\hat s\sim \hat\rho_{\hat\pi}}\big[ \mathbb{E}_{a\sim q(\cdot\mid \hat s)}\big[Q^{\pi}(s, a) - \alpha \log \tilde \pi(a\mid s)\big]\big] \qquad s.t. \; [D_\text{KL}\left( q(\cdot\mid \hat s)\;\|\; \hat \pi(\cdot\mid \hat s;\xi_k)\right) \leq \epsilon.
\end{equation}

This objective has closed-form solution $q_k$, which is called the non-parametric variational distribution:
\begin{equation}
\label{eq: const-e-step-result}
    q_k(a\mid\hat s)=\frac{\hat \pi(a\mid \hat s;\xi_k)\exp\left(f_{\hat s, k}(a)/\lambda^*\right)}{Z(\hat s)}, \qquad  Z(\hat s)=\int \hat \pi(a\mid \hat s;\xi_k)\exp\left(f_{\hat s, k}(a)/\lambda^*\right)\;da.
\end{equation}
\paragraph{M-step}
In the M-step, given the fixed non-parametric target distribution $q_k$, we update the policy parameters $\xi_k$ by projecting $q_k$ back to our parametric rollout policy class:
\begin{equation}
    \xi_{k+1} \in \argmin_{\xi}
    \mathbb{E}_{\hat\rho_{\hat{\pi}_k}}
    \left[
    D_{\mathrm{KL}}\!\left(q_k(\cdot\cmid \hat s)\,\|\,\hat{\pi}(\cdot\cmid \hat s;\xi_k)\right)
    \right].
\end{equation}

%% file: sections/appendix/proofs.tex
\section{Proofs}
\label{app: proofs}
\paragraph{Standing Assumptions.}
We consider an infinite-horizon $\gamma$-discounted MDP $\mathcal M=(\mathcal S,\mathcal A,p,r,\rho_0,\gamma)$.
Fix a reference measure $\mu$ on $\mathcal A$ and assume all policies admit densities with respect to $\mu$.
Throughout this appendix, we write $da$ in place of $d\mu(a)$ for notational simplicity.
Accordingly, $\tilde\pi(\cdot\mid s)$ and $\pi(\cdot\mid s)$ denote $\mu$-densities, and all KL divergences and expectations are taken with respect to these densities. Furthermore, we assume that $\tilde \pi$ and $\pi$ are differentiable in $\mathcal A$.

\begin{assumption}[Bounded Augmented Reward]
\label{assum:bounded_reward}
    The entropy-augmented reward is uniformly bounded. Specifically, there exists a constant $R_{\max}^{\mathrm{aug}} < \infty$ such that
    \begin{align}
        \big|r(s,a) + \alpha\,b(s,a)\big| \le R_{\max}^{\mathrm{aug}}, \qquad \forall (s,a),
    \end{align}
    where $b(s,a)$ denotes the entropy or cross-entropy bonus term (e.g., $-\log \tilde\pi(a\mid s)$).
\end{assumption}

\begin{assumption}[Regularity of Base Density]
\label{assum:policy_regularity}
    For each state $s$, the base density $\tilde{\pi}(\cdot \mid s)$ is three-times continuously differentiable with respect to $a$ and is strictly positive $\mu$-a.e. Moreover, the magnitudes of its third-order partial derivatives are pointwise bounded by a nonnegative function $M_s(a)$ satisfying the integrability condition:
    \begin{align}
        \int_{\mathcal A} M_s(a)\,da < \infty.
    \end{align}
\end{assumption}

\begin{assumption}[Integrable gradient of the base density]
\label{assum:integrable-gradient}
    The action gradient of the base density has a uniformly bounded $L^1$ norm over states:
    \begin{align}
        \sup_s\|\nabla_a \tilde\pi(\cdot\mid s)\|_{L^1}
        = \sup_s\int_{\mathcal A} \|\nabla_a \tilde\pi(a\mid s)\|_2 \, da < \infty.
    \end{align}
    This condition ensures that the density $\tilde\pi$ does not exhibit unbounded variation in the action space.
\end{assumption}

\begin{assumption}[Finite Partition Function]
\label{assum:finite_partition}
    For every state $\hat{s}$ and iteration $k$, the function $Q^k(\hat{s}, \cdot)$ is measurable, and the corresponding Boltzmann distribution is well-defined (i.e., normalizable):
    \begin{align}
        \int_{\mathcal A} \exp\left( \frac{Q^k(\hat{s}, a)}{\alpha} \right) da < \infty.
    \end{align}
\end{assumption}

\subsection{Proof of Proposition \ref{prop:surrogate-validity}}
\label{appendix:TV-KL-bound-proof}
In this section, we assume that the Gaussian smoothing covariance is state-independent and isotropic. 
For notational simplicity, we denote it by $\Sigma$, with $\Sigma=\sigma^2 I$ and denote $\tilde \pi_\theta$ simply as $\tilde \pi$ hereafter. We prove Proposition~\ref{prop:surrogate-validity} in two parts: Lemma~\ref{lem: app-TV-bound} establishes the TV bound, and Lemma~\ref{lem:app-kl-second-order} establishes the second-order KL bound. 

Let $\delta\sim \mathcal N(0, \Sigma)$ be an independent noise at $a=T_\theta(s, z)+\delta$. Mathematically, this takes the form 
\begin{equation}
\label{eq: convolution-form}
    \pi(\cdot\mid s)=\tilde \pi(\cdot\mid s) *\varphi_{\Sigma},
\end{equation}
where $*$ denotes convolution. The Gaussian smoothing calculation is performed on $\mathbb R^d$ with Lebesgue densities, before any action squashing. Applying the same invertible squashing map to both policies preserves their TV and KL divergences. 
\subsubsection{Total variation distance bound}
We assume that the following inequality holds for all $s$:
\begin{equation}
    \|\nabla_a \tilde \pi(\cdot\mid s)\|_{L^1}
    \triangleq \int \|\nabla_a \tilde \pi(a\mid s)\|_2\,da < \infty.
\end{equation}

\begin{lemma}
\label{lem: app-TV-bound}
    Suppose the global policy admits a Gaussian-kernel convolution form
    \begin{equation}
    \label{eq: app-gaussian-conv-kernel-form}
        \pi(a\mid s)=\int \tilde\pi(a-\delta\mid s)\,\varphi_{\Sigma}(\delta)\,d\delta,
    \end{equation}
    where $\varphi_{\Sigma}$ is the density of $\mathcal N(0,\Sigma)$.
    Then the following inequality holds for all $s\in\mathcal S$
    \begin{equation}
        D_{\mathrm{TV}}\left(\pi(\cdot\mid s),\tilde\pi(\cdot\mid s)\right)
        \leq \frac12 \|\nabla_a \tilde\pi(\cdot\mid s)\|_{L^1}\,
        \mathbb E_{\delta\sim \mathcal N(0,\Sigma)}\left[\|\delta\|_2\right].
    \end{equation}
\end{lemma}
\begin{proof}
     By definition of the convolution in \Eqref{eq: app-gaussian-conv-kernel-form},
    \begin{equation}
        \pi(a\mid s)-\tilde\pi(a\mid s)
        =
        \int\big(\tilde\pi(a-\delta\mid s)-\tilde\pi(a\mid s)\big)\,\varphi_{\Sigma}(\delta)\,d\delta.
    \end{equation}
    Using the triangle inequality, 
    \begin{align}
        \lVert \pi(\cdot\mid s)-\tilde \pi(\cdot\mid s)\rVert_{L^1}
        &=\int\left\lvert \int\left( \tilde\pi(a-\delta\mid s)-\tilde \pi(a\mid s)\right)\,\varphi_{\Sigma}(\delta) \,d\delta \right\rvert\, da \\
        &\leq\int\int\left\lvert 
        \tilde  \pi(a-\delta\mid s)-\tilde \pi(a\mid s)
        \right\rvert \varphi_{\Sigma}(\delta)\, d\delta\;da
    \end{align}
    Since the integrand is nonnegative, Tonelli's theorem yields
    \begin{equation}
    \label{eq: app-measure-diff-upper-bound}
        \lVert \pi(\cdot\mid s)-\tilde \pi(\cdot\mid s)\rVert_{L^1}
        \leq
        \int  \varphi_{\Sigma}(\delta) \left( 
        \int\left\lvert 
        \tilde \pi(a-\delta\mid s)-\tilde \pi(a\mid s)
        \right\rvert da
        \right) d\delta.
    \end{equation}
    For each fixed $(a,\delta)$, by the integral form of the mean value theorem,
    \begin{equation}
    \label{eq: app-mvt}
        \tilde\pi(a-\delta\mid s)-\tilde\pi(a\mid s)
        = -\int_0^1 \delta^\top \nabla_a \tilde\pi(a-t\delta\mid s)\,dt.
    \end{equation}
    Applying the Cauchy--Schwarz inequality yields
    \begin{equation}
    \label{eq: app-cs}
        \left|\tilde\pi(a-\delta\mid s)-\tilde\pi(a\mid s)\right|
        \le \|\delta\|_2\int_0^1 \|\nabla_a \tilde\pi(a-t\delta\mid s)\|_2\,dt.
    \end{equation}
    Integrating over $a$ and using the change of variables $u=a-t\delta$,
    \begin{align}
        \int \left|\tilde\pi(a-\delta\mid s)-\tilde\pi(a\mid s)\right|da
        &\le \|\delta\|_2 \int_0^1 \int \|\nabla_a \tilde\pi(a-t\delta\mid s)\|_2\,da\,dt  \notag \\
        &= \|\delta\|_2 \int_0^1 \int \|\nabla_a \tilde\pi(u\mid s)\|_2\,du\,dt \notag \\
        &= \|\delta\|_2 \|\nabla_a \tilde\pi(\cdot\mid s)\|_{L^1}.
    \end{align}
    
    Plugging this into \Eqref{eq: app-measure-diff-upper-bound} bound gives
    \begin{equation*}
        \boxed{
        \left\|\pi(\cdot\mid s)-\tilde\pi(\cdot\mid s)\right\|_{L^1}
        \le \left\|\nabla_a \tilde\pi(\cdot\mid s)\right\|_{L^1}\,
        \mathbb E_{\delta\sim \mathcal N(0,\Sigma)}[\|\delta\|_2].
        }
    \end{equation*}
    Finally, using $D_{\mathrm{TV}}(p,q)=\frac12\|p-q\|_{L^1}$ concludes the proof.
    Moreover, by Cauchy--Schwarz,
    \begin{equation*}
        \label{eq: app-edelta-bound}
        \mathbb E_{\delta\sim\mathcal N(0,\Sigma)}[\|\delta\|_2]
        \le \sqrt{\mathbb E[\|\delta\|_2^2]}
        = \sqrt{\mathbb E[\delta^\top\delta]}
        = \sqrt{\mathrm{tr}(\Sigma)}.
    \end{equation*}
    Combining the above bounds implies
    $D_{\mathrm{TV}}(\pi,\tilde\pi) = \mathcal O\big(\sqrt{\mathrm{tr}(\Sigma)}\big) \quad \text{as } \mathrm{tr}(\Sigma)\to 0.$
\end{proof}

\subsubsection{Kullback-Leibler Divergence Bound}
We establish the second-order KL bound under a weighted Taylor remainder condition.
Recall that, whenever both entropy terms are finite,
\begin{equation*}
\label{eq:app-crossent-entropy-kl}
\mathcal H(\pi,\tilde\pi)=\mathcal H(\pi)+D_{\mathrm{KL}}(\pi\;\|\;\tilde\pi),
\end{equation*}
where $\mathcal H(\pi,\tilde\pi)\triangleq \mathbb E_{\pi}[-\log \tilde\pi]$.

\begin{lemma}[Second-order KL bound under small Gaussian smoothing]
\label{lem:app-kl-second-order}
For each state $s$, fix a strictly positive base density $\tilde\pi(\cdot\mid s)$
satisfying Assumption~\ref{assum:policy_regularity}, and let $\pi$ have the
Gaussian smoothing form in \Eqref{eq: app-gaussian-conv-kernel-form} with
$\Sigma=\sigma^2 I$.
Assume the weighted integrability condition
\begin{equation}
\label{eq:app-kl-integrability}
\int \frac{\big(\operatorname{tr}(\nabla_a^2\tilde\pi(a\mid s))\big)^2}
{\tilde\pi(a\mid s)}\,da < \infty,
\end{equation}
and that the remainder in the smoothing expansion
\begin{equation}
\label{eq:app-pi-expansion}
\pi(a\mid s)=\tilde\pi(a\mid s)
+\frac12\operatorname{tr}\!\big(\Sigma\nabla_a^2\tilde\pi(a\mid s)\big)
+R_3(a,s)
\end{equation}
satisfies
\begin{equation*}
\int\frac{R_3(a,s)^2}{\tilde\pi(a\mid s)}\,da
=\mathcal O\!\big(\operatorname{tr}(\Sigma)^3\big).
\end{equation*}
Then, as $\operatorname{tr}(\Sigma)\to0$,
\begin{equation}
\label{eq:app-kl-order}
D_{\mathrm{KL}}\!\left(\pi(\cdot\mid s)\,\|\,\tilde\pi(\cdot\mid s)\right)
=\mathcal O\!\big(\operatorname{tr}(\Sigma)^2\big).
\end{equation}
Consequently, whenever the entropy terms are finite,
\begin{equation*}
\label{eq:app-crossent-approx}
\mathcal H(\pi,\tilde\pi)=\mathcal H(\pi)
+\mathcal O\!\big(\operatorname{tr}(\Sigma)^2\big).
\end{equation*}
\end{lemma}

\begin{proof}
Using $\log u\le u-1$ and
$\int \pi(a\mid s)\,da=\int\tilde\pi(a\mid s)\,da=1$, we obtain
\begin{align}
\label{eq:app-kl-bound-x2}
D_{\mathrm{KL}}(\pi\;\|\;\tilde\pi)
&\le \int \pi(a\mid s)
\left(\frac{\pi(a\mid s)}{\tilde\pi(a\mid s)}-1\right)\,da
\notag\\
&= \int\frac{\pi(a\mid s)^2}{\tilde\pi(a\mid s)}\,da-1
\notag\\
&= \int\left(
\frac{\pi(a\mid s)^2}{\tilde\pi(a\mid s)}
-2\pi(a\mid s)+\tilde\pi(a\mid s)
\right)\,da
\notag\\
&= \int
\frac{(\pi(a\mid s)-\tilde\pi(a\mid s))^2}
{\tilde\pi(a\mid s)}\,da.
\end{align}
By \Eqref{eq:app-pi-expansion} and $(x+y)^2\le2x^2+2y^2$,
\begin{align*}
\int\frac{(\pi-\tilde\pi)^2}{\tilde\pi}\,da
&\le\frac12\int
\frac{\big(\operatorname{tr}(\Sigma\nabla_a^2\tilde\pi)\big)^2}{\tilde\pi}\,da
+2\int\frac{R_3^2}{\tilde\pi}\,da\\
&=\mathcal O\!\big(\operatorname{tr}(\Sigma)^2\big)
+\mathcal O\!\big(\operatorname{tr}(\Sigma)^3\big).
\end{align*}
Here the first estimate uses $\Sigma=\sigma^2 I$, fixed dimension, and
\Eqref{eq:app-kl-integrability}; the second is the assumed remainder bound.
This proves \Eqref{eq:app-kl-order}, and the entropy identity gives the final claim.
\end{proof}

When applying this bound across states and training iterations,
we assume that the regularity bounds in this lemma hold uniformly.

\subsection{Connection to SAC and Zeroth-Order Policy Updates}
\label{app: cov-schedule-justification}
This section gives a complementary interpretation of FLAG's moment-matching update.
First, we bound the soft Q-function discrepancy between the composite policy
$\pi$ and the base flow policy $\tilde\pi$ when the local covariance is small
(Appendix~\ref{app:Q-ftn-Sigma-bound}). Second, we relate the target mean to the
gradient of a Gaussian-smoothed soft energy via Stein's identity
(Appendix~\ref{app:SNIS-and-zero-order}). These results explain the connection
to SAC-style action gradients; the monotonic-improvement statement is treated
separately in Appendix~\ref{appendix:monotonic-improvement}.
\subsubsection{Q-function Closeness between Composite Policy and Flow Policy}
\label{app:Q-ftn-Sigma-bound}
Let $\rho_\pi=\sum_{t\geq0}\gamma^t d_t^\pi$ denote the unnormalized
discounted state visitation measure, with the same initial state distribution
for both policies. We use the coupling argument of TRPO~\citep{TRPO}.
\begin{lemma}
\label{lem: app-state-dist-diff-bound}
    The difference between the state distributions induced by $\pi$ and $\tilde \pi$ has the following bound:
    \begin{equation}
        \| \rho_\pi - \rho_{\tilde \pi}\|_1 \leq \frac{2\gamma\delta_\mathrm{TV}}{(1-\gamma)^2},
    \end{equation}
    where $\delta_\mathrm{TV}$ is defined as
    \begin{equation*}
        \delta_\mathrm{TV}
        \triangleq
        D_\mathrm{TV}^{\mathrm{max}}(\pi, \tilde \pi)=\sup_s D_\mathrm{TV}\left( \pi(\cdot\mid s), \tilde \pi(\cdot\mid s)\right)
        \leq
        \frac{1}{2}\left( \sup_s \| \nabla_a \tilde \pi(\cdot\mid s)\|_{L^1}\right)\cdot \sqrt{\mathrm{tr}\left(\Sigma\right)},
    \end{equation*}
    by Lemma~\ref{lem: app-TV-bound}, assuming the displayed supremum of the $L^1$ gradient norm is finite.
\end{lemma}
\begin{proof}

Let $d_t^\pi$ denote the state distribution in timestep $t$, then we can bound the total variation distance between $d_t^\pi$ and $d_t^{\tilde{\pi}}$ as follows
\begin{equation}
    D_\text{TV}(d_t^{\pi}, d_t^{\tilde \pi})\leq P(n_t>0)\leq  1-(1-\delta_\text{TV})^t\leq t\delta_\text{TV},
\end{equation}
where $n_t$ counts action disagreements before time $t$ in a maximal coupling of the two trajectories.
See \citep{TRPO}, Lemma 3 for more details.

The total variation distance between the discounted visitation measure between $\pi$ and $\tilde \pi$ is represented with the state distribution $d_t$:
\begin{equation}
    \lVert \rho_\pi-\rho_{\tilde \pi}\rVert_1=\left\lVert
    \sum_{t=0}^{\infty}\gamma^t\left( 
    d_t^{\pi} - d_t^{\tilde \pi}
    \right)
    \right\rVert_1.
\end{equation}
By triangle inequality, 
\begin{equation}
    \lVert \rho_\pi-\rho_{\tilde \pi}\rVert_1\leq \sum_{t=0}^{\infty} \gamma^t \left\lVert d_t^\pi - d_t^{\tilde \pi}\right\rVert_1.
\end{equation}
Since $\lVert p-q\rVert_1=2D_\text{TV}(p, q)$, 
\begin{align}
    \lVert \rho_\pi-\rho_{\tilde \pi}\rVert_1&\leq \sum_{t=0}^{\infty} \gamma^t \left\lVert d_t^\pi - d_t^{\tilde \pi}\right\rVert_1 
    = 2\sum_{t=0}^{\infty}\gamma^t D_\text{TV}\left( 
    d_t^{\pi} , d_t^{\tilde \pi}
    \right) \\ 
    &\leq2\sum_{t=0}^{\infty}\gamma^t\left( t\delta_\text{TV} \right)
    = 2\gamma \delta_\text{TV}  \sum_{t=1}^{\infty} t\gamma^{t-1} = \frac{2\gamma \delta_\text{TV} }{(1-\gamma)^2}.
\end{align}
\end{proof}
This inequality states that if $\delta_\text{TV}$ is bounded $\delta_\text{TV}=\mathcal O\left( \sqrt{\text{tr}(\Sigma)}\right)$, then the difference of the discounted measure between $\pi$ and $\tilde \pi$ is bounded with $\mathcal O\left( \sqrt{\text{tr}(\Sigma)}\right)$.

Next, we compare the soft Q-functions using the same fixed base-density
reference $\tilde\pi$ in both Bellman operators.
\begin{lemma}[Q-function closeness between $\pi$ and $\tilde\pi$]
\label{lem: app-Q-func-bound}
Let $Q^\pi$ and $Q^{\tilde\pi}$ denote the cross-entropy-regularized and
entropy-regularized soft Q-functions, respectively. Under
Assumption~\ref{assum:bounded_reward} with $b(s,a)=-\log\tilde\pi(a\mid s)$,
\begin{equation}
\label{eq: app-sac-q-func-bound}
\|Q^\pi-Q^{\tilde\pi}\|_\infty
\leq \frac{2\gamma R_{\max}^{\mathrm{aug}}\delta_{\mathrm{TV}}}{(1-\gamma)^2}
=\mathcal O\!\left(\frac{\sqrt{\mathrm{tr}(\Sigma)}}{(1-\gamma)^2}\right),
\end{equation}
where the final order uses the uniform gradient bound in
Lemma~\ref{lem: app-state-dist-diff-bound}, with $\tilde\pi$ held fixed.
\end{lemma}
\begin{proof}
For $\nu\in\{\pi,\tilde\pi\}$, the soft Bellman equation is
\begin{equation*}
Q^\nu(s,a) = r(s,a) + \gamma\mathbb E_{s'\sim p(\cdot\mid s,a)}
\mathbb E_{a'\sim\nu(\cdot\mid s')}
\left[ Q^\nu(s',a')-\alpha\log\tilde\pi(a'\mid s')
\right].
\end{equation*}
Both equations use the same fixed reference density $\tilde\pi$.
Define $U^\nu(s,a) := Q^\nu(s,a)-\alpha\log\tilde\pi(a\mid s)$.
Subtracting the current log-density term from the Bellman equation gives
\begin{equation}
U^\nu(s,a)
=
r(s,a)-\alpha\log\tilde\pi(a\mid s)
+
\gamma\mathbb E_{s'\sim p(\cdot\mid s,a)}
\mathbb E_{a'\sim\nu(\cdot\mid s')}
\left[U^\nu(s',a')\right].
\end{equation}
Thus, $U^\nu$ is the discounted return for the augmented reward
$r-\alpha\log\tilde\pi$:
\begin{equation*}
U^\nu(s,a)
=
\mathbb E_\nu\!\left[
\left.
\sum_{t=0}^{\infty}\gamma^t
\bigl(r(s_t,a_t)-\alpha\log\tilde\pi(a_t\mid s_t)\bigr)
\,\right|\,
s_0=s,\ a_0=a
\right].
\end{equation*}
By Assumption~\ref{assum:bounded_reward}, each summand before
discounting has absolute value at most $R_{\max}^{\mathrm{aug}}$.
Therefore,
\begin{align}
|U^\nu(s,a)|
&\le
\mathbb E_\nu\!\left[
\left.
\sum_{t=0}^{\infty}\gamma^t
\left|r(s_t,a_t)-\alpha\log\tilde\pi(a_t\mid s_t)\right|
\,\right|\,
s_0=s,\ a_0=a
\right] \notag\\ 
&\le
\sum_{t=0}^{\infty}\gamma^t R_{\max}^{\mathrm{aug}}
=
\frac{R_{\max}^{\mathrm{aug}}}{1-\gamma}.
\end{align}
Since this bound holds for every $(s,a)$ and either policy,
\begin{equation*}
\|U^\nu\|_\infty
\le
\frac{R_{\max}^{\mathrm{aug}}}{1-\gamma}.
\end{equation*}

We next compare the two Bellman equations.
Their immediate augmented rewards are identical and cancel upon
subtraction. Adding and subtracting the expectation of
$U^{\tilde\pi}$ under $\pi$ yields
\begin{align*}
U^\pi(s,a)-U^{\tilde\pi}(s,a)
&=
\gamma\mathbb E_{s'\sim p(\cdot\mid s,a)}
\Big[
\mathbb E_{a'\sim\pi(\cdot\mid s')}
\big[U^\pi(s',a')-U^{\tilde\pi}(s',a')\big]\\
&\hspace{7em}
+
\mathbb E_{a'\sim\pi(\cdot\mid s')}
\big[U^{\tilde\pi}(s',a')\big]
-
\mathbb E_{a'\sim\tilde\pi(\cdot\mid s')}
\big[U^{\tilde\pi}(s',a')\big]
\Big].
\end{align*}
The first expectation inside the brackets is bounded in absolute
value by $\|U^\pi-U^{\tilde\pi}\|_\infty$.
For the remaining difference, the definition of total variation
gives, for every $s'$,
\begin{align}
\left|
\mathbb E_{a'\sim\pi(\cdot\mid s')}[U^{\tilde\pi}(s',a')]
-
\mathbb E_{a'\sim\tilde\pi(\cdot\mid s')}[U^{\tilde\pi}(s',a')]
\right| &=
\left|
\int U^{\tilde\pi}(s',a')
\bigl(\pi(a'\mid s')-\tilde\pi(a'\mid s')\bigr)\,da'
\right| \notag \\
&\le
\|U^{\tilde\pi}\|_\infty
\int\left|\pi(a'\mid s')-\tilde\pi(a'\mid s')\right|\,da' \notag \\
&=
2\|U^{\tilde\pi}\|_\infty
D_{\mathrm{TV}}\!\left(
\pi(\cdot\mid s'),\tilde\pi(\cdot\mid s')
\right) \notag \\
&\le
2\delta_{\mathrm{TV}}\|U^{\tilde\pi}\|_\infty.
\end{align}
Taking the supremum over $(s,a)$, we obtain
\begin{equation}
\label{eq: tri-ineq-q-diff}
\|U^\pi-U^{\tilde\pi}\|_\infty
\le
\gamma\|U^\pi-U^{\tilde\pi}\|_\infty
+
2\gamma\delta_{\mathrm{TV}}\|U^{\tilde\pi}\|_\infty.
\end{equation}
Rearranging and substituting the bound on
$\|U^{\tilde\pi}\|_\infty$ gives
\begin{align*}
\|U^\pi-U^{\tilde\pi}\|_\infty
&\le
\frac{2\gamma\delta_{\mathrm{TV}}}{1-\gamma}
\|U^{\tilde\pi}\|_\infty \le
\frac{2\gamma R_{\max}^{\mathrm{aug}}\delta_{\mathrm{TV}}}
{(1-\gamma)^2}.
\end{align*}

Finally, the common log-density term cancels in the difference:
\begin{align*}
U^\pi(s,a)-U^{\tilde\pi}(s,a) &=
\bigl(Q^\pi(s,a)-\alpha\log\tilde\pi(a\mid s)\bigr) -
\bigl(Q^{\tilde\pi}(s,a)-\alpha\log\tilde\pi(a\mid s)\bigr)\\
&=
Q^\pi(s,a)-Q^{\tilde\pi}(s,a).
\end{align*}
Consequently,
\begin{equation}
\|Q^\pi-Q^{\tilde\pi}\|_\infty
\le
\frac{2\gamma R_{\max}^{\mathrm{aug}}\delta_{\mathrm{TV}}}
{(1-\gamma)^2}.
\end{equation}
The stated order follows from
$\delta_{\mathrm{TV}}
=\mathcal O\big(\sqrt{\operatorname{tr}(\Sigma)}\big)$
under the uniform gradient bound used in
Lemma~\ref{lem: app-state-dist-diff-bound}.
\end{proof}

\subsubsection{SNIS Approach and Zeroth-Order Gradient}
\label{app:SNIS-and-zero-order}
Fix an augmented state $\hat s=(s,z)$ and let the current local policy be Gaussian:
\begin{equation*}
a = \mu + \delta,\qquad
\mu = T_{\theta_k}(s,z),\qquad
\delta\sim \mathcal N(0,\Sigma_k).
\end{equation*}
In the E-step of iteration $k$, the non-parametric target $q_k$ takes the form
\begin{equation}
    q_k(a\mid \hat s)\propto\hat \pi(a\mid \hat s;\theta_k)\cdot \exp\left( f_{\hat s, k}(a) / \lambda\right), \qquad f_{\hat s, k}(a)=Q^{\pi_k}(s, a)-\alpha\log \tilde \pi_{\theta_k}(a\mid s), \notag
\end{equation}
\paragraph{Moment matching via SNIS.}
The target mean in iteration $k$ is given in \Eqref{eq:SNIS-moment-matching}.
Using the reparameterization $a=\mu+\delta$, define $w(\delta):=\exp\big( f_{\hat s,k}(\mu+\delta)/\lambda\big)$.
Then
\begin{equation}
    \mu^* = \frac{\mathbb E_{\delta\sim \mathcal N(0,\Sigma_k)}\left[w(\delta)\,(\mu+\delta)\right]}{\mathbb E_{\delta\sim \mathcal N(0,\Sigma_k)}\left[w(\delta)\right]}.
\end{equation}
Hence the mean shift admits the normalized weighted-residual form
\begin{equation}
\label{eq:delta_mu_def}
\Delta(\mu)
:=\mu^*-\mu
=
\frac{\mathbb E_{\delta\sim \mathcal N(0,\Sigma_k)}\left[w(\delta)\,\delta\right]}
     {\mathbb E_{\delta\sim \mathcal N(0,\Sigma_k)}\left[w(\delta)\right]}.
\end{equation}
\paragraph{Connection to a zeroth-order gradient.}
Holding the critic, base-density reference, covariance, and temperature fixed at iteration $k$, write $f=f_{\hat s,k}$ and define the log-partition function $g(\mu)$:
\begin{equation}
\label{eq: g-func-definition}
    g(\mu):=
    \log \mathbb{E}_{\delta\sim \mathcal N(0, \Sigma_k)} \left[\exp\left( f(\mu+\delta)/\lambda\right)\right]
    =
    \log \mathbb{E}_{\delta \sim \mathcal N(0, \Sigma_k)}[w(\delta)]
\end{equation}
Assuming differentiation under the expectation is valid, differentiating $g$ with respect to $\mu$ gives
\begin{equation}
\label{eq: app-delta-u-log-part}
    \nabla_{\mu}g(\mu)
    =
    \frac{\nabla_{\mu}\mathbb{E}_{\delta\sim \mathcal N(0, \Sigma_k)} \left[ w(\delta) \right]}
    {\mathbb{E}_{\delta\sim \mathcal N(0, \Sigma_k)} \left[ w(\delta) \right]}
    =
    \frac{\mathbb{E}_{\delta\sim \mathcal N(0, \Sigma_k)} \left[\nabla_{\mu}w(\delta)\right]}{\mathbb{E}_{\delta\sim \mathcal N(0, \Sigma_k)} \left[w(\delta)\right]}
\end{equation}
Since $a=\mu + \delta$ implies
\begin{equation*}
    \frac{\partial a}{\partial\delta}=I, \; \frac{\partial a}{\partial \mu}=I \quad
    \longrightarrow \quad \nabla_\delta w(\delta)=\nabla_\mu w(\delta).
\end{equation*}
When we assume that $w(\delta)$ is sufficiently smooth and integrable with the boundary term equal to zero, we apply Stein's identity for
$\delta\sim\mathcal N(0,\Sigma_k)$
\begin{equation}
\label{eq:stein}
    \mathbb E_{\delta}\left[\delta\, w(\delta)\right]
    =
    \Sigma_k\,\mathbb E_{\delta}\left[\nabla_\delta w(\delta)\right].
\end{equation}
Combining \Eqref{eq: app-delta-u-log-part} and \Eqref{eq:stein} yields
\begin{equation}
\nabla_\mu g(\mu)
=
\frac{\mathbb E_{\delta}\left[\nabla_\delta w(\delta)\right]}
     {\mathbb E_{\delta}\left[w(\delta)\right]}
=
\Sigma_k^{-1}\,
\frac{\mathbb E_{\delta}\left[\delta\, w(\delta)\right]}
     {\mathbb E_{\delta}\left[w(\delta)\right]}
=
\Sigma_k^{-1}\,\Delta(\mu).
\end{equation}
Thus the exact target mean shift is a covariance-scaled gradient of the log-partition function, which admits a zeroth-order estimate from energy evaluations:
\begin{align}
\label{eq: zeroth-order-gradient-final}
    \Delta(\mu) &=\Sigma_k\nabla_{\mu}g(\mu) \notag \\
    &=\Sigma_k
    \nabla_{\mu}
    \log
    \mathbb{E}_{\delta\sim \mathcal N(0,\Sigma_k)}
    \left[
        \exp\left( f_{\hat s,k}(\mu+\delta)/\lambda\right)
    \right].
\end{align}

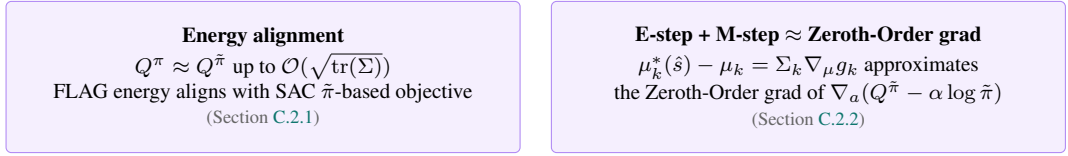
\begin{figure*}[t]
\centering
\input{figures/tikz/sac}
\caption{Comparison of SAC and FLAG at the $k$-th iteration. SAC realizes the
soft policy improvement through reparameterization-based action
gradients propagated to policy parameters via BPTT. FLAG instead solves an
EM variational lower bound whose reward is regularized by cross-entropy.
The two approaches are connected by (i) Q-function closeness
between the composite policy $\pi$ and base flow policy $\tilde\pi$ and
(ii) the relation between FLAG's moment-matching update and a zeroth-order
approximation of SAC's first-order action gradient.}
\label{fig:sac_vs_flag}
\end{figure*}

\begin{remark*}
\textbf{From SAC's soft policy improvement to FLAG.}
The soft policy improvement theorem of SAC~\citep{sac} states that any policy
$\pi_\mathrm{new}$ minimizing, at each state,
\begin{equation*}
    D_\mathrm{KL}\left(
    \pi(\cdot\mid s)
    \,\middle\|\,
    \frac{\exp(Q^{\pi_\mathrm{old}}(s,\cdot)/\alpha)}{Z(s)}
    \right)
\end{equation*}
over a policy class $\Pi$ containing $\pi_\mathrm{old}$
satisfies $Q^{\pi_\mathrm{new}} \geq Q^{\pi_\mathrm{old}}$.
The theorem is agnostic to how this KL minimizer is obtained.
SAC restricts the policy class $\Pi$ to a parametric Gaussian family and optimizes its parameters by gradient descent via reparameterization.
FLAG follows a related action-level improvement mechanism by constructing
non-parametric targets and fitting them through CFM
(Figure~\ref{fig:sac_vs_flag}).

\medskip
\noindent
\textbf{Connection to SAC's soft policy update.}
Two ingredients connect FLAG's target construction to SAC's soft-energy
update, using the reference policy
$\pi_\mathrm{old} = \tilde\pi_{\theta_k}$.

\begin{itemize}[leftmargin=1.2em, itemsep=2pt, topsep=2pt]
    \item \textbf{Approximation of the soft objective.}
    Lemma~\ref{lem: app-Q-func-bound} shows that
    $Q^{\pi_k}$ and $Q^{\tilde\pi_{\theta_k}}$ differ only by
    $\mathcal{O}(\sqrt{\mathrm{tr}(\Sigma)})$ for fixed $\gamma$.
    With the common reference term
    $-\alpha\log\tilde\pi_{\theta_k}(a\mid s)$, FLAG's energy
    $f_{\hat s,k}$ therefore approximates the values of SAC's
    $\tilde\pi$-based soft-energy objective.

    \item \textbf{Connection to the action-gradient update.}
    Appendix~\ref{app:SNIS-and-zero-order} shows that the
    moment-matching direction satisfies
    \begin{equation*}
        \mu_k^*(\hat s)-\mu_k
        =
        \Sigma_k\nabla_\mu g(\mu_k).
    \end{equation*}
    This identity connects the target displacement to a smoothed
    soft-energy gradient associated with
    \begin{equation*}
        \nabla_a
        \left(
        Q^{\pi_k}(s,a)
        -\alpha\log\tilde\pi_{\theta_k}(a\mid s)
        \right),
    \end{equation*}
    via Stein's identity~(\Eqref{eq:stein}) and log-sum-exp
    smoothing~(\Eqref{eq: zeroth-order-gradient-final}).
    SAC propagates its action-gradient signal through the actor,
    whereas FLAG estimates the target displacement from energy
    evaluations and uses it to construct supervision for the actor.
\end{itemize}

\noindent
Together, these two ingredients establish an action-level connection
between FLAG and SAC.
Whenever the resulting FLAG policy decreases SAC's statewise KL
objective above relative to the old policy, it inherits the same
soft policy improvement guarantee.
The conditional improvement bound for FLAG's objective, including
reward drift and CFM approximation, is analyzed separately in
Appendix~\ref{appendix:monotonic-improvement}.

\medskip
\noindent
\textbf{What is distinctive about FLAG.}
FLAG translates action-level improvement signals into target actions
$\{\mu_k^*(\hat s)\}_{\hat s}$ for supervised policy training.
The M-step CFM distillation~(\Eqref{eq: M-step-CFM}) provides a
supervised, BPTT-free mechanism for training the flow policy to
match these targets.
In this way, FLAG provides an approximate, target-based realization
of SAC-style policy improvement for expressive flow-based policies.
\end{remark*}
\subsection{Monotonic Improvement by Maximum a Posteriori Policy Optimisation}
\label{appendix:monotonic-improvement}
We now state a conditional monotonic-improvement result for FLAG in the
latent-augmented MDP. We follow the MPO framework \citep{mpo} and account
for the change in the cross-entropy augmented reward. The frozen-reward
improvement and the FLAG-specific conditions are stated explicitly below.

\paragraph{Setup and Notation.}
We work in a $z$-MDP with finite augmented state space
$\hat{\mathcal S}$ and continuous action coordinates $\mathcal A=\mathbb R^{d}$.
The local Gaussian policy is defined directly on these coordinates:
\begin{equation}
\label{eq:projected-gaussian-policy}
\hat\pi_{\theta_k}(a\mid\hat s)
:=\varphi_{\Sigma_k}\big(a-\mu_{\theta_k}(\hat s)\big),
\qquad \mu_{\theta_k}(\hat s)=T_{\theta_k}(s,z),
\qquad \Sigma_k=\sigma_k^2 I.
\end{equation}
All Gaussian mean-shift and covariance calculations below are performed in
these continuous coordinates. 
When actions are squashed, rewards and log-density terms are evaluated
after the corresponding change of variables (see Appendix \ref{app:implementation_details_action_scaling}). 

Let $\hat p_0$ be a fixed initial augmented-state distribution.
During each policy update, the covariance $\Sigma_k \succ0$ and temperatures $\alpha$, $\lambda>0$ are held fixed, and policy evaluation is assumed to be exact.
All expectations, KL divergences, and target moments used below are assumed
finite, and differentiation under the relevant integrals is valid.
For any policy $q$, the induced transition matrix is
\begin{equation*}
P^q(\hat s,\hat s')
=\int_{\mathcal A}q(a\mid\hat s)\hat p(\hat s'\mid\hat s,a)\,da.
\end{equation*}
Note that this is a finite-dimensional stochastic matrix. Thus continuous actions preserve the Bellman and transition-matrix arguments used in
Appendix A of \citet{mpo}, with action sums interpreted as integrals
under the stated regularity conditions.
For notational simplicity, we write 
\begin{equation*}
    \hat\pi_k(\cdot\mid \hat s) := \hat \pi_{\theta_k}(\cdot\mid \hat s), \qquad \tilde \pi_k(\cdot\mid s) :=\tilde \pi_{\theta_k}(\cdot\mid s).
\end{equation*}
The reward in iteration $k$ is given as 
\begin{equation*}
\hat r_k(\hat s,a)
:=
r(s,a)-\alpha \log \tilde\pi_k(a\mid s).
\end{equation*}
For a reference policy $\hat\pi$ and a non-parametric policy $q$, define the $\lambda$-regularized reward
\begin{equation*}
\hat r_{k,\lambda}^{\hat\pi,q}(\hat s,a) := \hat r_k(\hat s,a)-\lambda \log \frac{q(a\mid \hat s)}{\hat\pi(a\mid \hat s)}.
\end{equation*}

The corresponding Bellman operators are
\begin{align}
(\hat{\mathcal T}_k^q \hat V)(\hat s) &= \mathbb E_{a\sim q(\cdot\mid \hat s)} \Big[
\hat r_k(\hat s,a) +\gamma \mathbb E_{\hat s'\sim \hat p(\cdot\mid \hat s,a)}\big[\hat V(\hat s')\big] \Big], \\
(\hat{\mathcal T}_{k,\lambda}^{\hat\pi,q}\hat V)(\hat s)
&= \mathbb E_{a\sim q(\cdot\mid \hat s)}
\Big[ \hat r_{k,\lambda}^{\hat\pi,q}(\hat s,a) +\gamma \mathbb E_{\hat s'\sim \hat p(\cdot\mid \hat s,a)}[\hat V(\hat s')] \Big].
\end{align}
We also define the corresponding value functions
\begin{align}
\hat V_k^q(\hat s)
&= \mathbb E_q\!\left[ \sum_{t=0}^{\infty}\gamma^t \hat r_k(\hat s_t,a_t)
\,\middle|\,\hat s_0=\hat s \right], \\
\hat V_{k,\lambda}^{\hat\pi,q}(\hat s)
&=
\mathbb E_q\!\left[ \sum_{t=0}^{\infty}\gamma^t \left(
\hat r_k(\hat s_t,a_t) -\lambda \log \frac{q(a_t\mid \hat s_t)}{\hat\pi(a_t\mid \hat s_t)}
\right) \,\middle|\,\hat s_0=\hat s
\right].
\end{align}
In particular, the augmented-reward action value satisfies
\begin{equation*}
\hat Q_k^{\hat\pi_k}(\hat s,a)
=Q^{\pi_k}(s,a)-\alpha\log\tilde\pi_k(a\mid s)=f_{\hat s,k}(a).
\end{equation*}
Using the energy $f_{\hat s, k}(a)$ defined in \Eqref{eq: energy-func},
the target of the constrained E-step with fixed KL multiplier $\lambda$ in \Eqref{eq: const-e-step} is given by
\begin{equation*}
q_k(\cdot\mid \hat s) = \argmax_q \;
\mathbb E_{a\sim q(\cdot\mid \hat s)}\big[f_{\hat s,k}(a)\big]
-\lambda \big(D_{\mathrm{KL}}\big(q(\cdot\mid \hat s)\,\|\,\hat\pi_k(\cdot\mid \hat s) \big)-\epsilon\big),
\end{equation*}
which admits the closed form
\begin{equation}
\label{eq:flag-qstar-closed-form}
q_k(a\mid \hat s) = \frac{\hat\pi_k(a\mid \hat s)\exp\!\big(f_{\hat s,k}(a)/\lambda\big)}{Z_k(\hat s)}, \qquad 
Z_k(\hat s) = \int_{\mathcal A} \hat\pi_k(a\mid \hat s)\exp\!\big(f_{\hat s,k}(a)/\lambda\big)\,da.
\end{equation}

Finally, define the exact KL-projection objective
\begin{equation}
\label{eq:flag-exact-kl-projection-objective}
h(\hat \pi, q, \theta) :=
\mathbb E_{\hat p_0,\hat \pi,\hat p}\left[ 
\sum_{t=0}^{\infty} \gamma^t D_\text{KL}\big( q(\cdot\mid \hat s_t) \;\|\; \hat \pi(\cdot\mid \hat s_t;\theta)\big)
\right].
\end{equation}
For fixed isotropic covariance, $h$ equals the first-moment matching objective (Eq. \ref{eq: M-step-regression}), scaled by $1/(2\sigma_k^2)$, up to a $\theta$-independent constant.

\begin{proposition}[Bellman operator and Monotonicity]
\label{prop:flag-policy-eval-monotone}
For every iteration $k$, the E-step target $q_k$ satisfies
\begin{equation}
\label{eq:flag-policy-eval-ineq}
\hat{\mathcal{T}}_{k,\lambda}^{\hat{\pi}_k,q_k}\hat{V}_k^{\hat{\pi}_k} \ge \hat{V}_k^{\hat{\pi}_k}.
\end{equation}
Consequently, by the monotonicity of the Bellman operator, the regularized value function satisfies the following improvement for all states $\hat{s}$:
\begin{equation}
\label{eq:flag-regularized-value-improvement}
\hat{V}_{k,\lambda}^{\hat{\pi}_k,q_k} \ge \hat{V}_k^{\hat{\pi}_k}.
\end{equation}
\end{proposition}

\begin{proof}
The proof closely follows the policy evaluation argument of MPO, adapted to our $z$-MDP and cross-entropy reward structure.
By the optimality of the E-step target $q_k$,  for every $\hat{s} \in \hat{\mathcal{S}}$,
\begin{equation*}
    (\hat{\mathcal{T}}_{k,\lambda}^{\hat{\pi}_k,q_k}\hat{V}_k^{\hat{\pi}_k})(\hat{s}) \ge (\hat{\mathcal{T}}_{k,\lambda}^{\hat{\pi}_k,\hat{\pi}_k}\hat{V}_k^{\hat{\pi}_k})(\hat{s}) = \hat{V}_k^{\hat{\pi}_k}
\end{equation*}
Under the assumption (Assumption \ref{assum:bounded_reward}), the regularized Bellman operator is monotone and $\gamma$-contractive. Repeated application therefore produces a nondecreasing sequence converging to its fixed point, $\hat V_{k, \lambda}^{\hat \pi_k, q_k}$. Consequently,
\begin{equation*}
    \hat{V}_{k,\lambda}^{\hat{\pi}_k,q_k} \ge \hat{V}_k^{\hat{\pi}_k}.
\end{equation*}
\end{proof}

\paragraph{Objective decomposition.}
Next, we establish a conditional MPO-style one-step improvement bound for the EM update in Section~\ref{subsection: local-policy-update}. In standard MPO, the reward is fixed across the E-step and M-step, so the EM update can be interpreted as a coordinate-ascent step on a KL-regularized variational objective. In FLAG, however, the augmented reward
\begin{equation*}
    \hat r_{\color{myteal}k}(\hat s,a) = r(s,a)-\alpha\log\tilde\pi_{\color{myteal}\theta_k}(a\mid s)
\end{equation*}
also depends on the base flow policy. Therefore, after the base flow is updated
from $\theta_k$ to $\theta_{k+1}$, the reward changes from $\hat r_k$ to
$\hat r_{k+1}$. The proof below separates these two effects:
\begin{equation}
\label{eq: mpo-obj-decomp}
    \Big(\;\underbrace{\textit{frozen-reward MPO improvement}}_{\tcircled{1}}\;\Big)
    \qquad+\qquad
    \Big(\;\underbrace{\textit{cross-entropy reward drift}}_{\tcircled{2}}\;\Big).
\end{equation}

To match the regularized value function defined above, we multiply the objective in \Eqref{eq:em-objective} by $\lambda$. Since $\lambda > 0$ is fixed during the update, this scaling preserves the maximizers and the direction of improvement.
The trajectory KL is interpreted as the discounted conditional KL,
$\mathbb E_{\hat p_0,q,\hat p}\left[\sum_{t\ge0}\gamma^t
D_{\mathrm{KL}}(q(\cdot\mid\hat s_t)\|\hat\pi_\theta(\cdot\mid\hat s_t))\right]$.
For brevity, define
\begin{equation}
\label{eq:flag-Jk-definition}
    \mathcal J_k:=\lambda\mathcal J(q_k, \theta_k) = \mathbb{E}_{\hat s_0\sim \hat p_0}\left[ \hat V_{k,\lambda}^{\hat\pi_k,q_k}(\hat s_0)\right].
\end{equation}
For theoretical analysis, we first consider the ideal KL-projection update
\begin{equation}
\label{eq:flag-kl-update}
    \theta_{k+1}^{\mathrm{KL}}
    =
    \theta_k-\beta\nabla_\theta h(\hat\pi_k,q_k,\theta_k),
\qquad
\mathcal G_k:=\tfrac12\|\nabla_\theta h(\hat\pi_k,q_k,\theta_k)\|^2,
\end{equation}
and denote the corresponding policy by
$\hat\pi_{k+1}^{\mathrm{KL}}(\cdot\mid\hat s)
:=
\hat\pi(\cdot\mid\hat s;\theta_{k+1}^{\mathrm{KL}})$.
Unless a CFM superscript is present, the next target $q_{k+1}$ and base density $\tilde\pi_{k+1}$ correspond to this ideal iterate, including the refreshed reward $\hat r_{k+1}$.
The practical CFM projection update is handled later in
Lemma~\ref{lem:flag-full-monotone-cfm}.

\begin{assumption}[MPO-style regularity and update conditions]
\label{ass:flag-mpo-style}
For each iteration $k$, the following hold:
\begin{enumerate}[leftmargin=1.2em]
    \item The KL objective $h(\hat\pi_k,q_k,\theta)$ is $L$-smooth
    in a neighborhood of $\theta_k$.

    \item In a neighborhood of $\theta_k$, the score
    $\nabla_\theta\log\tilde\pi_\theta(a\mid s)$ satisfies
    \begin{equation*}
        \|\nabla_\theta\log\tilde\pi_\theta(a\mid s)\|
        \le L_{\tilde\pi}
    \end{equation*}
    for some $L_{\tilde\pi}>0$ and is Lipschitz continuous in $\theta$,
    with both bounds uniform over $(s,a)$.

    \item The E-step target and the KL-projected policy satisfy
    \begin{align*}
        \sup_{\hat s}
        \|q_{k+1}(\cdot\mid\hat s)-q_k(\cdot\mid\hat s)\|_1
        &=\mathcal O(\beta),\\
        \sup_{\hat s}
        \|\hat\pi_{k+1}^{\mathrm{KL}}(\cdot\mid\hat s)
        -\hat\pi_k(\cdot\mid\hat s)\|_1
        &=\mathcal O(\beta),
    \end{align*}
    where $\beta$ is the M-step size.
    
    \item For sufficiently small $\beta$, the actual next pair
    $(\hat\pi_{k+1}^{\mathrm{KL}},q_{k+1})$ satisfies the
    frozen-reward improvement bound in
    \Eqref{eq:flag-frozen-mpo-bound}, with a constant $C_1$
    independent of $\beta$.
\end{enumerate}
\end{assumption}
The final item is an additional update condition requiring
frozen-reward improvement for the actual next iterate,
including the refreshed E-step target.

\begin{assumption}[Variance-scaled drift alignment]
\label{ass:flag-drift-alignment}
Assume that, within iteration $k$, the local covariance is fixed with respect
to $\theta$, state-independent, and isotropic:
\begin{equation*}
    \Sigma_k=\sigma_k^2 I,\qquad \sigma_k^2>0.
\end{equation*}
Here $\Sigma_k$ denotes the covariance matrix for notational convenience, while
$\sigma_k^2$ denotes its scalar variance. Let the first-order reward-drift score be
\begin{equation*}
\mathsf G_k^{\tilde\pi}
:=
\mathbb E_{\hat p_0,q_k,\hat p}\!\left[
\sum_{t=0}^{\infty}
\gamma^t
\nabla_\theta \log \tilde\pi_{\theta_k}(a_t\mid s_t)
\right].
\end{equation*}
There exists a structural constant $C_\Sigma>0$, independent of the step size
$\beta$, such that
\begin{equation}
\label{eq:flag-drift-alignment}
\Big|
\big\langle
\mathsf G_k^{\tilde\pi},
\nabla_\theta h(\hat\pi_k,q_k,\theta_k)
\big\rangle
\Big|
\le
C_\Sigma\sigma_k^2\mathcal G_k .
\end{equation}
\end{assumption}

Under Assumption~\ref{ass:flag-mpo-style},
for sufficiently small $0\le \beta\le 1/L$, the descent lemma~\citep{nesterov2013introductory} gives
\begin{equation}
\label{eq:flag-exact-m-step-shifted}
h(\hat\pi_k,q_k,\theta_{k+1}^{\mathrm{KL}})
\le
h(\hat\pi_k,q_k,\theta_k)
-
\beta \mathcal G_k.
\end{equation}
This establishes descent of the KL projection objective.
The corresponding frozen-reward value improvement is the additional
update condition in Assumption~\ref{ass:flag-mpo-style}.
The cross-entropy term $-\alpha\log\tilde\pi_\theta(a\mid s)$ also changes
after the update, and its effect is controlled by the Assumption~\ref{ass:flag-drift-alignment}.

\paragraph{Interpretation of Assumption~\ref{ass:flag-drift-alignment}}
The following calculation motivates the variance scaling in the alignment assumption using the exact mean-shift identity from Appendix~\ref{app:SNIS-and-zero-order}. This arises because the M-step parameter update $\Delta\theta_k$, which governs the drift, is shaped by the zeroth-order moment-matching direction. Following Section~\ref{subsection: policy-parameterization}, we assume an isotropic covariance $\Sigma_k=\sigma_k^2 I$.

Here $g_k$ denotes the iteration-$k$ log-partition function in \Eqref{eq: g-func-definition}, with the state argument suppressed.
Using the exact mean-shift identity $\mu_k^*(\hat s)-\mu_k(\hat s) = \sigma_k^2 \nabla_\mu g_k(\mu_k(\hat s))$ from \Eqref{eq: zeroth-order-gradient-final}, the statewise M-step gradient with respect to $\theta$ is simplified to:
\begin{equation*}
-\nabla_\theta D_{\mathrm{KL}}\big( q_k(\cdot\mid \hat s) \;\|\; \hat \pi(\cdot\mid \hat s; \theta)\big) =
J_\theta(\hat s)^\top \sigma_k^{-2} \big(\mu_k^*(\hat s)-\mu_k(\hat s)\big) = J_\theta(\hat s)^\top v_k(\hat s),
\end{equation*}
where $J_\theta(\hat s) := \partial \mu_k(\hat s)/\partial \theta$ and $v_k(\hat s) := \nabla_\mu g_k(\mu_k(\hat s))$. 
Consequently, the explicit $\sigma_k^{-2}$ factor cancels in the squared statewise gradient norm, although $v_k$ can still depend on the covariance:
\begin{equation*}
\left\|\nabla_\theta D_{\mathrm{KL}}\big( q_k(\cdot\mid \hat s) \;\|\; \hat \pi(\cdot\mid \hat s; \theta)\big)\right\|^2 = v_k(\hat s)^\top J_\theta(\hat s)J_\theta(\hat s)^\top v_k(\hat s).
\end{equation*}

With the notation $\Delta\mu_k(\hat s)=\mu_k^*(\hat s)-\mu_k(\hat s)$,
the associated statewise mean-shift surrogate is:
\begin{align}
    J_\theta(\hat s)^\top \Delta\mu_k(\hat s)
    &= J_\theta(\hat s)^\top \Big( \mathbb E_{a\sim q_k(\cdot\mid \hat s)}[a] - \mu_k(\hat s) \Big) \notag \\
    &= J_\theta(\hat s)^\top \frac{ \mathbb E_{\delta\sim \mathcal N(0,\Sigma_k)} \left[
    \delta \exp\left(f_{\hat s,k}(\mu_k+\delta)/\lambda\right) \right] }
    { \mathbb E_{\delta\sim \mathcal N(0,\Sigma_k)} \left[ \exp\left(f_{\hat s,k}(\mu_k+\delta)/\lambda\right)\right]}
    \notag\\
    &= J_\theta(\hat s)^\top \Sigma_k \nabla_\mu g_k(\mu_k(\hat s)) \notag\\
    &= J_\theta(\hat s)^\top \Sigma_k v_k(\hat s).
\label{eq:flag-drift-zo-expansion}
\end{align}

Using the same frozen discounted state weighting as $h$, define the mean-shift surrogate
\begin{equation}
\label{eq:flag-mean-drift-surrogate}
\begin{aligned}
\bar{\mathsf G}_k^{\tilde\pi} &:= \mathbb E_{\hat p_0,\hat\pi_k,\hat p}
\left[ \sum_{t=0}^{\infty}\gamma^t J_\theta(\hat s_t)^\top
\Delta\mu_k(\hat s_t) \right] \\
&= -\sigma_k^2\nabla_\theta h(\hat\pi_k,q_k,\theta_k),
\end{aligned}
\end{equation}
where the equality follows from the statewise Gaussian KL gradient
derived above. Consequently,
\begin{equation}
\label{eq:flag-global-matrix-inner-product}
\begin{aligned}
\left| \left\langle \bar{\mathsf G}_k^{\tilde\pi},
\nabla_\theta h(\hat\pi_k,q_k,\theta_k) \right\rangle \right| &=
\sigma_k^2 \left\|\nabla_\theta h(\hat\pi_k,q_k,\theta_k)\right\|^2\\
&= 2\sigma_k^2\mathcal G_k.
\end{aligned}
\end{equation}
This identity motivates the $\sigma_k^2\mathcal G_k$ scaling in
Assumption~\ref{ass:flag-drift-alignment}.
The assumption requires an analogous bound, with constant $C_\Sigma$,
for the actual reward-drift score $\mathsf G_k^{\tilde\pi}$, which
involves the base-flow log-density and trajectories under $q_k$.

\begin{proposition}[Monotone Improvement in the z-MDP]
\label{prop:flag-full-monotone-exact}
Under Assumptions~\ref{ass:flag-mpo-style} and ~\ref{ass:flag-drift-alignment}, there exists a constant $C>0$, independent of $\beta$, such that, for all sufficiently small $\beta>0$, the ideal KL update satisfies
\begin{equation}
\label{eq:flag-full-monotone-exact}
\mathcal J_{k+1}^{\mathrm{KL}}
\ge \mathcal J_k + (\lambda-\alpha C_\Sigma \sigma_k^2)\beta \mathcal G_k
- C\beta^2,
\end{equation}
where
\begin{equation*}
    \mathcal J_{k+1}^{\mathrm{KL}} := \mathbb{E}_{\hat s_0\sim \hat p_0}\Big[\hat V_{k+1,\lambda}^{\hat\pi_{k+1}^{\mathrm{KL}},q_{k+1}}(\hat s_0)\Big].
\end{equation*}
In particular, if $\mathcal G_k>0$ and
\begin{equation}
\lambda > \alpha C_\Sigma \sigma_k^2,
\end{equation}
then, for sufficiently small $\beta$,
\begin{equation}
\mathcal J_{k+1}^{\mathrm{KL}} \ge \mathcal J_k.
\end{equation}
If $\mathcal G_k=0$, the ideal step is zero and the objective is unchanged
under the fixed-covariance, fixed-temperature comparison.
\end{proposition}

\begin{proof} \mbox{} \\
\textbf{Step 1: } We first write the decomposition point-wise for a fixed initial state $\hat s_0$.
Taking expectation over $\hat s_0\sim \hat p_0$ then gives the corresponding statement for
$\mathcal J_{k+1}^{\mathrm{KL}}-\mathcal J_k$.
\begin{align}
\label{eq:flag-proof-split}
\mathcal J_{k+1}^{\mathrm{KL}}-\mathcal J_k &=\mathbb{E}_{\hat s_0\sim \hat p_0}\left[\hat V_{k+1, \lambda}^{\hat \pi_{k+1}^{\text{KL}}, q_{k+1}}(\hat s_0) - \hat V_{k, \lambda}^{\hat \pi_{k}, q_{k}}(\hat s_0)\right] \notag \\
&=\mathbb{E}_{\hat s_0\sim \hat p_0}\bigg[\underbrace{
\Big(
\hat V_{k,\lambda}^{\hat\pi_{k+1}^{\mathrm{KL}},q_{k+1}}(\hat s_0)
-
\hat V_{k,\lambda}^{\hat\pi_k,q_k}(\hat s_0)
\Big)
}_{\ttcircled{1}\;\text{in }\textnormal{\Eqref{eq: mpo-obj-decomp}}}
+
\underbrace{
\Big(
\hat V_{k+1,\lambda}^{\hat\pi_{k+1}^{\mathrm{KL}},q_{k+1}}(\hat s_0)
-
\hat V_{k,\lambda}^{\hat\pi_{k+1}^{\mathrm{KL}},q_{k+1}}(\hat s_0)
\Big)
}_{\ttcircled{2}\;\text{in }\textnormal{\Eqref{eq: mpo-obj-decomp}}}\bigg].
\end{align}
The first bracket freezes the reward at $\hat r_k$, and is the MPO-style improvement term required by Assumption~\ref{ass:flag-mpo-style}. The second bracket is the only new term in FLAG, and it appears because the reward changes when $\tilde\pi_k$ is updated to $\tilde\pi_{k+1}$.

\textbf{Step 2: }
Holding the augmented reward $\hat r_k$ fixed, consider the auxiliary objective
\begin{equation*}
\mathcal H_k(\hat\pi,q,\theta,\hat\pi')
:= \mathbb E_{\hat p_0,\hat\pi,\hat p}\left[ \sum_{t=0}^{\infty}\gamma^t
\left( \mathbb E_{a_t\sim q(\cdot\mid\hat s_t)} \big[\hat Q_k^{\hat\pi'}(\hat s_t,a_t)\big] - \lambda D_{\mathrm{KL}}\!\big(q(\cdot\mid\hat s_t)
\,\|\, \hat\pi(\cdot\mid\hat s_t;\theta) \big) \right) \right],
\end{equation*}
where
$\hat Q_k^{\hat\pi'}(\hat s,a) := \hat r_k(\hat s,a) + \gamma\mathbb E_{\hat s'\sim\hat p(\cdot\mid\hat s,a)} [\hat V_k^{\hat\pi'}(\hat s')]$.
With $\hat\pi_k$, $q_k$, and the critic held fixed, only the KL term
in $\mathcal H_k$ depends on $\theta$.
Thus, by the KL-descent bound in
\Eqref{eq:flag-exact-m-step-shifted},
\begin{align*}
\mathcal H_k(\hat\pi_k,q_k,\theta_{k+1}^{\mathrm{KL}},\hat\pi_k)
-\mathcal H_k(\hat\pi_k,q_k,\theta_k,\hat\pi_k)
&=
\lambda\left[
h(\hat\pi_k,q_k,\theta_k)
-
h(\hat\pi_k,q_k,\theta_{k+1}^{\mathrm{KL}})
\right]\\
&\ge \lambda\beta\mathcal G_k.
\end{align*}
For the actual next pair
$(\hat\pi_{k+1}^{\mathrm{KL}},q_{k+1})$,
the additional update condition in
Assumption~\ref{ass:flag-mpo-style} requires the corresponding
value improvement up to a second-order remainder:
\begin{equation}
\label{eq:flag-frozen-mpo-bound}
\mathbb E_{\hat s_0\sim\hat p_0}\left[
\hat V_{k,\lambda}^{\hat\pi_{k+1}^{\mathrm{KL}},q_{k+1}}(\hat s_0)
-
\hat V_{k,\lambda}^{\hat\pi_k,q_k}(\hat s_0)
\right]
\ge
\lambda\beta\mathcal G_k-C_1\beta^2
\end{equation}
for sufficiently small $\beta$, where $C_1>0$ is independent of $\beta$.
This provides the bound on the first term in
\Eqref{eq:flag-proof-split}.

The comparison above holds the reward fixed at $\hat r_k$, whereas
$q_{k+1}$ is computed using the updated reward $\hat r_{k+1}$.
Since $\Delta\theta_k=\mathcal O(\beta)$, the bounded score in
Assumption~\ref{ass:flag-mpo-style} implies
$\|\hat r_{k+1}-\hat r_k\|_\infty=\mathcal O(\beta)$.
For a fixed policy, the resulting change in the Q-function is also
$\mathcal O(\beta)$. Together with the assumed $\mathcal O(\beta)$
change in the E-step target, this gives
\begin{equation*}
    \big\|\hat Q_{k+1}^{\hat\pi_{k+1}}-\hat Q_{k}^{\hat\pi_{k+1}}\big\|_\infty
    = \mathcal{O}(\beta),
    \qquad
    \sup_{\hat s}\,\|q_{k+1}(\cdot\mid\hat s)-q_k(\cdot\mid\hat s)\|_1
    = \mathcal{O}(\beta).
\end{equation*}
Comparing the statewise auxiliary objectives introduces the cross-term
of these two quantities,
\begin{equation*}
    \big(\mathbb{E}_{q_{k+1}}-\mathbb{E}_{q_k}\big)
    \big[\hat Q_{k}^{\hat\pi_{k+1}}-\hat Q_{k+1}^{\hat\pi_{k+1}}\big]
    = \mathcal{O}(\beta^2).
\end{equation*}
This cross-term is therefore $\mathcal O(\beta^2)$.
The value bound in \Eqref{eq:flag-frozen-mpo-bound} for the actual
next pair $(\hat\pi_{k+1}^{\mathrm{KL}},q_{k+1})$ is assumed separately
in Assumption~\ref{ass:flag-mpo-style}.

\textbf{Step 3: }
Due to the moving base flow policy, the reward-drift term in \Eqref{eq:flag-proof-split} is given by
\begin{align}
\label{eq:app-reward-drift-expand}
&\hat V_{k+1,\lambda}^{\hat\pi_{k+1}^{\mathrm{KL}},q_{k+1}}(\hat s_0)
-
\hat V_{k,\lambda}^{\hat\pi_{k+1}^{\mathrm{KL}},q_{k+1}}(\hat s_0)
=
\alpha\,
\mathbb E_{q_{k+1}} \left[
\sum_{t=0}^{\infty}\gamma^t
\Big(
\log \tilde\pi_k(a_t\mid s_t)-\log \tilde\pi_{k+1}^{\text{KL}}(a_t\mid s_t)
\Big)
\,\middle|\,\hat s_0
\right].
\end{align}
We define the reward-drift term $\Delta_k^{\mathrm{drift}}$ as the average of this quantity over $\hat s_0\sim\hat p_0$.

Let the parameter update be $\Delta\theta_k := \theta_{k+1}^{\mathrm{KL}}-\theta_k = -\beta \nabla_\theta h(\hat\pi_k,q_k,\theta_k)$. By Taylor expanding $\log \tilde\pi_{\theta_{k+1}^{\mathrm{KL}}}(a\mid s)$ around $\theta_k$, using the Lipschitz score for a quadratic remainder, we obtain:
\begin{align}
\Delta_k^{\mathrm{drift}}
&=
-\alpha\,
\mathbb E_{\hat p_0,q_{k+1},\hat p} \left[
\sum_{t=0}^{\infty}\gamma^t
\nabla_\theta \log \tilde\pi_{\theta_k}(a_t\mid s_t)
\right]^\top
\Delta\theta_k
+
\mathcal O(\|\Delta\theta_k\|^2)
\notag\\
&=
\alpha\beta\,
\mathbb E_{\hat p_0,q_{k+1},\hat p} \left[
\sum_{t=0}^{\infty}\gamma^t
\nabla_\theta \log \tilde\pi_{\theta_k}(a_t\mid s_t)
\right]^\top
\nabla_\theta h(\hat\pi_k,q_k,\theta_k)
+
\mathcal O(\beta^2).
\label{eq:flag-drift-first-order-qkp1}
\end{align}

Recall that $\mathsf G_k^{\tilde\pi}$ is defined using trajectories
under $q_k$. The policy-change and score bounds in
Assumption~\ref{ass:flag-mpo-style} imply
\begin{equation*}
\mathbb E_{\hat p_0,q_{k+1},\hat p}
\left[
\sum_{t=0}^{\infty}\gamma^t
\nabla_\theta\log\tilde\pi_{\theta_k}(a_t\mid s_t)
\right]
=
\mathsf G_k^{\tilde\pi}+\mathcal O(\beta).
\end{equation*}
The error accounts for changes in both actions and state visitation.
Since this expectation is multiplied by $\alpha\beta$ in
\Eqref{eq:flag-drift-first-order-qkp1}, replacing it by
$\mathsf G_k^{\tilde\pi}$ contributes an $\mathcal O(\beta^2)$
error to the drift. Therefore,
\begin{equation}
\label{eq:flag-drift-rigorous-bound}
\Delta_k^{\mathrm{drift}}
=
\alpha\beta
\big\langle
\mathsf G_k^{\tilde\pi},
\nabla_\theta h(\hat\pi_k,q_k,\theta_k)
\big\rangle
+
\mathcal O(\beta^2).
\end{equation}
Plugging the Assumption~\ref{ass:flag-drift-alignment} into \Eqref{eq:flag-drift-rigorous-bound}, the final bound on the reward-drift magnitude is
\begin{equation}
\label{eq:flag-drift-aligned-bound}
\big|\Delta_k^{\mathrm{drift}}\big|
\le \alpha C_\Sigma \sigma_k^2\beta \mathcal G_k + C_d\beta^2
\end{equation}
for some constant $C_d>0$ independent of $\beta$.

Substituting the reward-freeze improvement (\ttcircled{1} in \Eqref{eq:flag-frozen-mpo-bound}) and reward-drift bound (\ttcircled{2} in \Eqref{eq:flag-drift-aligned-bound}) into the objective decomposition (\ttcircled{1}$\;+\;$\ttcircled{2} in \Eqref{eq:flag-proof-split}), we conclude:
\begin{align}
\mathcal J_{k+1}^{\mathrm{KL}}-\mathcal J_k
& \geq \lambda \beta \mathcal G_k - C_1\beta^2 - \big|\Delta_k^{\mathrm{drift}}\big| \notag \\
&\geq \lambda \beta \mathcal G_k - C_1\beta^2
-\big( \alpha C_\Sigma \sigma_k^2\beta \mathcal G_k + C_d\beta^2\big) \notag\\
&= (\lambda-\alpha C_\Sigma \sigma_k^2)\beta \mathcal G_k - (C_1+C_d)\beta^2.
\label{eq:flag-full-monotone-final}
\end{align}
Hence, if $\mathcal G_k>0$ and
\begin{equation}
\label{eq:flag-sufficient-gain}
\lambda > \alpha C_\Sigma \sigma_k^2,
\end{equation}
then, for a sufficiently small step size $\beta$, the overall objective is nondecreasing:
\begin{equation*}
    \mathcal J_{k+1}^{\mathrm{KL}} \ge \mathcal J_k.
\end{equation*}
If $\mathcal G_k=0$, then $\nabla_\theta h=0$ and
$\theta_{k+1}^{\mathrm{KL}}=\theta_k$.
With fixed covariance and temperatures and exact evaluation, the reward
and the exact E-step target are unchanged, so
$\mathcal J_{k+1}^{\mathrm{KL}}=\mathcal J_k$.
\end{proof}

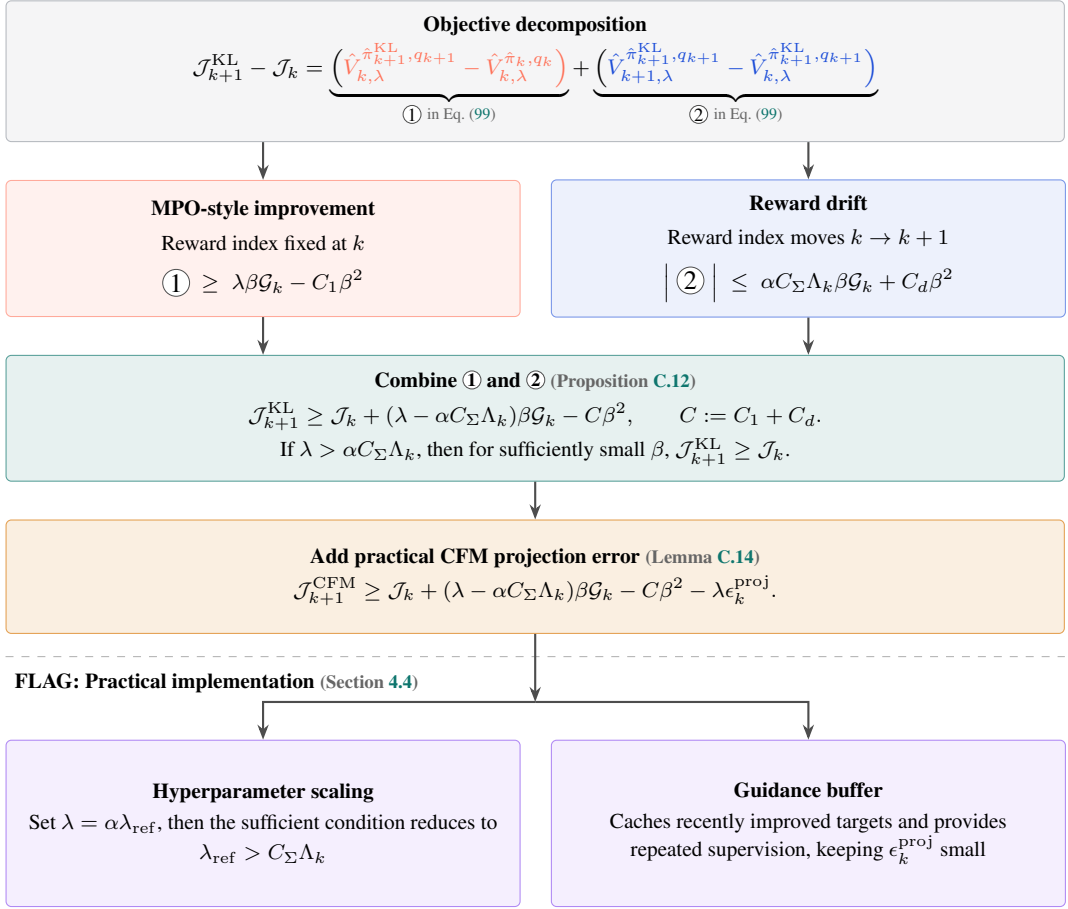
\begin{figure*}[t]
\centering
\input{figures/tikz/mpo}
\caption{
Proof roadmap for the conditional MPO-style monotonic improvement of FLAG.
The objective gap is decomposed into the frozen-reward improvement input
(\protect\ttcircled{1} in \Eqref{eq: mpo-obj-decomp}) and the FLAG-specific reward-drift term (\protect\ttcircled{2} in \Eqref{eq: mpo-obj-decomp}).
Combining the two yields Proposition~\ref{prop:flag-full-monotone-exact},
and adding the practical CFM projection error yields
Lemma~\ref{lem:flag-full-monotone-cfm}.
The bottom row shows two practical implementation choices motivated by the
sufficient conditions: hyperparameter scaling
$\lambda = \alpha\lambda_{\mathrm{ref}}$ and the guidance buffer for keeping
$\epsilon_k^{\mathrm{proj}}$ small
(Section~\ref{subsection: implementation-details}).
}
\label{fig:flag-mpo-flow}
\vspace{-1em}
\end{figure*}

\paragraph{Practical CFM update.}
Proposition~\ref{prop:flag-full-monotone-exact} analyzes the KL-based M-step.
In practice, however, FLAG distills the resulting
target into the base flow policy via the CFM loss (\Eqref{eq: M-step-CFM}).
To bound the gap, we relate the practical update to the KL update through an excess projection error,
rather than claiming that the CFM loss itself decreases $h(\hat \pi_k, q_k, \theta_k)$.

\begin{assumption}[Approximate realization of the ideal KL update by CFM]
\label{ass:flag-cfm-proj}
There exists $\epsilon_k^{\mathrm{proj}} \ge 0$ such that the practical CFM update satisfies
\begin{equation}
\label{eq:flag-cfm-proj-assumption}
\mathcal J_{k+1}^{\mathrm{CFM}} \ge \mathcal J_{k+1}^{\mathrm{KL}} - \lambda \epsilon_k^{\mathrm{proj}},
\end{equation}
where, for $X\in\{\mathrm{KL},\mathrm{CFM}\}$,
\begin{equation*}
\mathcal J_{k+1}^{X}
:=\mathbb E_{\hat s_0\sim\hat p_0}
\left[\hat V_{k+1,\lambda}^{\hat\pi_{k+1}^{X},q_{k+1}^{X}}(\hat s_0)\right].
\end{equation*}
Each branch starts from the same iteration-$k$ quantities and uses its own
updated base-density reference $\tilde\pi_{\theta_{k+1}^{X}}$ in the reward
and its corresponding refreshed target $q_{k+1}^{X}$.
For the KL branch, $q_{k+1}^{\mathrm{KL}}=q_{k+1}$ as used above.
The quantity $\epsilon_k^{\mathrm{proj}}$ bounds the gap in this objective;
it is not identified with the CFM training loss.

\end{assumption}

\begin{lemma}[Approximate monotone improvement under the practical CFM update]
\label{lem:flag-full-monotone-cfm}
Under Assumptions~\ref{ass:flag-mpo-style},~\ref{ass:flag-drift-alignment} and~\ref{ass:flag-cfm-proj}, there exists a constant $C>0$,
independent of $\beta$, such that, for all sufficiently small $\beta>0$,
\begin{equation}
\label{eq:flag-full-monotone-cfm}
\mathcal J_{k+1}^{\mathrm{CFM}}
\ge \mathcal J_k + (\lambda-\alpha C_\Sigma \sigma_k^2)\beta \mathcal G_k -
C\beta^2 - \lambda \epsilon_k^{\mathrm{proj}}.
\end{equation}
In particular, if
\begin{equation}
\label{eq:flag-cfm-monotone-condition}
(\lambda-\alpha C_\Sigma \sigma_k^2)\beta \mathcal G_k
\;\ge\;
C\beta^2+\lambda \epsilon_k^{\mathrm{proj}},
\end{equation}
then
\begin{equation*}
    \mathcal J_{k+1}^{\mathrm{CFM}} \ge \mathcal J_k.
\end{equation*}
\end{lemma}

\begin{proof}
By Assumption~\ref{ass:flag-cfm-proj},
\begin{equation*}
    \mathcal J_{k+1}^{\mathrm{CFM}} \ge \mathcal J_{k+1}^{\mathrm{KL}} -
    \lambda \epsilon_k^{\mathrm{proj}}.
\end{equation*}
Applying Proposition~\ref{prop:flag-full-monotone-exact} to the ideal KL iterate yields
\begin{equation*}
    \mathcal J_{k+1}^{\mathrm{KL}} \ge \mathcal J_k + (\lambda-\alpha C_\Sigma \sigma_k^2)\beta \mathcal G_k - C\beta^2.
\end{equation*}
Combining the two inequalities proves \Eqref{eq:flag-full-monotone-cfm}. The last statement
follows immediately from \Eqref{eq:flag-cfm-monotone-condition}.
\end{proof}

\begin{remark*}
\textbf{Remark: Practical implications of the bound.}
The monotonic improvement guarantee in Proposition~\ref{prop:flag-full-monotone-exact} and Lemma~\ref{lem:flag-full-monotone-cfm} carries 
two practical implications that motivate FLAG's design 
(see~\Figref{fig:flag-mpo-flow}). 
First, using the effective temperature $\lambda = \alpha\lambda_{\mathrm{ref}}$ 
(Section~\ref{subsection: implementation-details}), the positive-gain condition reduces to 
$\lambda_{\mathrm{ref}} > C_\Sigma \sigma_k^2$. Covariance scheduling or clipping can help satisfy this condition when $C_\Sigma$ remains uniformly controlled over the covariance range considered. Second, the projection error 
$\epsilon_k^{\mathrm{proj}}$ is intended to be reduced by the guidance buffer, which caches recently 
improved targets and provides dense, repeated supervision for the flow policy across 
off-policy updates. The nondecrease guarantee requires the full condition in \Eqref{eq:flag-cfm-monotone-condition}. Separate covariance or temperature updates introduce additional objective changes beyond the one-step comparison analyzed here.
\end{remark*}

%% file: figures/tikz/sac.tex
\begin{tikzpicture}[
    font=\footnotesize,
    >={Stealth[length=2mm, width=1.5mm]},
    sac box/.style={
        rectangle, rounded corners=2pt,
        fill=myblue!8, draw=myblue!60, line width=0.4pt,
        minimum width=44mm, minimum height=16mm,
        inner sep=2mm, align=center
    },
    sac obj/.style={
        rectangle, rounded corners=2pt,
        fill=mygray!7, draw=mygray!50, line width=0.4pt,
        minimum width=140mm, minimum height=14mm,
        inner sep=2mm, align=center
    },
    bptt box/.style={
        rectangle, rounded corners=2pt,
        fill=myorange!10, draw=myorange!60, line width=0.4pt,
        minimum width=44mm, minimum height=16mm,
        inner sep=2mm, align=center
    },
    flag box/.style={
        rectangle, rounded corners=2pt,
        fill=myteal!10, draw=myteal!70, line width=0.4pt,
        minimum width=44mm, minimum height=16mm,
        inner sep=2mm, align=center
    },
    flag obj/.style={
        rectangle, rounded corners=2pt,
        fill=mygray!7, draw=mygray!50, line width=0.4pt,
        minimum width=140mm, minimum height=16mm,
        inner sep=2mm, align=center
    },
    cfm box/.style={
        rectangle, rounded corners=2pt,
        fill=myamber!12, draw=myamber!70, line width=0.4pt,
        minimum width=44mm, minimum height=16mm,
        inner sep=2mm, align=center
    },
    just box/.style={
        rectangle, rounded corners=2pt,
        fill=mypurple!8, draw=mypurple!60, line width=0.4pt,
        minimum width=68mm, minimum height=20mm,
        inner sep=2mm, align=center
    },
    header/.style={font=\footnotesize\bfseries, anchor=west},
    arr/.style={->, thick, color=black!70},
    sep line/.style={dashed, color=black!30, line width=0.4pt},
]

\node[header] (sac header) at (0, 0)
    {SAC: parametric route via reparameterization};

\node[sac obj, below=2mm of sac header.south west, anchor=north west]
    (sac obj) {%
    \textbf{State-wise Objective}\\[2pt]
    $\displaystyle \mathcal J_\text{MaxEnt}(\theta)= \mathbb{E}_{s\sim \rho_{\tilde \pi}}\left[\mathbb{E}_{a\sim\tilde\pi_\theta(\cdot|s)}\!\left[ Q^{\tilde\pi}(s,a) - \alpha \log\tilde\pi_\theta(a\mid s) \right]\right]$
};

\node[sac box, below=4mm of sac obj.south west, anchor=north west]
    (sac1) {%
    \textbf{Reparameterization}\\[2pt]
    $z\sim\mathcal{N}(0, I), \quad a = T_{\theta_k}(s, z)$
};

\node[sac box, right=4mm of sac1]
    (sac2) {%
    \textbf{First-order action grad}\\[2pt]
    $\nabla_a \big(Q^{\tilde\pi}(s,a) - \alpha\log\tilde\pi_{\theta_k}(a\cmid s)\big)$
};

\node[bptt box, right=4mm of sac2]
    (sac3) {%
    \textbf{Chain rule}\\[2pt]
    $\partial a / \partial \theta$,~
    \textit{requires BPTT} \\
    \textcolor{black!60}{\scriptsize(Appendix~\ref{app: flow-maxent-rl})}
};

\draw[arr] (sac1) -- (sac2);
\draw[arr] (sac2) -- (sac3);

\draw[sep line] ([yshift=-3mm]sac1.south west) -- ([yshift=-3mm]sac3.south east);

\node[header, below=4mm of sac1.south west, anchor=north west]
    (flag header) {FLAG: non-parametric route via EM algorithm};

\node[flag obj, below=2mm of flag header.south west, anchor=north west]
    (flag obj) {%
    \textbf{Objective (EM variational lower bound)}\\[2pt]
    $\displaystyle \mathcal{J}_{\text{EM}}(q,\xi) = \mathbb{E}_q\!\left[ \sum\nolimits_{t=0}^\infty \gamma^t \big( r_t - \alpha \log\tilde\pi_\theta(a_t|s_t) \big) / \lambda \right] - D_{\text{KL}}\big( p_q(\hat\tau) \,\|\, p_{\hat\pi}(\hat\tau) \big)$
};

\node[flag box, below=4mm of flag obj.south west, anchor=north west]
    (flag1) {%
    \textbf{Sample on $z$-MDP}\\[2pt]
    $(s, z) \sim \hat\rho_{\hat\pi}$\\
    $a = T_{\theta_k}(s, z) + \delta$
};

\node[flag box, right=4mm of flag1]
    (flag2) {%
    \textbf{E-step + Moment matching}\\[2pt]
    $\begin{aligned}
        &q_k \propto \hat\pi \exp(f_{\hat s,k}/\lambda) && \!\!\textnormal{\textcolor{black!60}{\scriptsize(\Eqref{eq: e-step-target})}} \\
        &\mu_k^*(\hat s) = \mathbb{E}_{q_k}[a] && \!\!\textnormal{\textcolor{black!60}{\scriptsize(\Eqref{eq:SNIS-moment-matching})}}
    \end{aligned}$
};

\node[cfm box, right=4mm of flag2]
    (flag3) {%
    \textbf{CFM distillation (M-step)}\\[2pt]
    $\tilde\pi_\theta$ trained to produce $\mu_k^*(\hat s)$ \\ 
    \textit{BPTT-free}~~\textcolor{black!60}{\scriptsize(\Eqref{eq: M-step-CFM})}
};

\draw[arr] (flag1) -- (flag2);
\draw[arr] (flag2) -- (flag3);
\draw[sep line] ([yshift=-3mm]flag1.south west) -- ([yshift=-3mm]flag3.south east);

\node[header, below=4mm of flag1.south west, anchor=north west] 
    (just header) {How FLAG inherits SAC's soft policy improvement};
    
\node[just box, below=2mm of just header.south west, anchor=north west] 
    (just1) {%
    \textbf{Energy alignment}\\[2pt]
    $Q^\pi \approx Q^{\tilde\pi}$ up to $\mathcal{O}(\sqrt{\mathrm{tr}(\Sigma)})$\\
    FLAG energy aligns with SAC $\tilde\pi$-based objective \\
    \textcolor{black!60}{\scriptsize(Appendix~\ref{app:Q-ftn-Sigma-bound})}
};

\node[just box, right=4mm of just1]
    (just2) {%
    \textbf{E-step + M-step $\approx$ Zeroth-Order grad}\\[2pt]
    $\mu_k^*(\hat s) - \mu_k = \Sigma_k \nabla_\mu g_k$ approximates\\
    the Zeroth-Order grad of 
    $\nabla_a (Q^{\tilde\pi} - \alpha\log\tilde\pi)$\\
    \textcolor{black!60}{\scriptsize(Appendix~\ref{app:SNIS-and-zero-order})}
};
\end{tikzpicture}

%% file: figures/tikz/mpo.tex
            

\begin{tikzpicture}[
    font=\footnotesize,
    >={Stealth[length=2mm, width=1.5mm]},
    mpo badge/.style={
        circle,
        draw=black!70,
        fill=white,
        line width=0.35pt,
        inner sep=0.3pt,
        minimum size=1mm,
        font=\scriptsize,
        align=center
    },
    mpo badge2/.style={
        circle,
        draw=black!70,
        fill=white,
        line width=0.35pt, %
        inner sep=0.6pt,   %
        minimum size=2mm,  %
        font=\normalsize,  %
        align=center
    },
    flag obj/.style={
        rectangle, rounded corners=2pt,
        fill=mygray!7, draw=mygray!50, line width=0.4pt,
        minimum width=140mm, minimum height=16mm,
        inner sep=2mm, align=center
    },
    flag obj2/.style={
        rectangle, rounded corners=2pt,
        fill=myteal!10, draw=myteal!70, line width=0.4pt,
        minimum width=140mm, minimum height=16mm,
        inner sep=2mm, align=center
    },
    flag half/.style={
        rectangle, rounded corners=2pt,
        fill=teal!10, draw=teal!70, line width=0.4pt,
        minimum width=68mm, minimum height=24mm,
        inner sep=2mm, align=center
    },
    drift half1/.style={
        rectangle, rounded corners=2pt,
        fill=myorange!10, draw=myorange!60, line width=0.4pt,
        minimum width=68mm, minimum height=18mm,
        inner sep=2mm, align=center
    },
    drift half2/.style={
        rectangle, rounded corners=2pt,
        fill=myblue!8, draw=myblue!60, line width=0.4pt,
        minimum width=68mm, minimum height=18mm,
        inner sep=2mm, align=center
    },
    cfm obj/.style={
        rectangle, rounded corners=2pt,
        fill=myamber!12, draw=myamber!70, line width=0.4pt,
        minimum width=140mm, minimum height=15mm,
        inner sep=2mm, align=center
    },
    practical box/.style={
        rectangle, rounded corners=2pt,
        fill=mypurple!8, draw=mypurple!60, line width=0.4pt,
        minimum width=68mm, minimum height=22mm,
        inner sep=2mm, align=center
    },
    header/.style={font=\footnotesize\bfseries, anchor=west},
    arr/.style={->, thick, color=black!70},
    sep line/.style={dashed, color=black!30, line width=0.4pt},
]

\node[header] (header) at (0,0)
    {FLAG: MPO-style monotonic improvement};

\node[flag obj, below=2mm of header.south west, anchor=north west]
    (decomp) {%
    \textbf{Objective decomposition}\\[2pt]
    $\displaystyle
    \mathcal J_{k+1}^{\mathrm{KL}}-\mathcal J_k
    =
    \mathbb E_{\hat s_0\sim\hat p_0}\Bigg[
    \underbrace{
        \Big(
        {\color{myorange}
        \hat V_{k,\lambda}^{\hat\pi_{k+1}^{\mathrm{KL}},q_{k+1}}(\hat s_0)
        -
        \hat V_{k,\lambda}^{\hat\pi_k,q_k}(\hat s_0)
        }
        \Big)
    }_{%
        \tikz[baseline=-1.2ex]\node[mpo badge] {1};
        \;\textcolor{black!60}{%
            \scriptsize\text{in }
            \textnormal{\Eqref{eq: mpo-obj-decomp}}%
        }%
    }
    +
    \underbrace{
        \Big(
        {\color{myblue}
        \hat V_{k+1,\lambda}^{\hat\pi_{k+1}^{\mathrm{KL}},q_{k+1}}(\hat s_0)
        -
        \hat V_{k,\lambda}^{\hat\pi_{k+1}^{\mathrm{KL}},q_{k+1}}(\hat s_0)
        }
        \Big)
    }_{%
        \tikz[baseline=-1.2ex]\node[mpo badge] {2};
        \;\textcolor{black!60}{%
            \scriptsize\text{in }
            \textnormal{\Eqref{eq: mpo-obj-decomp}}%
        }%
    }
    \Bigg]
    $
};

\node[drift half1, below=5mm of decomp.south west, anchor=north west]
    (term1) {%
    \textbf{MPO-style improvement}\\[1ex]
    Reward index fixed at \(k\)\\[1.5ex]
    \tikz[baseline=-1.7ex]\node[mpo badge2] {1};
    $\displaystyle
    \;\ge\;
    \lambda\beta\mathcal G_k
    -
    C_1\beta^2
    $
};

\node[drift half2, right=4mm of term1]
    (term2) {%
    \textbf{Reward drift}\\[1ex]
    Reward index moves \(k\to k+1\)\\[1.5ex]
    $\displaystyle
    \Big|\;\tikz[baseline=-0.8ex]\node[mpo badge2] {2}; \;\Big| \;\le\;
    \alpha C_\Sigma\Lambda_k\beta\mathcal G_k
    +
    C_d\beta^2
    $
};

\draw[arr] (term1.north |- decomp.south) -- (term1.north);
\draw[arr] (term2.north |- decomp.south) -- (term2.north);

\node[flag obj2, below=5mm of term1.south west, anchor=north west]
    (prop) {%
    \textbf{Combine 
    \tikz[baseline=-1.4ex]\node[mpo badge] {1}; and
    \tikz[baseline=-1.4ex]\node[mpo badge] {2};
    \textcolor{black!60}{\scriptsize(Proposition~\ref{prop:flag-full-monotone-exact})}} \\[1ex]
    $\displaystyle
    \mathcal J_{k+1}^{\mathrm{KL}}
    \ge
    \mathcal J_k
    +
    (\lambda-\alpha C_\Sigma\Lambda_k)\beta\mathcal G_k
    -
    C\beta^2,
    \qquad
    C:=C_1+C_d.
    $\\[2pt]
    If \(\lambda>\alpha C_\Sigma\Lambda_k\), then for sufficiently small \(\beta\),
    \(\mathcal J_{k+1}^{\mathrm{KL}}\ge \mathcal J_k\).
};

\draw[arr] (term1.south) -- (term1.south |- prop.north);
\draw[arr] (term2.south) -- (term2.south |- prop.north);

\node[cfm obj, below=5mm of prop.south west, anchor=north west]
    (cfm) {%
    \textbf{Add practical CFM projection error 
    \textcolor{black!60}{\scriptsize(Lemma~\ref{lem:flag-full-monotone-cfm})}}\\[2pt]
    $\displaystyle
    \mathcal J_{k+1}^{\mathrm{CFM}}
    \ge
    \mathcal J_k
    +
    (\lambda-\alpha C_\Sigma\Lambda_k)\beta\mathcal G_k
    -
    C\beta^2
    -
    \lambda\epsilon_k^{\mathrm{proj}}.
    $
};

\draw[arr] (prop.south) -- (cfm.north);
\draw[sep line] ([yshift=-3mm]cfm.south west) -- ([yshift=-3mm]cfm.south east);

\node[header, below=4mm of cfm.south west, anchor=north west]
    (impl header) {FLAG: Practical implementation
    \textcolor{black!60}{\scriptsize(Section~\ref{subsection: implementation-details})}};

\node[practical box, below=5mm of impl header.south west, anchor=north west]
    (impl1) {%
    \textbf{Hyperparameter scaling}\\[2pt]
    Set $\lambda = \alpha \lambda_{\mathrm{ref}}$, then the sufficient condition reduces to\\[2pt]
    $\lambda_{\mathrm{ref}} > C_\Sigma \Lambda_k$
};
\node[practical box, right=4mm of impl1]
    (impl2) {%
    \textbf{Guidance buffer}\\[2pt]
    Caches recently improved targets and provides\\
    repeated supervision, keeping $\epsilon_k^{\mathrm{proj}}$ small
};
\coordinate (split) at ($(cfm.south) + (0, -9mm)$);
\draw[arr] (cfm.south) -- (split);
\draw[arr] (split) -| (impl1.north);
\draw[arr] (split) -| (impl2.north);
\end{tikzpicture}

%% file: sections/appendix/details.tex
\section{Implementation Details} \label{app:implementation_details}

\subsection{Action scaling and SNIS details}
\label{app:implementation_details_action_scaling}
We map pre--action samples $u$ to bounded actions $a \in [-1,1]^{|\mathcal{A}|}$ via an element--wise $\tanh$, where $|\mathcal{A}|$ is the cardinality of the action space.
Let us denote the pre--action sampling distribution at iteration $k$ as $r_k(u \mid \hat{s})$, which is a Gaussian centered on the flow policy output $\mu(\hat{s};\theta_k)$ and has a scheduled covariance $\Sigma_k$.
To apply change of variables to $r_k$, we need the determinant of the Jacobian (of the element--wise $\tanh$), $|\det{J_{\tanh}(u)}| = \prod_{i=1}^{|\mathcal{A}|}(1 - a_i^2)$ where $a = \tanh(u)$.
Applying change of variables leads to the following:
\begin{equation}
\label{eq: app-pushforward}
\hat \pi_k(a\mid \hat s) = \frac{r_k(u\mid \hat s)}{|\det J_{\tanh}(u)|},
\qquad
\log \hat \pi_k(a\mid \hat s)\;=\;\log r_k(u\mid \hat s)-\sum_{i=1}^D \log\!\big(1-a_i^2\big).
\end{equation}
We form the E-step target distribution following \Eqref{eq: e-step-target}, only difference being the parameterization where we do not parameterize the covariance $\Sigma_k$, thus $\xi_k=\theta_k$.
Pulling back to pre--action space via \Eqref{eq: app-pushforward} yields
\begin{equation}
\label{eq:pullback}
    q_k(u\mid \hat s) =
    q_k(\tanh^{-1}(a) \mid \hat{s}) \, |\det J_{\tanh}(u)|
    \propto r_k(u \mid \hat s)\, \exp \left( \frac{f_{\hat s, k}(\tanh(u))}{\alpha_k\lambda_\text{ref}} \right)
\end{equation}
Using $r_k$ as the proposal distribution, we define unnormalized importance weight $w$ and subsequently self-normalized importance weight $\bar{w}$ with $N$ pre--action samples.
The moment matching in the pre--action space ($u$--space) is estimated by self-normalized importance sampling:
\begin{equation}
    w(u) :=
    \exp \left( \frac{f_{\hat s, k}(\tanh(u))}{\alpha_k\lambda_\text{ref}}\right), \quad
    \bar w_i=\frac{w(u_i)}{\sum_{j=1}^N w(u_j)}, \quad
    \mu^*(\hat s)=\sum_{i=1}^{N}\bar w_i u_i .
\end{equation}

\subsection{Critic networks}

We adapt the distributional Q-function from \citep{distributionalrl} to account for the entropy of the policy.
The TD-target $y$ in \Eqref{eq: cross-entropy-soft-bellman-op} is computed by
\begin{equation}
    y = r(s, a) +
    \gamma \mathbb{E}_{a'\sim \pi(\cdot\mid s')}\Bigg[\sum_{i=0}^{b-1} \left( Q_{\min} + \left(\frac{Q_{\max}-Q_{\min}}{b - 1}\right) i \right) \cdot p_i -\alpha \log \tilde \pi_\theta(a'\mid s') \Bigg],
\end{equation}
where $b$ is the number of bins and $p_i$ denotes the probability of the $i$-th bins which is the output of the critic network. 
After calculating the target we project it to discrete bins, $\hat{y} = \mathtt{two\_hot}(y)$.
The critic minimizes the following loss, which was proposed in \citep{dime}.
\begin{equation}
\label{eq: critic-loss}
    \mathcal L_Q(\phi)=-\sum_{i=1}^{b} \hat y\log p_i - 0.005\sum_{i=1}^{b} p_i\log p_i,
\end{equation}
Following \citep{nauman2024theory}, we use the \emph{mean} of the two TD-targets from two Q-function instead of \emph{min}.

\subsection{Variance Reduction in Hutchinson Trace Estimation} \label{subsection:hutchinson-variance-reduction}

\paragraph{Rademacher Distribution.}

We sample $\epsilon$ from the Rademacher distribution, which typically yields a lower variance estimator \citep{hutchinsontraceestimator,avron2011randomized}.

\paragraph{Common Random Numbers (CRN).}

When performing the estimation in Eq.~\ref{eq:hutch_logpi}, we evaluate the difference in log-probabilities between an action $a$ and its perturbed neighbors $a + \delta$ to reduce the variance of the estimate compared to the case of estimating them individually.
Specifically, the variance of the paired difference estimate can be minimized by employing the \textit{Common Random Numbers} (CRN) technique \citep{carlo2001monte}.
This technique uses the exact same noise vector $\epsilon$ for both $a$ and $a + \delta$, inducing a strong coupling between the two stochastic estimates.

Let $\hat{L}(a; \epsilon) = \epsilon^\top \nabla u(a) \epsilon$ denote the stochastic estimator for the trace term at $a$ using a noise vector $\epsilon$, where $u$ is a vector field. 
The variance of the paired difference estimator is given by:
\begin{equation}
    \mathrm{Var}(\hat{L}(a + \delta; \epsilon) - \hat{L}(a; \epsilon)) =
    \mathrm{Var}(\hat{L}(a + \delta; \epsilon)) + \mathrm{Var}(\hat{L}(a; \epsilon)) - 2 \, \mathrm{Cov}(\hat{L}(a+\delta; \epsilon), \hat{L}(a; \epsilon)).
\end{equation}
Assuming that the Jacobian $\nabla u(a)$ varies smoothly when $\delta$ is small, $\hat{L}(a + \delta; \epsilon)$ and $\hat{L}(a; \epsilon)$ are positively correlated under the shared $\epsilon$.
Consequently, the large covariance term cancels a substantial portion of the individual variances, yielding a lower-variance estimate of the difference.

\paragraph{Application to Weight Calculation.}

To leverage the variance-reducing property of CRN, we formulate the importance sampling weights in terms of \emph{differences}.
We hereafter omit the iteration index $k$ for brevity.
Ideally, the updated action $\mu^*$ is computed as the weighted average of candidate actions using weights $w(a) = \exp( f_{\hat s}(a))$. 
We observe that the SNIS ratio is invariant to a constant baseline shift.
Dividing $w(a)$ by $\exp(f_{\hat s}(\mu))$, we define the baseline subtracted weight $\tilde{w}(a) := \exp(f_{\hat{s}}(a) - f_{\hat{s}}(\mu))$:
\begin{equation}
    \mu^* =
    \sum_{i=1}^{N} \frac{\exp \big( f_{\hat s}(a_i)\big)\cdot a_i}{\sum_{j=1}^{N}\exp \big( f_{\hat s}(a_j)\big) } =
    \sum_{i=1}^{N} \frac{\tilde w(a_i)\cdot a_i}{\sum_{j=1}^{N} \tilde w(a_j)}.
\end{equation}
$\exp(f_{\hat{s}}(a) - f_{\hat{s}}(\mu))$ contains a log-probability difference term $\log \tilde{\pi}(a \mid s) - \log \tilde{\pi}(\mu \mid s)$, allowing us to directly apply the CRN technique described above.

\subsection{Learning the Covariance Network} \label{subsection:covariance_learning}

Covariance $\Sigma_\psi$ in \Eqref{eq:local-policy} is learnable via second moment matching, leveraging the target mean estimates from \ref{subsection:hutchinson-variance-reduction}.
We assume diagonal covariance, where a neural network predicts $\sigma = f_\psi(s) \in \mathbb{R}^{|\mathcal{A}|}$ and equivalently $\Sigma = \sigma^2 I \in \mathbb{R}^{|\mathcal{A}| \times |\mathcal{A}|}$.
The target variance of the $d$-th action as the weighted second moment is given by
\begin{equation}
    (\sigma_d^*)^2 \approx \sum_{i=1}^{N} \frac{\bar w(a_i)\, (a_{i,d} - \mu_d^*)^2}{\sum_{j=1}^{N}\bar w(a_j)}.
\end{equation}

\paragraph{Preventing Deterministic Collapse via Entropy Regularization.}

When the temperature $\alpha$ becomes small, the normalized weights $\bar{w}(a_i)$ can collapse toward a one-hot distribution, causing $\sigma_i \to 0$.
This drives the covariance network toward a nearly deterministic policy and eliminates local exploration.
To counteract this effect, we add an explicit entropy-preserving force based on the entropy of the local Gaussian policy.
For $\Sigma = \mathrm{diag}(\sigma_1^2,\dots,\sigma_{|\mathcal{A}|}^2)$, the entropy of the local policy is $\mathcal{H}_\text{local} = \sum_{i=1}^{|\mathcal{A}|} \log \sigma_i + \text{const.}$, whose derivative w.r.t. $\sigma_i$ is $1 / \sigma_i$.
Combining this entropy regularization term with the term induced by moment matching yields the following update:
\begin{equation}
    \nabla_{\sigma_i} \mathcal{J}_{\text{total}} \approx
    \underbrace{\frac{(\sigma_i^*)^2 - \sigma_i^2}{\sigma_i^3}}_{\text{second moment matching}} +
    \underbrace{\beta \frac{1}{\sigma_i}}_{\text{entropy regularization}}.
\end{equation}
The equilibrium yields $\sigma_i^2 = \frac{(\sigma_i^*)^2}{1-\beta}$ where larger $\beta \in [0, 1)$ enlarges the target variance relative to the pure moment matching and helps prevent variance collapse.

\paragraph{Log-Space Target Matching.}

In practice, the covariance network predicts the log-standard deviation $\varsigma_i = \log \sigma_i$. Applying the chain rule gives
\begin{equation}
    \frac{\partial \mathcal{J}_{\text{total}}}{\partial \varsigma_i}
    =
    \frac{\partial \mathcal{J}_{\text{total}}}{\partial \sigma_i}
    \frac{\partial \sigma_i}{\partial \varsigma_i}
    =
    \frac{(\sigma_i^*)^2}{\sigma_i^2} - 1 + \beta.
    \label{eq:logstd_gradient}
\end{equation}
\Eqref{eq:logstd_gradient} identifies the desired equilibrium, but directly optimizing this asymmetric gradient can be numerically unstable in bounded parameterizations of $\varsigma_i$. Instead, we construct a surrogate loss whose minimizer matches the same equilibrium in log-space.
From the equilibrium $\sigma_i^2 = \frac{(\sigma_i^*)^2}{1-\beta}$, the corresponding target log standard deviation is $\varsigma_{\text{target},i} = \frac{1}{2}\log\big((\sigma_i^*)^2\big) - \frac{1}{2}\log(1-\beta)$.
To avoid the singularity at $(\sigma_i^*)^2 = 0$, we introduce a minimum variance floor $\sigma_{\min}^2$ and define the stabilized exploitation target covariance $\hat{\Sigma}_i = \max ((\sigma_i^*)^2,\sigma_{\min}^2)$.
We then compute the bounded target
\begin{equation}
    \varsigma_{\text{target},i}
    =
    \mathrm{clip}\!\left(
        \frac{1}{2}\log \hat{\Sigma}_i
        -
        \frac{1}{2}\log(1-\beta),
        \varsigma_{\min},
        \varsigma_{\max}
    \right).
\end{equation}

Finally, we update the covariance network using the symmetric log-space MSE objective
\begin{equation}
    \mathcal{L}(\psi)
    =
    \mathbb{E}_{s \sim \mathcal{D}}
    \left[
        \frac{1}{D}
        \sum_{i=1}^{D}
        \Big(
            \varsigma_i(s) - \mathtt{sg}(\varsigma_{\text{target},i}(s))
        \Big)^2
    \right],
\end{equation}
where $\mathtt{sg}(\cdot)$ is a stop-gradient operator. 

\begin{algorithm}[tb]
\caption{FLAG: Flow Policy MaxEnt-RL by Latent Augmented Guidance}
\label{algorithm::FLAG}
\begin{algorithmic}[1]
\STATE Initialize critic networks $Q_{\phi_1}$, $Q_{\phi_2}$, and vector field network $u_\theta$ with random parameters $\phi_1, \phi_2, \theta$.
\STATE Initialize environment replay buffer $\mathcal{D}_{\text{ENV}}$ and small guidance buffer $\mathcal{D}_\text{FLAG}$
\FOR{$k=1$ to $M$}
    \IF{$k\,\%\, \text{UTD}=0$}
    \STATE Sample $a\sim \pi(a \mid s;\theta_k)$
    \STATE Step environment: $s'\sim p(s' \mid  {s},  {a})$
    \STATE Store $( {s},  {a}, r,  {s}')$ in $\mathcal{D}_\text{ENV}$
    \ENDIF 
    \STATE Sample $B$ transitions $(s, a, s', r)\sim\mathcal{D}$
    \FOR{each update step}
    \STATE Sample $B$ transitions $( {s},  {a}, r,  {s}')$ from $\mathcal{D}$
    \STATE Sample $ {a}' \sim \pi(a'\mid  {s}';\theta_k)$
    \STATE Compute the log probability of $\tilde \pi$ (\Eqref{eq:hutch_logpi})
    \STATE Update critics: $\phi\leftarrow \phi - \eta_\phi\nabla_\phi \mathcal L(\phi)$ (\Eqref{eq: critic-loss})
    
    \IF{$K\,\%\,\text{POLICY DELAY}$}
    \STATE Sample $N$ latent vectors $z_1, \dots z_N\sim p_z(z)$ and transport to actions $\{T_{\theta_k}(s, z_i)\}_{i=1}^{N}$
    \STATE Compute the log probability of $\tilde \pi$ (Eq.~\ref{eq:hutch_logpi})
    \STATE Estimate $f_{\hat s, k}(a_i)$ (\Eqref{eq: energy-func}) and normalize with $\alpha$ 
    \STATE Estimate $\mu^*$ (\Eqref{eq:M-step-final}) and Update base flow policy: $\theta_{k+1}\leftarrow\theta_k-\eta_k\nabla_\theta\mathcal L(\theta)$ (\Eqref{eq: M-step-CFM})
    \STATE Store $(s, z, \mu^*)$ in $\mathcal{D}_\text{FLAG}$
    \STATE Update $\alpha_{k+1}\leftarrow\alpha_k-\eta_\alpha\mathcal J(\alpha)$ (\Eqref{eq: alpha-objective})
    \ENDIF
    \STATE Sample $B$ augmented-state--action pairs $(s, z, \mu)\sim \mathcal{D}_\text{FLAG}$
    \STATE Update base flow policy: $\theta_{k+1}\leftarrow\theta_k-\eta_\theta\nabla_\theta\mathcal L(\theta)$ (\Eqref{eq: M-step-CFM})
    \ENDFOR
    
\ENDFOR
\vskip -0.2in
\end{algorithmic}
\end{algorithm}
\clearpage

%% file: sections/appendix/baselines.tex
\section{Baseline}
\subsection{Baselines in Section~\ref{subsection: global policy sampling}}
\label{app: Baseline_5.1}

\textbf{MaxEntDP\footnote{\url{https://github.com/diffusionyes/MaxEntDP}}}
utilizes the Q-weighted Noise Estimation (QNE) method to approximate the exponential of the target distribution of MaxEnt-RL. In the original paper, They employ $N=500$ action samples to ensure a low-variance and accurate training target for the noise prediction network, which is expensive. The policy naturally balances local and global update characteristics because the diffusion time step $t$ controls the noise schedule $\alpha_t$, which directly determines the scale of the sampled action area $a_t$ during the training process. Our implementation follows their official Github repository.

\textbf{DPMD\footnote{\url{https://github.com/mahaitongdae/diffusion_policy_online_rl}}} utilizes \textbf{Reweighted Score Matching (RSM)} to optimize diffusion policies without requiring direct samples from the optimal policy. The algorithm weights the score matching objective with the exponential of the $Q$-function, $\exp(Q(s, a)/\lambda)$. To ensure efficient exploitation and manage the inherent randomness of the diffusion process, DPMD adopts batch action sampling---a form of soft rejection sampling---where the action $a = \arg\max_i Q(s, a_{(i)})$ is selected from numerous generated candidates ($P=32$) for environment interaction, whereas $N=1$ action is sampled in the policy update. The inference time is more expensive than FLAG, because we use "policy delay" for computational efficiency in policy update, while the batch action sampling cannot. To match the sampling budget with FLAG, we modify the original DPMD so that the algorithm does not leverage batch action sampling heuristics ($P=1$) and use multiple action samples ($N\le64$) during policy update. Our implementation follows their official Github repository.

\textbf{QVPO\footnote{\url{https://github.com/wadx2019/qvpo}}} derives its policy update by leveraging a lower bound on the policy gradient objective. Specifically, rather than directly maximizing $\mathbb{E}_{a \sim \pi_\theta}[Q(s, a)]$ via gradient ascent, QVPO shows that this objective admits a tractable lower bound expressible as a Q-weighted score matching loss, enabling the use of a diffusion model as the policy. Concretely, at each update step, $N=64$ candidate actions are drawn from the current diffusion policy for each state, and the top-$k$ ($k=1$) actions by the minimum of two Q-networks are retained as positive targets. Only candidates whose Q-value exceeds a threshold ($Q > 1.0$) receive a nonzero weight, acting as a quality filter. The diffusion policy is then trained via the weighted denoising score matching objective. For exploration, QVPO augments the training batch with $N_{\text{neg}}=10$ actions per positive sample drawn uniformly at random from $\mathcal{U}(-1, 1)$; these negative samples are assigned weights $w_i = \alpha \cdot Q_{\text{pos}}$ where $\alpha$ follows a linear decay schedule from $0.02$ to $0.002$ over $2\times10^5$ steps, encouraging broad exploration early in training. Our implementation follows their official Github repository.

For CrossQ update and distributional critic architecture we follows DIME official Github.
\subsection{Baselines in Section~\ref{subsection: comparison_with_diff-flow-policy}}

\paragraph{DIME~\citep{dime}.}
For MuJoCo tasks, we run DIME\footnote{\url{https://github.com/ALRhub/DIME}} using the implementation released in the official
GitHub repository. 
For DMC Dog and MyoSuite tasks, we directly use the results reported in the
original paper, together with the reported CrossQ and QSM baselines.

\paragraph{QSM~\citep{qsm}.}
For MuJoCo tasks, we run QSM\footnote{\url{https://github.com/escontra/score_matching_rl}} using the official implementation.
For a fair comparison, we align its hyperparameters with those used in DIME
whenever applicable.

\paragraph{DACERv2~\citep{dacerv2} and DIPO~\citep{dipo}.}
Our implementations of DACERv2\footnote{\url{https://github.com/happy-yan/DACER-Diffusion-with-Online-RL}} and DIPO\footnote{\url{https://github.com/BellmanTimeHut/DIPO}} are based on their official GitHub
repositories

Following DIME, we additionally incorporate the CrossQ update for the critic and
use a distributional critic architecture.

\paragraph{FlowRL~\citep{flowrl}.}
Our implementation of FlowRL\footnote{\url{https://github.com/bytedance/FlowRL}} follows the official GitHub repository.
FlowRL uses layer normalization~\citep{ba2016layer} in the critic architecture
and employs a three-layer MLP with hidden dimensions $[512,512,512]$, which is
architecturally different from the wider CrossQ critic with hidden dimensions
$[2048,2048]$. Therefore, we do not apply the CrossQ modification to FlowRL and
instead retain the original critic architecture.

All hyperparameters are listed in Table~\ref{tab:diffusion_flow_hyperparameters}.

%% file: sections/appendix/additional_ablations.tex
\section{Additional Ablation Studies} \label{app: additional_ablation_studies}

\begin{figure}[t]
    \centering
    \begin{subfigure}[b]{0.48\linewidth}
        \centering
        \includegraphics[width=\linewidth]{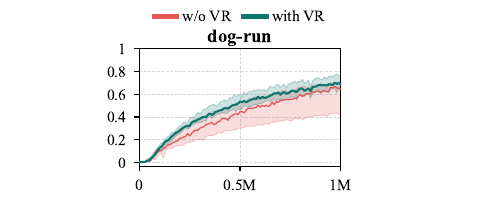}
        \caption{Variance Reduction on Trace Estimation}
        \label{fig: hutchinson-vr}
    \end{subfigure}
    \hfill
    \begin{subfigure}[b]{0.48\linewidth}
        \centering
        \includegraphics[width=\linewidth]{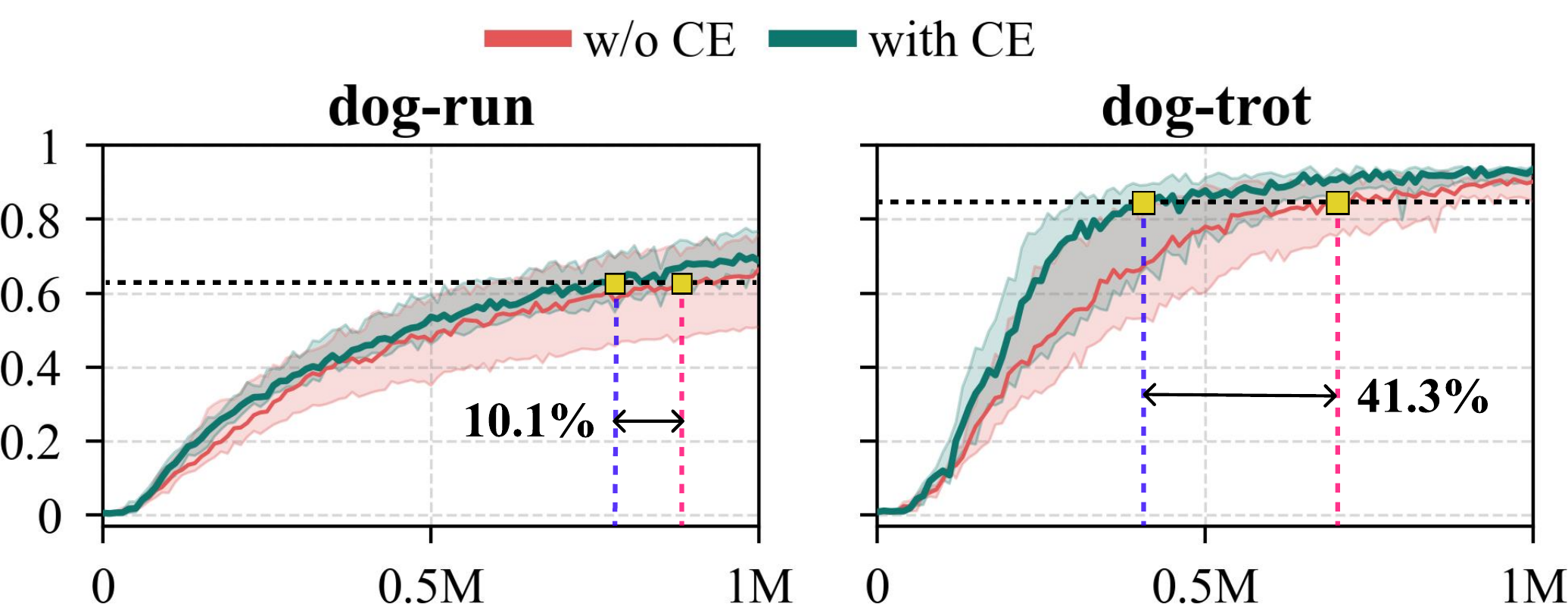}
        \caption{Effect of Cross-Entropy}
        \label{fig: entropy}
    \end{subfigure}
    \caption{
    \textbf{Variance Reduction and Cross-Entropy Ablations.}
    \emph{(a)} Training stability with and without the variance reduction technique for Hutchinson's trace estimator.
    \emph{(b)} We compare FLAG against a variant trained without the cross-entropy term. The horizontal dashed line represents $90\%$ of peak performance, with intersection points marked by \protect\raisebox{-0.15em}{\includegraphics[height=0.8em]{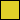}}.
    }
    \label{fig:vr_entropy_row}
\end{figure}

\begin{figure}[t]
    \centering
    \begin{subfigure}[b]{0.39\linewidth}
        \centering
        \includegraphics[width=\linewidth]{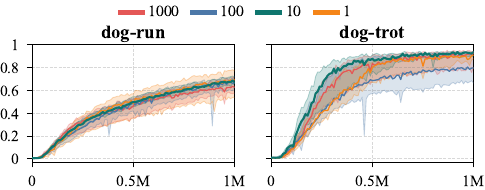}
        \caption{Sensitivity to $\lambda_\text{ref}$}
        \label{fig:lambda-ref}
    \end{subfigure}
    \hfill
    \begin{subfigure}[b]{0.39\linewidth}
        \centering
        \includegraphics[width=\linewidth]{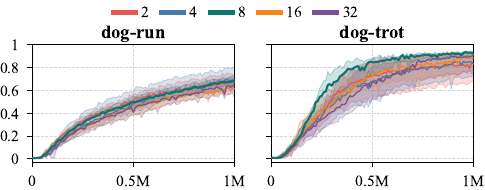}
        \caption{Effect of ODE solver steps}
        \label{fig:num-solver-steps}
    \end{subfigure}
    \hfill
    \begin{subfigure}[b]{0.19\linewidth}
        \centering
        \includegraphics[width=\linewidth]{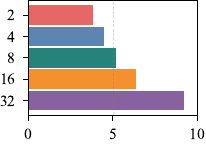}
        \caption{Mean runtime (H)}
        \label{fig:runtime}
    \end{subfigure}

    \caption{
    \textbf{Ablation and Efficiency Analysis.}
    \emph{(a)} Performance curves across values of the inverse reference temperature $\eta = \lambda_\text{ref}^{-1}$ (shown in the legend).
    \emph{(b)} Aggregate performance as a function of the number of integration steps used during flow inference.
    \emph{(c)} Runtime comparison.
    }
    \label{fig:ablation_row}
\end{figure}

\subsection{Variance Reduction in Hutchinson's Trace Estimator}
\label{app:hutchinson}
To efficiently estimate the log-likelihood of the flow policy, we employ Hutchinson's trace estimator.
To mitigate the stochasticity of this approximation, we adopt a variance reduction technique that shares the random tangent vectors across the action samples, as detailed in Appendix \ref{subsection:hutchinson-variance-reduction}.
We ablate this design choice to quantify its impact on training stability (Figure \ref{fig: hutchinson-vr}).
The results demonstrate that this variance reduction strategy not only accelerates learning but also significantly reduces the variance across random seeds, leading to more robust convergence.

\subsection{Cross-Entropy}
\label{app:cross_entropy}
Although MaxEnt RL promotes systematic exploration, exact entropy computation for expressive policies is infeasible.
We circumvent this by introducing a cross-entropy term as a tractable surrogate.
To verify this design, we conduct an ablation study comparing our approach with a standard formulation that lacks the MaxEnt component.
We omit the log-probability from the
soft Q-function update (Eq. \ref{eq: cross-entropy-soft-v}) and the weight calculation (Eq. \ref{eq: e-step-target}).
The model with cross-entropy reaches the $90\%$ of peak performance approximately $10.1\%$ and $41.3\%$ faster in respective tasks compared to the baseline in \Figref{fig: entropy}. This shows that the inclusion of cross-entropy regularizes the policy by preventing overfitting to early value overestimation, as evidenced by the reduced variance across random seeds.
We measure how closely the cross-entropy proxy tracks the true entropy in \Cref{app:kl_proxy_accuracy}.

\subsection{Sensitivity to $\lambda_\text{ref}$}
\label{app:lambda_ref}
We examine the sensitivity of FLAG to the reference temperature $\lambda_\text{ref}$, reporting performance on the DMC dog-run and dog-trot tasks in Figure \ref{fig:lambda-ref}.
Since the E-step weights scale as $\exp(f / (\alpha \lambda_\text{ref}))$, a larger $\lambda_\text{ref}$ flattens the importance weights and yields a more conservative update toward the current policy, while a smaller $\lambda_\text{ref}$ concentrates the weights on the highest-value samples and yields a more aggressive update.
Our implementation parameterizes the weights by the inverse temperature $\eta = \lambda_\text{ref}^{-1}$, which is the value shown in the legend of Figure \ref{fig:lambda-ref}.
As illustrated, performance is sensitive to this choice.
Overly conservative updates (small $\eta$) slow down learning, whereas overly aggressive updates (large $\eta$) destabilize training and degrade performance.
For our main experiments, we found $\lambda_\text{ref} = 10$ (i.e., $\eta = 0.1$) to yield the most consistent results across tasks.

\subsection{Influence of ODE Solver Steps}
\label{app:ode_steps}
We analyze the trade-off between inference precision and computational cost by varying the number of ODE solver steps used for the flow policy.
The results are summarized in Figure \ref{fig:num-solver-steps}. This result shows that the performance is relatively robust to the number of ODE solver steps. Since the training duration increases proportionally with the number of steps (\Figref{fig:runtime}), we fixed the solver steps to $8$ for all experiments to achieve computational efficiency.

\subsection{Covariance Learning and Connections to Zeroth-Order Gradient Methods}
\label{app:variants}
\begin{table}[t]
    \centering
    \caption{\textbf{Learned covariance variant.} The learned covariance variant does not consistently outperform FLAG and shows sensitivity to the hyperparameters $\beta$ and $\log\sigma_{\text{final}}$, making it less practical despite its closer alignment with the theoretical results.}
    \label{tab:learned-covariance}
    \footnotesize
    \begin{tabular}{cccccc}
    \toprule
    \multirow{2}{*}{} & \multirow{2}{*}{$\bm{\log\sigma_{\text{final}}}$} & \multicolumn{4}{c}{$\bm{\beta}$} \\ \cmidrule(lr){3-6}
    & & \textbf{0.0} & \textbf{0.1} & \textbf{0.5} & \textbf{0.9} \\
    \midrule
    \multirow{4}{*}{\textbf{Dog-Run}}
    & \multirow{2}{*}{$1\text{e-}3$} & 684 & 646 & 672 & 651 \\
    & & [645, 738] & [613, 693] & [635, 735] & [571, 724] \\ \cmidrule{2-6}
    & \multirow{2}{*}{$5\text{e-}4$} & 484 & 597 & 645 & 618 \\
    & & [344, 672] & [562, 704] & [562, 698] & [550, 716] \\
    \midrule
    \multirow{4}{*}{\textbf{Dog-Trot}}
    & \multirow{2}{*}{$1\text{e-}3$} & 916 & 910 & 931 & 914 \\
    & & [885, 939] & [664, 929] & [909, 945] & [765, 933] \\ \cmidrule{2-6}
    & \multirow{2}{*}{$5\text{e-}4$} & 891 & 876 & 923 & 899 \\
    & & [767, 922] & [746, 930] & [901, 947] & [753, 930] \\
    \bottomrule
    \end{tabular}
\end{table}

\begin{figure}[t]
    \centering
    \includegraphics[width=0.4\linewidth]{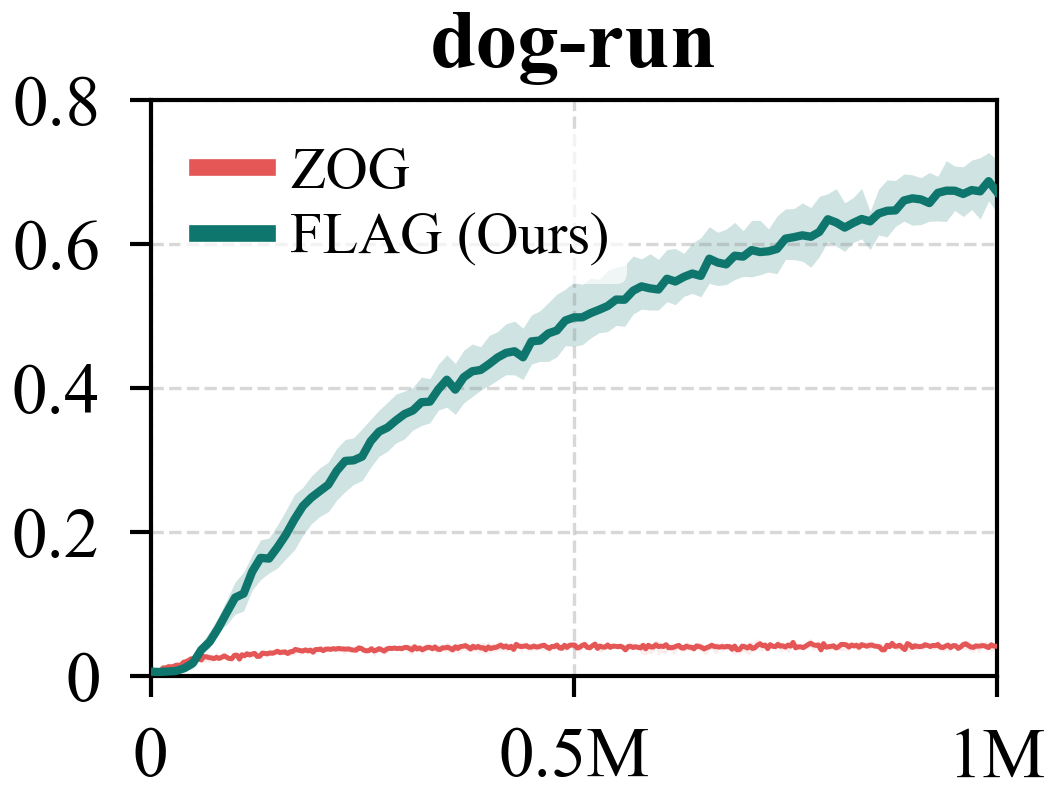}
    \caption{\textbf{Zeroth-order gradient variant.} The zeroth-order gradient variant of FLAG performs worse than FLAG, showing that our SNIS-based moment matching outperforms and is more stable than the zeroth-order gradient variant.}
    \label{fig:zeroth-order-gradient}
\end{figure}

FLAG's moment-matching update admits a zeroth-order gradient interpretation.
The mean shift $\mu^*_k(\hat{s}) - \mu_k$ from the E-step is equivalent to $\Sigma_k \nabla_\mu g_k(\mu_k)$, where $g_k(\mu) = \log \mathbb{E}_{\delta \sim \mathcal{N}(0, \Sigma_k)}[\exp(f_{\hat{s},k}(\mu + \delta)/\lambda)]$ is the log-sum-exp smoothing of the energy function, connecting FLAG's update to a zeroth-order estimate of the true action gradient $\nabla_a (Q^{\tilde{\pi}}(s,a) - \alpha \log \tilde{\pi}(a|s))$.
From this perspective, $\Sigma_k$ plays a dual role: it directly scales the zeroth-order gradient estimate, and simultaneously determines the width of the local Gaussian proposal $\hat{\pi}(a|\hat{s};\theta) = \mathcal{N}(a; T_\theta(s,z), \Sigma_k)$, controlling the search region around each anchor action.

Theoretically, a smaller $\Sigma_k$ is preferable as it reduces both the Q-function discrepancy between $\pi$ and $\tilde{\pi}$ and the smoothing bias of the gradient estimate, motivating the use of the learned covariance variant described in \Cref{subsection:covariance_learning}.
However, as shown in \Cref{tab:learned-covariance}, the learned covariance variant does not consistently outperform FLAG and exhibits sensitivity to the choice of $\beta$ and $\log\sigma_{\text{final}}$, making it less practical despite its closer theoretical alignment.
Simple linear annealing therefore remains our design choice, as it maintains a well-calibrated gradient scale and a sufficiently wide search region without additional hyperparameter sensitivity.

One might naturally ask whether FLAG's moment-matching update could be replaced by a direct zeroth-order gradient update, which estimates the action gradient from samples without the EM framework.
However, as shown in \Cref{fig:zeroth-order-gradient}, this zeroth-order gradient variant performs substantially worse than FLAG, failing to discover high-rewarding action sequences.
This highlights a key advantage of FLAG's SNIS-based moment matching over na\"{i}ve sample-based gradient estimation: by reweighting samples according to the E-step target distribution rather than following a raw gradient direction, FLAG produces more informed and stable policy updates, while also benefiting from the broader search region of the annealing schedule as an implicit exploratory mechanism.

\section{Additional Analyses} \label{app:additional_analyses}

\subsection{Global Importance Sampling with the Flow Policy}
\label{app:flow_global_is}
The comparison in \Cref{subsection: global policy sampling} contrasts FLAG with existing global-IS methods, which differ from FLAG in their policy class and training details as well as in their proposal.
To isolate the proposal, we construct a variant that keeps FLAG's flow policy, critic, and sample budget and replaces only the latent-conditioned local sampling with global sampling from the flow policy.
For each state $s$, we draw $N$ latents $z_i \sim p_z$ and generate candidates $a_i = T_{\theta_k}(s, z_i)$, so the proposal covers the full region of the action space that the policy can generate.
FLAG's local samples share a single latent and are summarized by one moment-matched target $\mu_k^*(\hat s)$ in \Eqref{eq:SNIS-moment-matching}.
Global samples come from different latents and may fall in distinct modes, so a single mean is not a meaningful target.
We instead weight the CFM objective of each candidate by its importance weight,
\begin{equation}
\label{eq:weighted-cfm-global}
    \mathcal{L}_{\mathrm{global}}(\theta; s) = \sum_{i=1}^{N} w_i \, \mathbb{E}_{\tau}\Big[ \big\| u_\theta^\tau\big(s, \zeta^\tau(z_i, a_i)\big) - (a_i - z_i) \big\|^2 \Big],
    \qquad w_i = \operatorname{softmax}_i\big( f_{\hat s, k}(a_i) / \lambda \big),
\end{equation}
which regresses each candidate onto its own latent with weight $w_i$.
Since the energy function in \Eqref{eq: energy-func} depends on $\hat s$ only through $s$, the weights are well defined across latents.
All other settings, including $N = 8$, the critic, and the training budget, are unchanged.

\Cref{tab:flow_global_is} shows that the global variant fails to learn on both tasks, with returns of 0.035 and 0.043, whereas FLAG reaches 0.680 and 0.912.
Since the two runs differ only in the proposal, the local proposal is the decisive component of FLAG.

\begin{table}[h]
    \centering
    \caption{\textbf{Global versus local importance sampling with the same flow policy on DMC Dog.} Normalized IQM return at 1M environment steps over 5 seeds, with 95\% confidence intervals in brackets. The two methods differ only in the proposal.}
    \label{tab:flow_global_is}
    \footnotesize
    \begin{tabular}{lcc}
    \toprule
    \textbf{Algorithm} & \textbf{Dog-Run} & \textbf{Dog-Trot} \\
    \midrule
    \textbf{FLAG (local IS)} & \bestnum{0.680} \textcolor{statgray}{\scriptsize [0.646, 0.768]} & \bestnum{0.912} \textcolor{statgray}{\scriptsize [0.901, 0.944]} \\
    Flow + global IS         & 0.035 \textcolor{statgray}{\scriptsize [0.033, 0.039]} & 0.043 \textcolor{statgray}{\scriptsize [0.040, 0.048]} \\
    \bottomrule
    \end{tabular}
\end{table}

\subsection{Distributional Critic versus Scalar Critic}
\label{app:distributional_critic}
The previous analysis isolates the proposal. Here we check the other component that the global-IS baselines in \Cref{section:experiments} share with FLAG.
FLAG follows DIME in training a distributional critic with CrossQ (\Cref{subsection: implementation-details}), and the comparisons in \Cref{section:experiments} match this critic across all methods.
To measure how much of FLAG's performance depends on this choice, we replace the distributional critic with a standard scalar critic and keep everything else fixed.
\Cref{tab:distributional_critic} reports the result on DMC Dog-run.
The distributional critic contributes, but FLAG with a scalar critic still reaches a normalized return of 0.576, an order of magnitude above the global-IS baselines in \Cref{fig:5_1}.
The gap between FLAG and these baselines is therefore not explained by the critic.

\begin{table}[h]
    \centering
    \caption{\textbf{Distributional versus scalar critic on DMC Dog-run.} Normalized IQM return at 1M environment steps over 5 seeds, with 95\% confidence intervals in brackets.}
    \label{tab:distributional_critic}
    \footnotesize
    \begin{tabular}{lc}
    \toprule
    \textbf{Algorithm} & \textbf{Dog-Run} \\
    \midrule
    \textbf{FLAG}            & \bestnum{0.680} \textcolor{statgray}{\scriptsize [0.646, 0.768]} \\
    FLAG w/ scalar critic    & 0.576 \textcolor{statgray}{\scriptsize [0.440, 0.673]} \\
    \bottomrule
    \end{tabular}
\end{table}

\subsection{Diversity versus Exploitation}
\label{app:diversity}
A natural question is whether the final FLAG policy preserves meaningful diversity, or whether its performance comes from a policy that has become sharper over training.
In this analysis we find that the policy keeps its spread throughout training, and sharpening it at test time recovers at most about 8\% of the return.
The gain comes from where the policy places its probability mass.

\paragraph{Setup.}
We take FLAG checkpoints on DMC Dog-trot at 100k, 500k, and 1M environment steps ($N = 8$) and sample actions from identical states and latents at every checkpoint, using 256 states with 4{,}096 samples each.
Since the states and latents are fixed, any change across checkpoints reflects the policy alone.
We report the standard deviation of the sampled actions, averaged over action dimensions and states, together with the normalized return of each checkpoint.

\paragraph{Action spread does not shrink.}
\Cref{tab:diversity_std} shows that the return grows by more than $8\times$ from 100k to 1M steps while the action standard deviation changes by less than 0.4\%.
The policy also remains far from deterministic.
What changes is where the mass sits, and we check by comparing the cosine similarity between the two mean-action vectors, which is only 0.22.
The policy retains its width and reorients toward a nearly orthogonal region of the action space.

\begin{table}[h]
    \centering
    \caption{\textbf{Action spread and return across checkpoints on DMC Dog-trot.} Actions are sampled from identical states and latents at each checkpoint (256 states, 4{,}096 samples each). Return is normalized by 1k, mean $\pm$ std over 10 evaluation episodes.}
    \label{tab:diversity_std}
    \footnotesize
    \begin{tabular}{lcccc}
    \toprule
    \textbf{Checkpoint} & $\bm{\mathrm{std}(\pi(a \mid s))}$ & $\bm{\Delta}$ \textbf{vs.\ 100k} & \textbf{Return (1k)} & \textbf{vs.\ 100k} \\
    \midrule
    100k & 0.4532 & --       & 0.094 $\pm$ 0.002 & $1.0\times$ \\
    500k & 0.4513 & $-0.4\%$ & 0.712 $\pm$ 0.013 & $7.6\times$ \\
    1M   & 0.4512 & $-0.4\%$ & 0.783 $\pm$ 0.010 & $8.3\times$ \\
    \bottomrule
    \end{tabular}
\end{table}

\paragraph{Test-time sharpening recovers little.}
Next, we check whether the policy places its mass near high-value actions early in training, and that the later gain comes from concentrating this mass more tightly on the best of them.
If this were the case, sharpening a policy by hand should recover a large part of the gain, and sharpening an early checkpoint should recover much of what later training delivers.
We test this directly by sampling the flow latent from a lower-temperature prior $z \sim \mathcal{N}(0, \tau^2 I)$, which concentrates the action distribution while leaving the flow map, the critic, and the state distribution untouched.
$\tau = 1$ is the trained policy, and $\tau \to 0$ collapses it onto the deterministic output $T_\theta(s, 0)$.
\Cref{tab:diversity_temperature} reports the return for $\tau \in \{0, 0.5, 1\}$ at each checkpoint.
Sharpening yields at most about 8\% at any checkpoint.
The early checkpoint is the decisive comparison, since the exploitation explanation predicts that sharpening the 100k policy should recover much of the later gain.
Instead, sharpening it to its best temperature adds 0.004 to the return, against the 0.689 added by training, which is less than 1\% of the total improvement.
The mass of the early policy is simply not in the right place, and no amount of concentration fixes that.
We therefore attribute FLAG's performance to probability mass being relocated onto high-value actions rather than to a sharper policy.
The same behavior is visible in the multi-goal example of \Figref{fig:multigoal}, where FLAG captures the multi-modal target instead of collapsing onto a single mode.

\begin{table}[h]
    \centering
    \caption{\textbf{Test-time sharpening on DMC Dog-trot.} Normalized return when the flow latent is sampled from $\mathcal{N}(0, \tau^2 I)$. $\tau = 1$ is the trained policy and $\tau = 0$ is its deterministic output. Mean $\pm$ std over 10 evaluation episodes.}
    \label{tab:diversity_temperature}
    \footnotesize
    \begin{tabular}{lccc}
    \toprule
    \textbf{Checkpoint} & $\bm{\tau = 0}$ & $\bm{\tau = 0.5}$ & $\bm{\tau = 1}$ \textbf{(trained)} \\
    \midrule
    100k & 0.090 $\pm$ 0.002 & 0.098 $\pm$ 0.004 & 0.094 $\pm$ 0.002 \\
    500k & 0.770 $\pm$ 0.006 & 0.762 $\pm$ 0.010 & 0.712 $\pm$ 0.013 \\
    1M   & 0.837 $\pm$ 0.010 & 0.834 $\pm$ 0.006 & 0.783 $\pm$ 0.010 \\
    \bottomrule
    \end{tabular}
\end{table}

\subsection{GPU Memory Allocation Comparison with DIME}
\label{app:gpu_memory}
DIME computes updates through direct gradient flow, which necessitates BPTT and leads to memory costs that grow rapidly with model size.
FLAG avoids this bottleneck entirely by replacing gradient-based updates with a supervision loss, decoupling memory consumption from the temporal depth of the computation graph.
As shown in \Cref{tab:gpu-memory-allocation}, BPTT causes memory usage to escalate sharply as the actor grows larger, whereas FLAG scales substantially more efficiently.
At 1B parameters, FLAG reduces memory consumption by over 43\% relative to DIME, and savings reach as high as 75\% at the 10M scale.
These results demonstrate that a supervision-based update scheme is a practical necessity for scaling actor networks into the billion-parameter regime.

\begin{table}[htbp]
\centering
\caption{GPU memory usage across actor parameter scales}
\label{tab:gpu-memory-allocation}
\footnotesize
\begin{tabular}{@{}cccccc@{}}
\toprule
 & \multicolumn{5}{c}{\textbf{Actor Parameters}} \\ 
 \cmidrule(l){2-6}
\textbf{Algorithm} & \textbf{100K} & \textbf{1M} & \textbf{10M} & \textbf{100M} & \textbf{1B} \\ \midrule
\textbf{DIME (MB)}             & 146           & 329         & 1133         & 1771          & 14220       \\
\textbf{FLAG (MB)}             & 110           & 117         & 283          & 1381          & 8094        \\
\midrule
\textbf{Memory Reduction (\%)} & 24.66         & 64.44       & 75.02        & 22.02         & 43.08       \\ \bottomrule
\end{tabular}
\end{table}

\subsection{Accuracy of the Cross-Entropy Proxy}
\label{app:kl_proxy_accuracy}
Proposition~\ref{prop:surrogate-validity} justifies the cross-entropy proxy asymptotically in the small-covariance regime.
Since $\mathcal{H}(\pi, \tilde\pi) - \mathcal{H}(\pi) = D_{\mathrm{KL}}(\pi \,\|\, \tilde\pi)$, the error of the proxy is exactly this KL divergence.
It can become non-negligible when the local covariance is large, or when the Gaussian smoothing substantially changes a sharply multimodal flow policy.
We therefore measure the KL directly, both for the covariances actually used by the annealing schedule of \Cref{subsection: implementation-details} and for inflated values that probe beyond the training regime.

\paragraph{Estimation protocol.}
We write the KL as $\mathbb{E}_{s \sim \mathcal{D}}\,\mathbb{E}_{a \sim \pi(\cdot \mid s)}\big[\log \pi(a \mid s) - \log \tilde\pi(a \mid s)\big]$ and estimate it by Monte Carlo over states from the replay buffer and actions drawn from $\pi$ as in \Eqref{eq:marginal-policy}.
The flow density $\log \tilde\pi(a \mid s)$ is obtained by integrating the flow in reverse from $a$ to its latent and accumulating the log-determinant with Hutchinson's trace estimator, as in training (Appendix~\ref{subsection:hutchinson-variance-reduction}).
The marginal density $\log \pi(a \mid s)$ has no closed form, so we estimate it by importance sampling over $z$.
The proposal is a Gaussian centered at the latent $z^* = T_\theta^{-1}(s, a)$ recovered by the reverse flow, with covariance $J^{-1} \Sigma J^{-\top}$, where $J = \partial T_\theta(s, z) / \partial z$ is the Jacobian of the flow map at $z^*$.
This proposal is the local linearization of the posterior over $z$ given $a$, so it concentrates its samples where the integrand $p_z(z)\,\hat\pi(a \mid s, z; \theta)$ is non-negligible.

\begin{table}[h]
    \centering
    \caption{\textbf{Accuracy of the cross-entropy proxy.} Estimated $D_{\mathrm{KL}}(\pi \,\|\, \tilde\pi)$ in nats as a function of the local variance $\sigma^2$. The first two rows use the scheduled variance at the indicated checkpoint, and the remaining rows inflate $\sigma^2$ beyond the training regime.}
    \label{tab:kl_proxy_accuracy}
    \footnotesize
    \begin{tabular}{lcc}
    \toprule
    \textbf{Setting} & $\bm{\sigma^2}$ & $\bm{D_{\mathrm{KL}}(\pi \,\|\, \tilde\pi)}$ \\
    \midrule
    Scheduled (1M steps)   & 0.00256 & 0.0003 \\
    Scheduled (200k steps) & 0.01832 & 0.00352 \\
    \midrule
    Inflated & 0.0596 & 0.054 \\
    Inflated & 0.236  & 0.579 \\
    Inflated & 1.0    & 4.633 \\
    \bottomrule
    \end{tabular}
\end{table}

\paragraph{Results.}
\Cref{tab:kl_proxy_accuracy} shows that the gap is negligible for the covariances visited during training, below $4 \times 10^{-3}$ nats at 200k steps and $3 \times 10^{-4}$ nats at 1M steps, and that it shrinks further as the schedule anneals $\sigma$.
Once $\sigma^2$ is inflated beyond the schedule, the gap grows rapidly and reaches 4.6 nats at $\sigma^2 = 1$.
The proxy is therefore accurate in the regime where FLAG operates, but not automatically so.
The covariance scheduling in \Cref{subsection: implementation-details} is what keeps training in the small-covariance regime, consistent with the role Theorem~\ref{thm:flag-monotonic-improvement} assigns to it, so the assumption of Proposition~\ref{prop:surrogate-validity} is satisfied by design.

\subsection{Protocol for the Diagnostics in \Cref{subsection:key_design_choices}}
\label{app:theorem_diagnostics}
This section details how \Cref{tab:theorem_diagnostics} in \Cref{subsection:key_design_choices} is computed.

\paragraph{Checkpoints and batch.}
We use checkpoints at 0.1M, 0.4M, 0.7M, and 1M environment steps from 3 seeds of FLAG trained on MuJoCo Walker2d with the default hyperparameters, including the guidance buffer and covariance schedule marked by $\dagger$ in \Cref{tab:ablation_buffer_covariance}.
At each checkpoint, we sample $B = 128$ replay states $s_i$ with latents $z_i$ and compute their E-step targets with the SNIS estimator of \Eqref{eq: e-step-target}, keeping the critic, $\alpha$, and $\sigma_k$ frozen.
We use $128$ local samples per state instead of the $8$ used in training to reduce the estimation noise of the targets.

\paragraph{CFM and ideal updates.}
From the same parameters $\theta_k$, we take one CFM update $\theta^{\mathrm{CFM}}$, a plain gradient step on \Eqref{eq: M-step-CFM}, and one ideal KL-projection update $\theta^{\mathrm{KL}}$, a gradient step on $\mathcal{L}_{\mathrm{KL}}$ backpropagated through the ODE solver.
Empirically, we observed that the raw KL gradient is two to three orders of magnitude larger in norm than the CFM gradient, so we rescale the KL update to the norm of the CFM update. The two updates then take equally sized steps and differ only in direction.
At every checkpoint, $\Delta\mathcal{L}_{\mathrm{KL}}(\theta^{\mathrm{KL}})$ agrees with its first-order prediction $\beta\|\nabla_\theta\mathcal{L}_{\mathrm{KL}}(\theta_k)\|^2$ to within $20\%$, where $\beta$ is the step size of the rescaled KL update, so the probe operates in the small-step regime assumed by the theorem.
Under this norm matching, $1 - e_{\mathrm{proj}}$ equals the cosine similarity between the two update directions to first order.

\paragraph{Reward drift.}
For $e_{\mathrm{drift}}$, we sample one action per state from the current policy and evaluate $\log\tilde\pi_{\theta_k}$ and $\log\tilde\pi_{\theta^{\mathrm{CFM}}}$ with the same actions and Hutchinson probe vectors, so the difference reflects only the parameter update.
Taking the absolute value prevents increases and decreases from cancelling, which makes $e_{\mathrm{drift}}$ a conservative measure of the drift.

%% file: sections/appendix/experiment_details.tex
\section{Experiment Details} \label{app:experiment_details}

We conduct all experiments using JAX~\citep{jax2018github}. For FlowRL~\citep{flowrl}, we follow their implementation with PyTorch~\citep{paszke2019pytorch}.
\begin{table}[]
\centering
\caption{Environment state and action space dimensions}
\label{tab:env_state_action_dim}
\scriptsize
\begin{tabular}{@{}lcccccc@{}}
\toprule
                          & Ant-v5 & HalfCheetah-v5 & Humanoid-v5 & Walker2d-v5 & DMC dog & MyoSuite \\
\midrule
$\text{dim}(\mathcal{S})$ & 105    & 17             & 348         & 17          & 223     & 93       \\
$\text{dim}(\mathcal{A})$ & 8      & 6              & 17          & 6           & 38      & 39       \\
\bottomrule
\end{tabular}
\end{table}

\begin{table}[]
\centering
\caption{Hyperparameters of experiments in \Figref{fig:5_1}}
\label{tab:proposal_sampling_performance_hyperparameters}
\scriptsize
\begin{tabular}{@{}lccccc@{}}
\toprule
                        & \textbf{SDAC}  & \textbf{DPMD} & \textbf{QVPO} & \textbf{MaxEntDP} & \textbf{FLAG} \\
\midrule
$H_\text{target}$       & $-0.9\text{dim}(\mathcal{A})$ & $-0.9\text{dim}(\mathcal{A})$ & N/A & N/A & $-\text{dim}(\mathcal{A})$ \\
Temp. Learn. Rate       & 7e-3            & 7e-3            & N/A                 & N/A        & 1e-3 \\
Critic Learn. Rate      & 3e-4          & 3e-4            & 3e-4            & 3e-4  & 3e-4 \\
Actor/Score Learn. Rate & 3e-4 $\to$ 3e-5 & 3e-4 $\to$ 3e-5 & 3e-4 $\to$ 3e-5 & 3e-4   & 3e-4 \\
Diffusion Steps         & 20                & 20                & 20                & 20       & 8      \\
Discount                & 0.99              & 0.99              & 0.99              & 0.99     & 0.99   \\
Batch size              & 256               & 256               & 256               & 256      & 256    \\
Buffer size             & 1e6             & 1e6             & 1e6             & 1e6    & 1e6  \\
Critic Hidden Depth     & 2                 & 2                 & 3                 & 2        & 2      \\
Critic Hidden Size      & 2048              & 2048              & 2048              & 2048     & 2048   \\
Critic CrossQ           & True              & True              & True              & True     & True   \\
Critic LayerNorm        & False             & False             & False             & False    & False  \\
Dist. Critic Bins       & 101               & 101               & 101               & 101      & 101    \\
Actor/Score Depth       & 3                 & 3                 & 3                 & 3        & 3      \\
Actor/Score Size        & 256               & 256               & 256               & 256      & 256    \\
Update-to-data Ratio    & 2                 & 2                 & 2                 & 2        & 2      \\
Policy delay            & 3                 & 3                 & 3                 & 3        & 3      \\
Exploration Steps       & 10000             & 10000             & 10000             & 10000    & 10000  \\ 
\bottomrule
\end{tabular}
\end{table}

\begin{table}[]
\centering
\caption{Hyperparameters of experiments in \Cref{tab:table4_dmc_myo_split}}
\label{tab:diffusion_flow_hyperparameters}
\scriptsize
\begin{tabular}{@{}lccccccc@{}}
\toprule
                        & \textbf{DIME} & \textbf{FlowRL} & \textbf{DACERv2} & \textbf{QSM} & \textbf{DIPO} & \textbf{CrossQ} & \textbf{FLAG} \\
\midrule
$H_\text{target}$       & $4\text{dim}(\mathcal{A})$ & N/A                 & $-0.9\text{dim}(\mathcal{A})$ & N/A                 & N/A                 & $-\text{dim}(\mathcal{A})$ & $-\text{dim}(\mathcal{A})$ \\
Temp. Learn. Rate       & 1e-3                       & N/A                 & 3e-2                          & N/A                 & N/A                 & 3e-4                       & 1e-3                       \\
Critic Learn. Rate      & 3e-4                       & 3e-4                & 1e-4                          & 3e-4                & 3e-4                & 7e-4                       & 3e-4                       \\
Actor/Score Learn. Rate & 3e-4                       & 3e-4                & 1e-4                          & 3e-4                & 3e-4                & 7e-4                       & 3e-4                       \\
Diffusion Steps         & 16                         & 1                   & 20                            & 15                  & 100                 & N/A                        & 8                          \\
Discount                & 0.99                       & 0.99                & 0.99                          & 0.99                & 0.99                & 0.99                       & 0.99                       \\
Batch size              & 256                        & 256                 & 256                           & 256                 & 256                 & 256                        & 256                        \\
Buffer size             & 1e6                        & 1e6                 & 1e6                           & 1e6                 & 1e6                 & 1e6                        & 1e6                        \\
Critic hidden depth     & 2                          & 3                   & 2                             & 2                   & 2                   & 2                          & 2                          \\
Critic hidden size      & 2048                       & 512                 & 2048                          & 2048                & 2048                & 2048                       & 2048                       \\
Critic use CrossQ       & True                       & False               & True                          & True                & True                & True                       & True                       \\
Critic use LayerNorm    & False                      & True                & False                         & False               & False               & False                      & False                      \\
Num. Bin/Quantiles      & 101                        & N/A                 & 2                             & N/A                 & 101                 & N/A                        & 101                        \\
Actor/Score depth       & 3                          & 2                   & 3                             & 3                   & 4                   & 3                          & 3                          \\
Actor/Score size        & 256                        & 512                 & 256                           & 256                 & 256                 & 256                        & 256                        \\
Update-to-data ratio    & 2                          & 1                   & 1                             & 1                   & 2                   & 2                          & 2                          \\
Policy delay            & 3                          & 1                   & 1                             & 1                   & 3                   & 3                          & 3                          \\
Exploration Steps       & 5000                       & 10000               & 10000                         & 10000               & 10000               & 5000                       & 10000                      \\
Prior Distr.            & $\mathcal{N}(0, 2.5)$      & $\mathcal{N}(0, 1)$ & $\mathcal{N}(0, 1)$           & $\mathcal{N}(0, 1)$ & $\mathcal{N}(0, 1)$ & N/A                        & $\mathcal{N}(0, 1)$        \\
Optimizer               & Adam                       & Adam                & Adam                          & Adam                & Adam                & Adam                       & Adam                       \\
Score-Q align. factor   & N/A                        & N/A                 & N/A                           & 50                  & N/A                 & N/A                        & N/A                        \\
\bottomrule
\end{tabular}
\end{table}

\begin{table}[]
\centering
\caption{List of FLAG-specific hyperparameters}
\label{tab:FLAG_hyperparameters}
\scriptsize
\begin{tabular}{cccc}
\toprule
$\alpha_\text{init}$ 
& $\lambda_\text{ref}$ 
& \makecell{$\log \sigma_\text{init}$\\ / $\log \sigma_\text{min}$} 
& \makecell{$\log \sigma$ Warmup\\ / Decay Steps} \\
\midrule
$0.01$ 
& $10$ 
& $-2$ / $-3$ 
& $2e5$ / $8e5$ \\
\midrule
\makecell{Supervision Warmup\\ / Ramp Steps} 
& Flow Buffer Size 
& \makecell{Number of $z$\\ given $s$ in update} 
& \makecell{Number of samples\\ for Variance Reduction} \\
\midrule
$1e5$ / $1e5$ 
& $10240$ 
& $1$ 
& $1$ \\
\bottomrule
\end{tabular}
\end{table}

\subsection{Multigoal Environment} \label{multigoal}
We conduct a small-scale experiment in the \texttt{MultiGoalEnv}, adapted from~\cite{softqlearning}, to qualitatively assess whether SDAC, DPMD, QVPO, MaxEntDP, and FLAG can recover a multimodal action distribution from a shared value function.
An optimal Q-function $Q(s,a)$ is first computed via dynamic programming on a discrete grid and then frozen across all methods as a common scoring oracle, isolating policy extraction from representation learning.
Each method optimizes a diffusion or flow matching policy to approximate the Boltzmann target $\pi^*(a \mid s) \propto \exp(Q(s,a)/\tau)$ with $\tau = 0.1$, and we visualize the learned action distributions at two probe states, $(0, 0)$ and $(3, 3)$, via kernel density estimation.
By fixing the Q-function, this setup directly measures each method's ability to capture the target distribution as a function of the number of importance samples $N \in \{2, 8, 32\}$ per gradient step, without confounding factors from Q-learning or environment interaction.

\paragraph{Quantifying the importance-sampling mechanism.}
Figure~\ref{fig:multigoal} shows that FLAG captures the target modes with few samples, but the learned densities alone do not reveal why.
To support the mechanism directly, we track three statistics during training in the same setup with $N = 8$ importance samples per gradient step.
\emph{ESS$/K$} is the per-state effective sample size of the normalized IS weights divided by the number of samples $K$, so that $1$ means every sample contributes equally and $1/K$ means a single sample carries all the weight.
\emph{$H/\log K$} is the entropy of the normalized weights divided by $\log K$, where $1$ corresponds to uniform weights and $0$ to a one-hot weight vector.
\emph{In-mode mass} is the fraction of $512$ policy samples that fall within radius $r = 0.2$ of one of the modes of $Q(s, \cdot)$ at the two probe states $(0, 0)$ and $(3, 3)$.
\Cref{tab:multigoal_is_stats} reports these statistics at $1$k, $5$k, and $10$k gradient steps.

The global-IS baselines keep only one to three of the eight samples effective (ESS$/K \approx 0.15$ to $0.38$), whereas FLAG's local IS keeps nearly the full budget effective (ESS$/K \to 0.93$) and attains the highest in-mode mass at both probe states.
High ESS by itself could also arise from uniform, uninformative weights, so the in-mode mass is what makes the statistic diagnostic.
Together, the two columns show that FLAG's weights flatten because the proposal itself has moved onto the target region, so most samples are already close to the target and receive comparable weight.

\begin{table}[h]
    \centering
    \caption{\textbf{Importance-sampling statistics in the multigoal environment.} Values at $1$k $\rightarrow$ $5$k $\rightarrow$ $10$k gradient steps with $N = 8$ samples per update. ESS$/K$ and $H/\log K$ measure how many samples effectively contribute to each update; in-mode mass measures how much of the learned policy lies on the target modes.}
    \label{tab:multigoal_is_stats}
    \scriptsize
    \setlength{\tabcolsep}{4pt}
    \begin{tabular}{lcccccccccccc}
    \toprule
    & \multicolumn{3}{c}{\textbf{ESS$/K$}} & \multicolumn{3}{c}{\textbf{$H/\log K$}} & \multicolumn{3}{c}{\textbf{In-mode mass @ $(0,0)$}} & \multicolumn{3}{c}{\textbf{In-mode mass @ $(3,3)$}} \\
    \cmidrule(lr){2-4} \cmidrule(lr){5-7} \cmidrule(lr){8-10} \cmidrule(lr){11-13}
    \textbf{Algorithm} & 1k & 5k & 10k & 1k & 5k & 10k & 1k & 5k & 10k & 1k & 5k & 10k \\
    \midrule
    \textbf{FLAG} & \bestnum{0.530} & \bestnum{0.859} & \bestnum{0.926} & \bestnum{0.745} & \bestnum{0.949} & \bestnum{0.976} & \bestnum{0.13} & \bestnum{0.35} & \bestnum{0.43} & 0.03 & \bestnum{0.41} & \bestnum{0.62} \\
    SDAC     & 0.165 & 0.195 & 0.183 & 0.051 & 0.096 & 0.086 & 0.10 & 0.09 & 0.20 & 0.03 & 0.19 & 0.30 \\
    DPMD     & 0.158 & 0.263 & 0.378 & 0.044 & 0.302 & 0.502 & 0.12 & 0.05 & 0.12 & \bestnum{0.04} & 0.26 & 0.40 \\
    MaxEntDP & 0.170 & 0.223 & 0.215 & 0.087 & 0.177 & 0.161 & 0.12 & 0.05 & 0.21 & 0.03 & 0.17 & 0.32 \\
    QVPO     & 0.220 & 0.197 & 0.211 & 0.404 & 0.395 & 0.397 & 0.12 & 0.10 & 0.16 & \bestnum{0.04} & 0.10 & 0.16 \\
    \bottomrule
    \end{tabular}
\end{table}

\subsection{Comparison with Target Matching Methods}
\subsubsection{\Cref{fig:5_1}.}
\label{app: subsub-fig3}
Scores are averaged over four tasks per benchmark: HalfCheetah, Ant, Walker2d, and Humanoid for MuJoCo; Dog-run, Dog-trot, Dog-walk, and Dog-stand for DMC Dog; and reach-hard, obj-hold-hard, key-turn-hard, and pen-twirl-hard for MyoSuite.
Environment state and action space dimensions are summarized in \Cref{tab:env_state_action_dim}.
All methods are evaluated at 1M environment steps using 5 evaluation episodes, with IQM returns and confidence intervals computed over 10 seeds for $P = 1$ and 5 seeds for $P = 32$ due to the computational burden.
GPU hours are measured on a single NVIDIA L40S GPU.
Hyperparameters are reported in Table~\ref{tab:proposal_sampling_performance_hyperparameters}, and FLAG-specific hyperparameters shared across all experiments are listed in Table~\ref{tab:FLAG_hyperparameters}.
Full learning curves corresponding to \Figref{fig:5_1} are provided in Figures~\ref{fig:5_1_learning_curves_p1} and~\ref{fig:5_1_learning_curves_p32}.

\subsubsection{\Figref{fig:wo_crossq}}
\label{app: subsub-fig4}
We report performance at 1M steps using 5 evaluation episodes, with mean and standard deviation over 10 seeds for $P = 1$ and 5 seeds for $P = 32$, consistent with \Figref{fig:5_1}.
Hyperparameters follow those used for \Figref{fig:5_1} and \Cref{tab:table4_dmc_myo_split}, except for CrossQ usage and critic depth and hidden size. Specifically, we use $[256, 256, 256]$ critic architectures without Layer Normalization or Batch Renormalization. Please see Table~\ref{tab:dog_trot_run_performance_no_crossq} for detailed results.

\subsection{Comparisons with Diffusion and Flow-based Algorithms}
\paragraph{\Cref{tab:table4_dmc_myo_split}.}
We report performance at 1M environment interaction steps using 5 evaluation episodes with 10 random seeds per algorithm.
Hyperparameters are reported in Table~\ref{tab:diffusion_flow_hyperparameters}, and additional learning curves for the four MuJoCo tasks are provided in Figure~\ref{fig:5_2_learning_curves}.

\begin{figure}[]
    \centering
    \includegraphics[width=\textwidth]{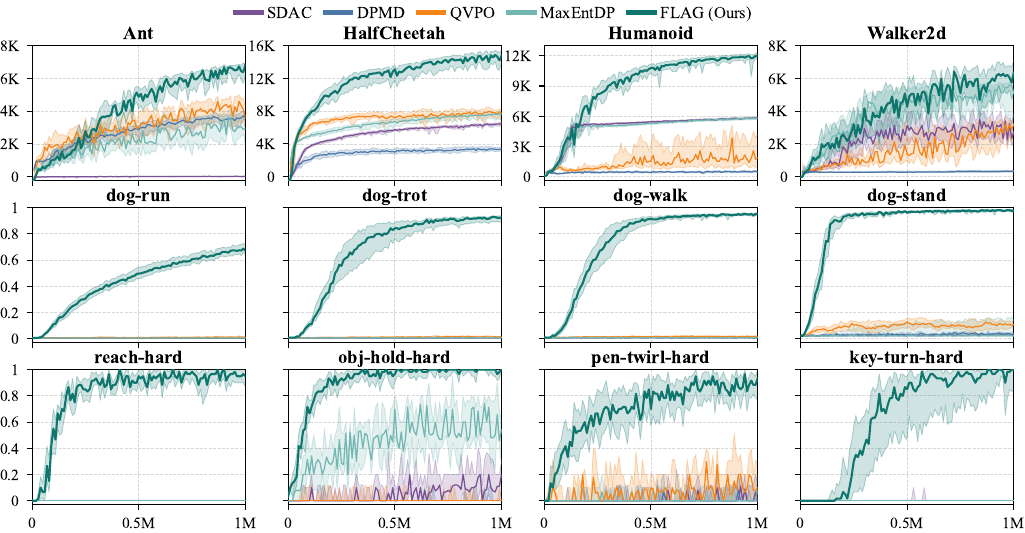}
    \caption{Full learning curves of \Figref{fig:5_1} (P=1)}
    \label{fig:5_1_learning_curves_p1}
\end{figure}

\begin{figure}[]
    \centering
    \includegraphics[width=\textwidth]{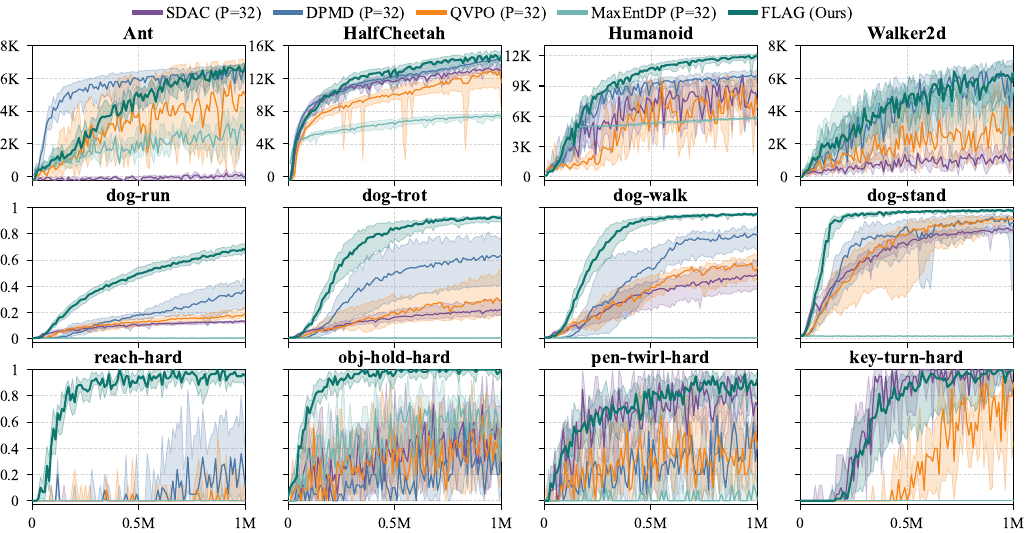}
    \caption{Full learning curves of \Figref{fig:5_1} (P=32)}
    \label{fig:5_1_learning_curves_p32}
\end{figure}

\begin{figure}[]
    \centering
    \includegraphics[width=\textwidth]{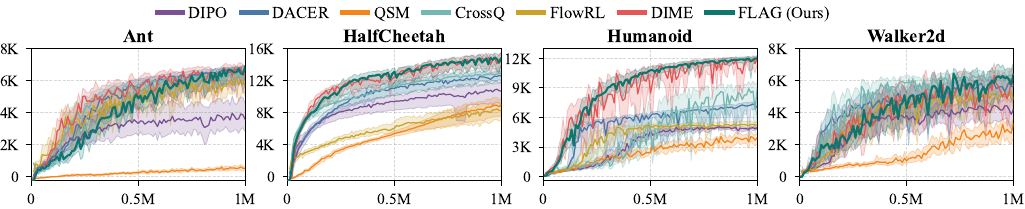}
    \caption{Learning curves other diffusion \& flow-matching algorithms in 4 MuJoCo tasks}
    \label{fig:5_2_learning_curves}
\end{figure}

\begin{table}[]
\centering
\caption{
\textbf{DMC Dog performance without CrossQ in Section~\ref{subsection: global policy sampling}.}
$\Delta_{\mathrm{F}}$ denotes the relative performance drop from FLAG,
defined as $(1 - \text{Ret}_\text{Alg} / \text{Ret}_\text{FLAG}) \times 100$.
$\dagger$ denotes the baselines' policy update relevant default hyperparameters used in the original papers.
The \capbest{} and \capsecond{} scores are highlighted.
}
\label{tab:dog_trot_run_performance_no_crossq}
\footnotesize
\renewcommand{\arraystretch}{0.95}
\setlength{\tabcolsep}{2pt}

\newcommand{\gap}[1]{#1}
\newcommand{\bestgap}[1]{\cellcolor{hltealbest}{#1}}
\newcommand{\secondgap}[1]{\cellcolor{hltealsecond}#1}
\newcommand{\unithead}[1]{\textcolor{statgray}{\scriptsize #1}}
\newcommand{\algname}[1]{{\footnotesize #1}}

\begin{tabular}{@{}c c l c r c r@{}}
\toprule
\multirow{2}{*}{\textbf{P}} &
\multirow{2}{*}{\textbf{N}} &
\multirow{2}{*}{\textbf{Alg.}} &
\multicolumn{2}{c}{\textbf{Dog-trot}} &
\multicolumn{2}{c}{\textbf{Dog-run}} \\
\cmidrule(lr){4-5}
\cmidrule(l){6-7}
&
&
&
\subhead{Return \unithead{(1k)}$\uparrow$} &
\subhead{$\Delta_{\mathrm{F}} \downarrow$} &
\subhead{Return \unithead{(1k)}$\uparrow$} &
\subhead{$\Delta_{\mathrm{F}} \downarrow$} \\
\midrule

\multirow{4}{*}{1}
& 64  & \algname{SDAC}
& \res{0.008}{0.001} & \gap{0.984}
& \res{0.006}{0.000} & \gap{0.980} \\

& 64  & \algname{DPMD}
& \res{0.013}{0.008} & \gap{0.974}
& \res{0.006}{0.002} & \gap{0.980} \\

& 64  & \algname{QVPO}
& \res{0.020}{0.022} & \gap{0.960}
& \res{0.039}{0.036} & \gap{0.872} \\

& 500 & \algname{$\text{MaxEntDP}^\dagger$}
& \res{0.007}{0.001} & \gap{0.986}
& \res{0.005}{0.001} & \gap{0.984} \\

\cmidrule{1-7}

32 & 64 & \algname{$\text{SDAC}^\dagger$}
& \res{0.019}{0.023} & \gap{0.962}
& \res{0.007}{0.002} & \gap{0.977} \\

32 & 1 & \algname{$\text{DPMD}^\dagger$}
& \second{0.117}{0.125} & \secondgap{0.766}
& \res{0.023}{0.016} & \gap{0.924} \\

4 & 64 & \algname{$\text{QVPO}^\dagger$}
& \res{0.077}{0.082} & \gap{0.846}
& \res{0.045}{0.054} & \gap{0.852} \\

\cmidrule{1-7}

-- & -- & \algname{$\text{DIME}^\dagger$}
& \res{0.088}{0.073} & \gap{0.824}
& \res{0.025}{0.033} & \gap{0.918} \\

-- & -- & \algname{$\text{DACERv2}^\dagger$}
& \res{0.107}{0.043} & \gap{0.786}
& \second{0.065}{0.059} & \secondgap{0.786} \\

\cmidrule{1-7}

1 & 8 & \algname{$\text{FLAG}^\dagger$(ours)}
& \best{0.500}{0.026} & \bestgap{0.0}
& \best{0.304}{0.027} & \bestgap{0.0} \\

\bottomrule
\end{tabular}
\end{table}
\clearpage

%% file: checklist.tex
\section*{NeurIPS Paper Checklist}

\begin{enumerate}

\item {\bf Claims}
    \item[] Question: Do the main claims made in the abstract and introduction accurately reflect the paper's contributions and scope?
    \item[] Answer: \answerYes{}
    \item[] Justification:
    The three claims stated in the introduction are each substantiated in the paper.
    \begin{itemize}
        \item The latent-augmented MDP and the proxy MaxEnt-RL objective with theoretical consistency are introduced in \Cref{subsection: z-MDP} and proved in Appendix \ref{app: derivations}.
        \item The EM-based FLAG algorithm with a monotonic improvement guarantee is presented in \Cref{subsection: local-policy-update} and proved in Appendix \ref{app: proofs}.
        \item Empirical superiority over global IS baselines and state-of-the-art performance are demonstrated across DMC Dog, MyoSuite, and MuJoCo benchmarks in \Cref{section:experiments} (\Cref{fig:5_1,fig:wo_crossq}, \Cref{tab:ablation_num_training_samples_dmc_dog,tab:table4_dmc_myo_split,tab:ablation_buffer_covariance}).
    \end{itemize}

\item {\bf Limitations}
    \item[] Question: Does the paper discuss the limitations of the work performed by the authors?
    \item[] Answer: \answerYes{}
    \item[] Justification:
    \Cref{section:conclusion} explicitly identifies the main limitation: the specific combination of a base flow policy and a local Gaussian head is one instantiation of the proposed framework, and a more principled construction via ODE-to-SDE conversion~\citep{domingoadjoint} is left for future work.
    In addition, the monotonic improvement guarantee (\Cref{thm:flag-monotonic-improvement}) relies on the approximation assumptions stated from \ref{ass:flag-mpo-style} to  \ref{ass:flag-cfm-proj}, whose violation in practice is discussed in the surrounding remarks.

\item {\bf Theory assumptions and proofs}
    \item[] Question: For each theoretical result, does the paper provide the full set of assumptions and a complete (and correct) proof?
    \item[] Answer: \answerYes{}
    \item[] Justification:
    All standing assumptions (from \ref{assum:bounded_reward} to \ref{assum:finite_partition}) and iteration-specific assumptions (from \ref{ass:flag-mpo-style} to \ref{ass:flag-cfm-proj}) are explicitly stated in \Cref{app: proofs}.
    Full proofs are provided for every formal result: Corollaries \ref{cor:marginal-consistency} and \ref{cor: q-func-consistency} in Appendix \ref{app: derivations}, Proposition \ref{prop:surrogate-validity} in Appendix \ref{appendix:TV-KL-bound-proof}, the SAC-perspective connection in Appendix \ref{app: cov-schedule-justification}, and the MPO-style monotonic improvement (\ref{prop:flag-full-monotone-exact}, \ref{lem:flag-full-monotone-cfm}) in Appendix \ref{appendix:monotonic-improvement}.
    Proof sketches are embedded in the main text (\Cref{subsection: local-policy-update}) and a visual roadmap is provided in \Cref{fig:flag-mpo-flow}.

    \item {\bf Experimental result reproducibility}
    \item[] Question: Does the paper fully disclose all the information needed to reproduce the main experimental results of the paper to the extent that it affects the main claims and/or conclusions of the paper (regardless of whether the code and data are provided or not)?
    \item[] Answer: \answerYes{}
    \item[] Justification:
    The complete training procedure is given in \Cref{algorithm::FLAG} and Appendix \ref{app:implementation_details}.
    All FLAG-specific hyperparameters are listed in \Cref{tab:FLAG_hyperparameters}, and full hyperparameter tables for the comparison experiments are provided in \Cref{tab:proposal_sampling_performance_hyperparameters,tab:diffusion_flow_hyperparameters}.
    Baseline implementations reference official GitHub repositories (Appendix \ref{app: Baseline_5.1}), and any deviations from the originals are described in detail.
    Environment dimensions are reported in \Cref{tab:env_state_action_dim}, and evaluation protocols (number of seeds, evaluation episodes, metric definitions) are specified in \Cref{section:experiments} and Appendix~\ref{app:experiment_details}.

\item {\bf Open access to data and code}
    \item[] Question: Does the paper provide open access to the data and code, with sufficient instructions to faithfully reproduce the main experimental results, as described in supplemental material?
    \item[] Answer: \answerYes{}
    \item[] Justification:
    A codebase to reproduce the results introduced in this paper is presented with the paper.
    All environments used (DMC, MyoSuite, MuJoCo) are publicly available benchmarks, and all baseline implementations are linked to their official repositories in Appendix~\ref{app: Baseline_5.1}.

\item {\bf Experimental setting/details}
    \item[] Question: Does the paper specify all the training and test details (e.g., data splits, hyperparameters, how they were chosen, type of optimizer) necessary to understand the results?
    \item[] Answer: \answerYes{}
    \item[] Justification:
    Optimizer (Adam), learning rates, batch sizes, buffer sizes, network architectures, discount factors, update-to-data ratios, and all other training details are fully reported in \Cref{tab:proposal_sampling_performance_hyperparameters,tab:diffusion_flow_hyperparameters,tab:FLAG_hyperparameters}.
    FLAG-specific design choices (covariance scheduling, guidance buffer size, effective temperature) are explained in \Cref{subsection: implementation-details} and further ablated in \Cref{subsection:key_design_choices} and Appendix \ref{app: additional_ablation_studies}.

\item {\bf Experiment statistical significance}
    \item[] Question: Does the paper report error bars suitably and correctly defined or other appropriate information about the statistical significance of the experiments?
    \item[] Answer: \answerYes{}
    \item[] Justification:
    All results are reported using the Interquartile Mean (IQM) with 95\% confidence intervals computed over 10 random seeds (5 for the $P{=}32$ condition, as noted in \Cref{subsection: global policy sampling}).
    Each policy is evaluated with 5 evaluation episodes per seed.
    Regarding the $\pm$ notations in \Cref{tab:ablation_num_training_samples_dmc_dog,tab:table4_dmc_myo_split}, we directly adopt the notation from \citep{simba-v2}.

\item {\bf Experiments compute resources}
    \item[] Question: For each experiment, does the paper provide sufficient information on the computer resources (type of compute workers, memory, time of execution) needed to reproduce the experiments?
    \item[] Answer: \answerYes{}
    \item[] Justification:
    \Cref{fig:5_1} reports wall-clock GPU hours for a single 1M-step training run on an NVIDIA L40S GPU for all compared methods, serving as the primary compute reference.
    \Cref{fig:runtime} provides a runtime bar chart comparing FLAG across different numbers of ODE solver steps, and notes that training duration increases proportionally with solver steps.
    The chosen solver step count of 8 is justified by this efficiency analysis.

\item {\bf Code of ethics}
    \item[] Question: Does the research conducted in the paper conform, in every respect, with the NeurIPS Code of Ethics \url{https://neurips.cc/public/EthicsGuidelines}?
    \item[] Answer: \answerYes{}
    \item[] Justification:
    This work studies continuous control in simulated physics environments and involves no human subjects, personal data, or sensitive applications.
    No aspects of the research raise ethical concerns under the NeurIPS Code of Ethics.

\item {\bf Broader impacts}
    \item[] Question: Does the paper discuss both potential positive societal impacts and negative societal impacts of the work performed?
    \item[] Answer: \answerNA{}
    \item[] Justification:
    This is foundational research on reinforcement learning algorithms evaluated exclusively on standard simulated locomotion benchmarks.
    There is no direct path to deployment or to harmful applications arising from improvements in simulation-based continuous control.
    Potential downstream benefits---such as more capable robotic controllers---are generic to the field and do not require specific discussion here.

\item {\bf Safeguards}
    \item[] Question: Does the paper describe safeguards that have been put in place for responsible release of data or models that have a high risk for misuse (e.g., pre-trained language models, image generators, or scraped datasets)?
    \item[] Answer: \answerNA{}
    \item[] Justification:
    The paper releases neither pre-trained models nor datasets.
    The proposed algorithm operates on simulated locomotion tasks and does not pose misuse risks that would necessitate safeguards.

\item {\bf Licenses for existing assets}
    \item[] Question: Are the creators or original owners of assets (e.g., code, data, models), used in the paper, properly credited and are the license and terms of use explicitly mentioned and properly respected?
    \item[] Answer: \answerYes{}
    \item[] Justification:
    All baseline algorithms (SDAC, DPMD, QVPO, MaxEntDP, DIME, FlowRL, DACERv2, DIPO, QSM) are citepd with their original papers and their official GitHub repositories are linked in Appendix \ref{app: Baseline_5.1}.
    The benchmark environments DMC~\citep{dmc}, MyoSuite~\citep{myosuites}, and MuJoCo~\citep{todorov2012mujoco} are all citepd.
    The CrossQ critic implementation is similarly citepd~\citep{crossq}.
    \item[] Guidelines:
    \begin{itemize}
        \item The answer \answerNA{} means that the paper does not use existing assets.
        \item The authors should citep the original paper that produced the code package or dataset.
        \item The authors should state which version of the asset is used and, if possible, include a URL.
        \item The name of the license (e.g., CC-BY 4.0) should be included for each asset.
        \item For scraped data from a particular source (e.g., website), the copyright and terms of service of that source should be provided.
        \item If assets are released, the license, copyright information, and terms of use in the package should be provided. For popular datasets, \url{paperswithcode.com/datasets} has curated licenses for some datasets. Their licensing guide can help determine the license of a dataset.
        \item For existing datasets that are re-packaged, both the original license and the license of the derived asset (if it has changed) should be provided.
        \item If this information is not available online, the authors are encouraged to reach out to the asset's creators.
    \end{itemize}

\item {\bf New assets}
    \item[] Question: Are new assets introduced in the paper well documented and is the documentation provided alongside the assets?
    \item[] Answer: \answerNA{}
    \item[] Justification:
    The paper does not introduce new datasets, benchmarks, or pre-trained model checkpoints as standalone assets.

\item {\bf Crowdsourcing and research with human subjects}
    \item[] Question: For crowdsourcing experiments and research with human subjects, does the paper include the full text of instructions given to participants and screenshots, if applicable, as well as details about compensation (if any)?
    \item[] Answer: \answerNA{}
    \item[] Justification:
    The paper involves no crowdsourcing and no research with human subjects.

\item {\bf Institutional review board (IRB) approvals or equivalent for research with human subjects}
    \item[] Question: Does the paper describe potential risks incurred by study participants, whether such risks were disclosed to the subjects, and whether Institutional Review Board (IRB) approvals (or an equivalent approval/review based on the requirements of your country or institution) were obtained?
    \item[] Answer: \answerNA{}
    \item[] Justification:
    The paper involves no human subjects and therefore requires no IRB approval or equivalent review.

\item {\bf Declaration of LLM usage}
    \item[] Question: Does the paper describe the usage of LLMs if it is an important, original, or non-standard component of the core methods in this research? Note that if the LLM is used only for writing, editing, or formatting purposes and does \emph{not} impact the core methodology, scientific rigor, or originality of the research, declaration is not required.
    \item[] Answer: \answerNA{}
    \item[] Justification:
    No LLMs are used as a component of the proposed method, the experimental setup, or the evaluation procedure.
    LLMs were not used in the core methodology or scientific contributions of this work.

\end{enumerate}